\documentclass[12pt]{colt2016}

\usepackage{times}
\usepackage{algorithm}
\usepackage{algorithmic}

\usepackage{tabulary}
\usepackage{booktabs}

\usepackage{epsf,psfrag,amsfonts, xcolor,array}
\usepackage{amsmath}
\usepackage{multirow}
\usepackage{cases}%[subnum]
\usepackage{verbatim}
\usepackage{xspace}
\usepackage{bbm}
\usepackage{mathrsfs}
\usepackage{setspace}
\usepackage[fleqn,tbtags]{mathtools}
\usepackage{tikz}

\usetikzlibrary{positioning}

\usepackage[colorinlistoftodos, textwidth=20mm]{todonotes}
\usepackage{color}

\newcommand{\vecy}{\vec y}

\newcommand{\smucrl}{\textsc{\small{SM-UCRL}}\xspace}
\newcommand{\ucrl}{\textsc{\small{UCRL}}\xspace}

\newcommand{\Dr}{D_{\text{ratio}}}

\newcommand{\w}{W}
\newcommand{\VV}{B}
\newcommand{\vv}{b}
\newcommand{\HH}{S}
\newcommand{\hhh}{s}
\newcommand{\deltaO}{d_O}

\renewcommand{\Re}{\mathbb{R}}
\newcommand{\calP}{\mathcal P}

\newcommand{\A}{\mathcal A}
\newcommand{\B}{\mathcal B}
\newcommand{\X}{\mathcal X}
\newcommand{\T}{\mathcal T}

\newcommand{\TN}{N}

\newcommand{\Y}{\mathcal Y}

\newcommand{\M}{\mathcal M}

\newcommand{\E}{\mathbb E}

\newcommand{\Prob}{\mathbb P}
\newcommand{\I}{\mathbbm 1}

\newcommand{\R}{\mathcal{R}}

\newcommand{\wt}{\widetilde}
\newcommand{\wh}{\widehat}
\newcommand{\wb}{\overline}

\newtheorem{assumption}{Assumption}
\newtheorem*{theorem*}{Theorem}
\title[Reinforcement Learning of POMDPs using Spectral Methods]{Reinforcement Learning of POMDPs using Spectral Methods}

 \coltauthor{\Name{Kamyar Azizzadenesheli} \thanks{K. Azizzadenesheli is supported in part by NSF Career award CCF-1254106 and  ONR Award N00014-14-1-0665}\Email{kazizzad@uci.edu}\\
 \addr University of California, Irvine
 \AND
 \Name{Alessandro Lazaric} \thanks{A. Lazaric is supported in part by a grant from CPER Nord-Pas de Calais/FEDER DATA Advanced data science and technologies 2015-2020, CRIStAL (Centre de Recherche en Informatique et Automatique de Lille), and the French National Research Agency (ANR) under project ExTra-Learn n.ANR-14-CE24-0010-01.}\Email{alessandro.lazaric@inria.fr}\\
 \addr Institut National de Recherche en Informatique et en Automatique, (Inria)
  \AND
 \Name{Animashree Anandkumar}\thanks{A. Anandkumar is supported in part by Microsoft Faculty Fellowship, NSF Career award CCF-1254106,  ONR Award N00014-14-1-0665, ARO YIP Award W911NF-13-1-0084 and AFOSR YIP  FA9550-15-1-0221} \Email{a.anandkumar@uci.edu}\\
 \addr University of California, Irvine
 }

\def\bfy{{\mathbf y}}
\def\calY{{\cal Y}}
\def\calX{{\cal X}}
\usepackage{hyperref}

\usepackage[explicit]{titlesec}
\titlespacing{\paragraph}{0pt}{0.5em}{0.5em}[]

\begin{document}

%\twocolumn[
%\icmltitle{Reinforcement Learning for POMDPs using Spectral Methods}

% unless you've provided the [accepted] option to the icml2015
%\icmlauthor{Your Name}{email@yourdomain.edu}
%\icmladdress{Your Fantastic Institute,
%            314159 Pi St., Palo Alto, CA 94306 USA}
%\icmlauthor{Your CoAuthor's Name}{email@coauthordomain.edu}
%\icmladdress{Their Fantastic Institute,
%            27182 Exp St., Toronto, ON M6H 2T1 CANADA}
%
%% You may provide any keywords that you
%% find helpful for describing your paper; these are used to populate
%% the "keywords" metadata in the PDF but will not be shown in the document
%\icmlkeywords{boring formatting information, machine learning, ICML}
\maketitle

%\vskip 0.3in

\begin{abstract}
We propose a new reinforcement learning algorithm for partially observable Markov decision processes (POMDP) based on spectral decomposition methods. While spectral methods have been previously employed for consistent learning of (passive) latent variable models such as hidden Markov models, POMDPs are more challenging  since the learner interacts with the environment and possibly changes the future observations in the process. We devise a learning algorithm running through episodes, in each episode we employ spectral techniques to learn the POMDP parameters from a trajectory generated by a fixed policy. At the end of the episode, an optimization oracle returns the optimal memoryless planning policy which maximizes the expected reward based on the estimated POMDP model. We prove an order-optimal regret bound with respect to the optimal memoryless policy and efficient scaling with respect to the dimensionality of observation and action spaces. 
\end{abstract}

%\todoa{I think we should avoid using COLT style in the arXiv submission (JMLR is perfect but I would rather avoid putting the COLT tags)}


%%%%%%%%%%%%%%%%%%%%%%%%%%%%%%%%%%%%%%%%%%%%%%%%%%%%%%%%%%%%%%%%
%%%%%%%%%%%%%%%%%%%%%%%%%%%%%%%%%%%%%%%%%%%%%%%%%%%%%%%%%%%%%%%%
%% INTRODUCTION
%%%%%%%%%%%%%%%%%%%%%%%%%%%%%%%%%%%%%%%%%%%%%%%%%%%%%%%%%%%%%%%%
%%%%%%%%%%%%%%%%%%%%%%%%%%%%%%%%%%%%%%%%%%%%%%%%%%%%%%%%%%%%%%%%

% !TEX root = master.tex

%%%%%%%%%%%%%%%%%%%%%%%%%%%%%%%%%%%%%%%%%%%%%%%%%%%%%%%%%%%%%%%%
%%%%%%%%%%%%%%%%%%%%%%%%%%%%%%%%%%%%%%%%%%%%%%%%%%%%%%%%%%%%%%%%
%% INTRODUCTION
%%%%%%%%%%%%%%%%%%%%%%%%%%%%%%%%%%%%%%%%%%%%%%%%%%%%%%%%%%%%%%%%
%%%%%%%%%%%%%%%%%%%%%%%%%%%%%%%%%%%%%%%%%%%%%%%%%%%%%%%%%%%%%%%%

\section{Introduction}\label{s:intro}

Reinforcement Learning (RL) is an effective approach to solve the problem of sequential decision--making under uncertainty. RL agents learn how to maximize  long-term reward using the experience obtained by direct interaction with a stochastic environment~\citep{bertsekas1996neuro-dynamic,sutton1998introduction}. Since the environment is initially unknown, the agent has to balance between \textit{exploring} the environment to estimate its structure, and \textit{exploiting} the estimates to compute a policy that maximizes the long-term reward. As a result, designing a RL algorithm requires three different elements: \textbf{1)} an estimator for the environment's structure, \textbf{2)} a planning algorithm to compute the optimal policy of the estimated environment~\citep{lavalle2006planning}, and \textbf{3)} a strategy to make a trade off between exploration and exploitation to minimize the \textit{regret}, i.e., the difference between the performance of the exact optimal policy and the rewards accumulated by the agent over time.

Most of RL literature assumes that the environment can be modeled as a Markov decision process (MDP), with a Markovian state evolution that  is fully observed. A number of exploration--exploitation strategies have been shown to have strong performance guarantees for MDPs, either in terms of regret or sample complexity (see Sect.~\ref{ss:related} for a review). However, the assumption of full observability of the state evolution is often violated in practice, and the agent may only have noisy observations of the true state of the environment (e.g., noisy sensors in robotics). In this case, it is more appropriate to use the partially-observable MDP or POMDP~\citep{sondik1971the-optimal} model.

Many challenges arise in designing RL algorithms for POMDPs. Unlike in MDPs, the estimation problem (element 1) involves identifying the parameters of a latent variable model (LVM). In an MDP  the agent directly observes (stochastic) state transitions, and the estimation of the generative model is straightforward via empirical  estimators. On the other hand, in  a POMDP the transition and reward models must be inferred from noisy observations and the Markovian state evolution is hidden. The planning problem (element 2), i.e., computing the optimal policy for a POMDP with known parameters, is PSPACE-complete~\citep{papadimitriou1987the-complexity}, and it requires solving an augmented MDP built on a continuous belief space (i.e., a distribution over the hidden state of the POMDP). Finally, integrating estimation and planning in an exploration--exploitation strategy (element 3) with guarantees is non-trivial and no no-regret strategies are currently known (see Sect.~\ref{ss:related}).

%%%%%%%%%%%%%%%%%%%%%%%%%%%%%%%%%%%%%%%%%%%%%%%%%%%%%%%%%%%%

\subsection{Summary of Results}\label{ss:summary}

The main contributions of this paper are as follows: (i) We propose a new RL algorithm for POMDPs that incorporates spectral parameter estimation within a exploration-exploitation framework,   (ii) we analyze regret bounds assuming access to an optimization oracle that provides the best memoryless planning policy at the end of each learning episode, (iii) we prove order optimal regret and efficient scaling with dimensions, thereby providing the first guaranteed RL algorithm for a wide class of POMDPs.

The estimation of the POMDP is carried out via spectral methods which involve decomposition of certain moment tensors computed from data. This learning algorithm is interleaved  with  the optimization of the planning policy  using an exploration--exploitation strategy inspired by the \ucrl method for MDPs~\citep{ortner2007logarithmic,jaksch2010near-optimal}. The resulting algorithm, called \smucrl (\textit{Spectral Method for Upper-Confidence Reinforcement Learning}), runs through episodes of variable length, where the agent follows a fixed policy until enough data are collected and then it updates the current policy according to the estimates of the POMDP parameters and their accuracy. Throughout the paper we focus on the estimation and exploration--exploitation aspects of the algorithm, while we assume access to a \textit{planning oracle} for the class of memoryless policies (i.e., policies directly mapping observations to a distribution over actions).\footnote{This assumption is common in many works in bandit and RL literature~(see e.g.,~\citet{abbasi2011regret} for linear bandit and \citet{chen2013combinatorial} in combinatorial bandit), where the focus is on the exploration--exploitation strategy rather than the optimization problem.}

\textbf{Theoretical Results.}
We prove the following learning result. For the full details see Thm.~\ref{thm:estimates} in  Sect.~\ref{s:learning.pomdp}.

\begin{theorem*}\textbf{(Informal Result on Learning POMDP Parameters)}\label{thm:estimates-informal}
Let $M$ be a POMDP with $X$ states, $Y$ observations, $A$ actions, $R$ rewards, and $Y > X$, and characterized by densities $f_T(x'|x,a)$, $f_O(y|x)$, and $f_R(r|x,a)$ defining state transition, observation, and the reward models. Given a sequence of observations, actions, and rewards generated by executing a memoryless policy where each action $a$ is chosen $N(a)$ times, there exists a spectral method which returns estimates $\wh{f}_T$, $\wh{f}_O$, and $\wh{f}_R$ that, under suitable assumptions on the POMDP, the policy, and the number of samples, satisfy
\begin{align*}
\|\wh{f}_O(\cdot|x) \!-\! f_O(\cdot|x)\|_1 &\leq \wt{O}\bigg(\sqrt{\frac{YR}{N(a)}} \bigg),\\
\|\wh{f}_R(\cdot|x,a) - f_R(\cdot|x,a)\|_1 &\leq \wt{O}\bigg(\sqrt{\frac{YR}{N(a)}}\bigg),\\
\| \wh{f}_T(\cdot|x,a) \!-\! f_T(\cdot|x,a) \|_2 &\leq \wt{O}\bigg( \sqrt{\frac{YRX^2}{N(a)}} \bigg),
\end{align*}
with high probability, for any state $x$ and any action $a$.
\end{theorem*}

This result shows the consistency of the estimated POMDP parameters and it also provides explicit confidence intervals.

By employing the above learning result in a \ucrl framework, we prove the following bound on the regret $\text{Reg}_N$ w.r.t.\ the optimal memoryless policy. For full details see Thm.~\ref{thm:regret} in  Sect.~\ref{s:learning}.

\begin{theorem*}\textbf{(Informal Result on Regret Bounds)}\label{thm:regret-informal}
Let $M$ be a POMDP with $X$ states, $Y$ observations, $A$ actions, and $R$ rewards, with a diameter $D$ defined as
\begin{align*}
D:=\max_{x,x'\in\X,a,a'\in\A}\min_{\pi}\mathbb{E}\big[\tau(x',a'|x,a; \pi)\big],
\end{align*}
i.e., the largest mean passage time between any two state-action pairs in the POMDP using a memoryless policy $\pi$ mapping observations to actions. If \smucrl is run over $N$ steps using the confidence intervals of Thm.~\ref{thm:estimates}, under suitable assumptions on the POMDP, the space of policies, and the number of samples, we have
\begin{align*}
\text{Reg}_N\leq \wt{O}\Big(DX^{3/2}\sqrt{AYRN}\Big),
\end{align*}
with high probability.

\end{theorem*}

The above  result shows that despite the complexity of estimating the POMDP parameters from noisy observations of hidden states, the regret of \smucrl is similar to the case of MDPs, where the regret of \ucrl scales as $\wt{O}(D_{\text{MDP}}X\sqrt{AN})$. The regret is order-optimal, since $\wt{O}(\sqrt{N})$ matches the lower bound for MDPs.

Another interesting aspect is that the diameter of the POMDP is a natural extension of the MDP case. While $D_{\text{MDP}}$ measures the mean passage time using state--based policies (i.e., a policies mapping \textit{states} to actions), in POMDPs policies cannot be defined over states but rather on observations and this naturally translates into the definition of the diameter $D$. More details on other problem-dependent terms in the bound are discussed in Sect.~\ref{s:learning}.

The derived regret bound is with respect to the best memoryless (stochastic) policy for the given POMDP. Indeed, for a general POMDP, the optimal policy need not be memoryless. However, finding the optimal policy is    uncomputable for infinite horizon regret minimization~\citep{madani1998computability}. Instead memoryless policies have shown good performance in practice (see the Section on related work). Moreover, for the class of so-called {\em contextual MDP}, a special class of POMDPs, the optimal policy is also  memoryless~\citep{krishnamurthy2016contextual-mdps}.

\paragraph{Analysis of the learning algorithm.} The learning results in Thm.~\ref{thm:estimates} are based on
spectral tensor decomposition methods, which have been previously used for consistent estimation of a wide class of LVMs~\citep{anandkumar2014tensor}. This is in contrast with traditional learning methods, such as expectation-maximization (EM)~\citep{dempster1977maximum}, that have no consistency guarantees and may converge to local optimum which is arbitrarily bad.

While spectral methods have been previously employed in sequence modeling such as in HMMs ~\citep{anandkumar2014tensor}, by representing it as multiview model, their application to POMDPs is not trivial. In fact, unlike the HMM, the consecutive observations of a POMDP are no longer conditionally independent, when conditioned on the hidden state of middle {\em view}. This is because the decision (or the action) depends on the observations themselves. By limiting to memoryless policies, we can control the range of this dependence, and by conditioning on the actions, we show that we can obtain conditionally independent {\em views}. As a result, starting with samples collected along a trajectory generated by a fixed policy, we can construct a multi-view model and use the tensor decomposition method on each action separately, estimate the parameters of the POMDP, and define confidence intervals.

While the proof follows similar steps as in previous works on spectral methods \citep[e.g., HMMs][]{anandkumar2014tensor}, here we extend concentration inequalities for dependent random variables to matrix valued functions by combining the results of~\citet{kontorovich2008concentration} with the matrix Azuma's inequality of~\citet{tropp2012user}. This allows us to remove the usual assumption that the samples are generated from the stationary distribution of the current policy. This is particularly important in our case since the policy changes at each episode and we can avoid discarding the initial samples and waiting until the corresponding Markov chain converged (i.e., the \textit{burn-in} phase).

The condition that the POMDP has more observations than states ($Y > X$) follows from standard non-degeneracy conditions to apply the spectral method. This corresponds to considering POMDPs where the underlying MDP is defined over a few number of states (i.e., a low-dimensional space) that can produce a large number of noisy observations. This is common in applications such as spoken-dialogue systems~\citep{atrash2006efficient,png2012building} and medical applications~\citep{hauskrecht2000planning}. We also show how this assumption can be relaxed and the result can be applied to a wider family of POMDPs.

%%%%%%%%%%%%%%%%%%%%%%%%%%%%%%%%%%%%%%%%%%%%%%%%%%%%%%%%%%%%

\paragraph{Analysis of the exploration--exploitation strategy.}
\smucrl applies the popular \textit{optimism-in-face-of-uncertainty} principle\footnote{This principle has been successfully used in a wide number of exploration--exploitation problems ranging from multi-armed bandit~\citep{auer2002finite-time}, linear contextual bandit~\citep{abbasi-yadkori2011improved}, linear quadratic control~\citep{abbasi2011regret}, and reinforcement learning~\citep{ortner2007logarithmic,jaksch2010near-optimal}. } to the confidence intervals of the estimated POMDP and compute the optimal policy of the most optimistic POMDP in the admissible set. This \textit{optimistic} choice provides a smooth combination of the exploration encouraged by the confidence intervals (larger confidence intervals favor uniform exploration) and the exploitation of the estimates of the POMDP parameters.

While the algorithmic integration is rather simple, its analysis is not trivial. The spectral method cannot use samples generated from different policies and the length of each episode should be carefully tuned to guarantee that estimators improve at each episode. Furthermore, the analysis requires redefining the notion of diameter of the POMDP. In addition, we carefully bound the various perturbation terms in order to obtain efficient scaling in terms of dimensionality factors.
 
Finally, in the Appendix~\ref{s:experiments}, we report preliminary synthetic experiments that demonstrate superiority of our method over existing RL methods such as Q-learning and \ucrl for MDPs, and also over purely exploratory methods such as random sampling, which randomly chooses actions independent of the observations. \smucrl converges much faster and to a better solution. The solutions relying on the MDP assumption, directly work in the (high) dimensional observation space and perform poorly. In fact, they can even be worse than the random sampling policy baseline. In contrast, our method aims to find the lower dimensional latent space to derive the policy and this allows \ucrl to find a much better memoryless policy with vanishing regret.

It is worth noting that, in general, with slight changes on the learning set up, one can come up with new algorithms to learn different POMDP models with, slightly, same upper confidence bounds. Moreover, after applying memoryless policy and collecting sufficient number of samples, when the model parameters are learned very well, one can do the planing on the belief space, and get memory dependent policy, therefore improve the performance even further.  
%%%%%%%%%%%%%%%%%%%%%%%%%%%%%%%%%%%%%%%%%%%%%%%%%%%%%%%%%%%%%%%%%%%%%

\subsection{Related Work}\label{ss:related}
In last few decades, MDP has been widely studied ~\citep{kearns2002near,brafman2003r,bartlett2009regal:,jaksch2010near-optimal} in different setting. Even for the large state space MDP, where the classical approaches are not scalable, \cite{kocsis2006bandit} introduces MDP Monte-Carlo planning tree which is one of the few viable approaches to find the near-optimal policy. In addition, for special class of MDPs, Markov Jump Affine Model, when the action space is continuous, \citep{baltaoglu2016online} proposes an order optimal learning policy. 

While RL in MDPs has been widely studied, the design of effective exploration--exploration strategies in POMDPs is still relatively unexplored. \citet{ross2007bayes} and \citet{poupart2008model-based} propose to integrate the problem of estimating the belief state into a model-based Bayesian RL approach, where a distribution over possible MDPs is updated over time. The proposed algorithms are such that the Bayesian inference can be done accurately and at each step, a POMDP is sampled from the posterior and the corresponding optimal policy is executed. While the resulting methods implicitly balance exploration and exploitation, no theoretical guarantee is provided about their regret and their algorithmic complexity requires the introduction of approximation schemes for both the inference and the planning steps. An alternative to model-based approaches is to adapt model-free algorithms, such as Q-learning, to the case of POMDPs. \citet{perkins2002reinforcement} proposes a Monte-Carlo approach to action-value estimation and it shows convergence to locally optimal memoryless policies. While this algorithm has the advantage of being computationally efficient, local optimal policies may be arbitrarily suboptimal and thus suffer a linear regret.

An alternative approach to solve POMDPs is to use policy search methods, which avoid estimating value functions and directly optimize the performance by searching in a given policy space, which usually contains memoryless policies (see e.g.,~\citep{ng2000pegasus},\citep{baxter2001infinite-horizon},\citep{poupart2003bounded,bagnell2004policy}). Beside its practical success in offline problems, policy search has been successfully integrated with efficient exploration--exploitation techniques and shown to achieve small regret~\citep{gheshlaghi-azar2013regret,gheshlaghi-azar2014resource-efficient}. Nonetheless, the performance of such methods is severely constrained by the choice of the policy space, which may not contain policies with good performance. Another approach to solve POMDPs is proposed by \citep{guo2016pac}. In this work, the agent randomly chooses actions independent of the observations and rewards. The agent executes random policy until it collects sufficient number of samples and then estimates the model parameters given collected information. The authors propose Probably Approximately Correct (PAC) framework for RL in POMDP setting and shows polynomial sample complexity for learning of the model parameters. During learning phase, they defines the induced Hidden Markov Model and applies random policy to capture different aspects of the model, then in the planing phase, given the estimated model parameters, they compute the optimum policy so far. In other words, the proposed algorithm explores the environment sufficiently enough and then exploits this exploration to come up with a optimal policy given estimated model. In contrast, our method considers RL of POMDPs in an episodic learning framework.

Matrix decomposition methods have been previously used in the more general setting of predictive state representation (PSRs)~\citep{boots2011closing} to reconstruct the structure of the dynamical system. Despite the generality of PSRs, the proposed model relies on strong assumptions on the dynamics of the system and it does not have any theoretical guarantee about its performance. \citet{gheshlaghi-azar2013sequential} used spectral tensor decomposition methods in the multi-armed bandit framework to identify the hidden generative model of a sequence of bandit problems and showed that this may drastically reduce the regret. Recently, \citep{hamilton2014efficient} introduced compressed PSR (CPSR) method to reduce the computation cost in PSR by exploiting the advantages in dimensionality reduction, incremental matrix decomposition, and compressed sensing. In this work, we take these ideas further by considering more powerful tensor decomposition techniques.  

\citet{krishnamurthy2016contextual-mdps} recently analyzed the problem of learning in contextual-MDPs and  proved sample complexity bounds polynomial in the capacity of the policy space, the number of states, and the horizon. While their objective is to minimize the regret over a finite horizon, we instead consider the infinite horizon problem. It is an open question to analyze and modify our spectral \ucrl algorithm for the finite horizon problem. As stated earlier, contextual MDPs are a special class of POMDPs for which memoryless policies are optimal. While they assume that the samples are drawn from a  contextual MDP, we can handle a much more general class of POMDPs, and we minimize regret with respect to the best memoryless policy for the given POMDP. 

Finally, a related problem is considered by~\citet{ortner2014selecting}, where a series of possible representations based on observation histories is available to the agent but only one of them is actually Markov. A \ucrl-like strategy is adopted and shown to achieve near-optimal regret.

In this paper, we focus on the learning problem, while we consider access to an optimization oracle to compute the optimal memoryless policy.
The problem of planning in general POMDPs is intractable (PSPACE-complete for finite horizon~\citep{papadimitriou1987the-complexity} and uncomputable for infinite horizon~\citep{madani1998computability}).%\todoaout{What does ``uncomputable'' mean? I guess it should belong to a specific class of complexity.}

Many exact, approximate, and heuristic methods have been proposed to compute the optimal policy (see~\citet{spaan2012partially} for a recent survey). An alternative approach is to consider memoryless policies which directly map observations (or a finite history) to actions~\citep{littman1994memoryless,singh1994learning,li2011finding}. While deterministic policies may perform poorly, stochastic memoryless policies are shown to be near-optimal in many domains~\citep{barto1983neuronlike,loch1998using,williams1998experimental} and even optimal in the specific case of contextual MDPs~\citep{krishnamurthy2016contextual-mdps}. Although computing the optimal stochastic memoryless policy is still NP-hard~\citep{littman1994memoryless}, several model-based and model-free methods are shown to converge to nearly-optimal policies with polynomial complexity under some conditions on the POMDP~\citep{jaakkola1995reinforcement,li2011finding}. In this work, we employ memoryless policies and prove regret bounds for reinforcement learning of POMDPs. The above works suggest that focusing to memoryless policies may not be a restrictive limitation in practice.

%\aacomment{let us call everywhere planning oracle to make it explicit}

\subsection{Paper Organization}\label{ss:organization}
The paper is organized as follows. Sect.~\ref{s:preliminaries} introduces the notation (summarized also in a table in Sect.~\ref{s:notation}) and the technical assumptions concerning the POMDP and the space of memoryless policies that we consider. Sect.~\ref{s:learning.pomdp} introduces the spectral method for the estimation of POMDP parameters together with Thm.~\ref{thm:estimates}. In Sect.~\ref{s:learning}, we outline \smucrl where we integrate the spectral method into an exploration--exploitation strategy and we prove the regret bound of Thm.~\ref{thm:regret}. Sect.~\ref{s:conclusions} draws conclusions and discuss possible directions for future investigation. The proofs are reported in the appendix together with  preliminary empirical results showing the effectiveness of the proposed method.

%%%%%%%%%%%%%%%%%%%%%%%%%%%%%%%%%%%%%%%%%%%%%%%%%%%%%%%%%%%%%%%%
%%%%%%%%%%%%%%%%%%%%%%%%%%%%%%%%%%%%%%%%%%%%%%%%%%%%%%%%%%%%%%%%
%% PRELIMINARIES
%%%%%%%%%%%%%%%%%%%%%%%%%%%%%%%%%%%%%%%%%%%%%%%%%%%%%%%%%%%%%%%%
%%%%%%%%%%%%%%%%%%%%%%%%%%%%%%%%%%%%%%%%%%%%%%%%%%%%%%%%%%%%%%%%

% !TEX root = master.tex

%%%%%%%%%%%%%%%%%%%%%%%%%%%%%%%%%%%%%%%%%%%%%%%%%%%%%%%%%%%%%%%%
%%%%%%%%%%%%%%%%%%%%%%%%%%%%%%%%%%%%%%%%%%%%%%%%%%%%%%%%%%%%%%%%
%% PRELIMINARIES
%%%%%%%%%%%%%%%%%%%%%%%%%%%%%%%%%%%%%%%%%%%%%%%%%%%%%%%%%%%%%%%%
%%%%%%%%%%%%%%%%%%%%%%%%%%%%%%%%%%%%%%%%%%%%%%%%%%%%%%%%%%%%%%%%

\section{Preliminaries}\label{s:preliminaries}

\begin{figure}[t!]
\small
\begin{center}
\begin{psfrags}
\psfrag{x1}[][1]{$x_t$}
\psfrag{x2}[][1]{$x_{t+1}$}
\psfrag{x3}[][1]{$x_{t+2}$}
\psfrag{y1}[][1]{$\vec{y}_t$}
\psfrag{y2}[][1]{$\vec{y}_{t+1}$}
\psfrag{r1}[][1]{$\vec{r}_t$}
\psfrag{r2}[][1]{$\vec{r}_{t+1}$}
\psfrag{a1}[][1]{$a_{t}$}
\psfrag{a2}[][1]{$a_{t+1}$}
\includegraphics[width=0.6\textwidth,natwidth=810,natheight=642,trim={-1cm 0cm 0 11cm},clip]{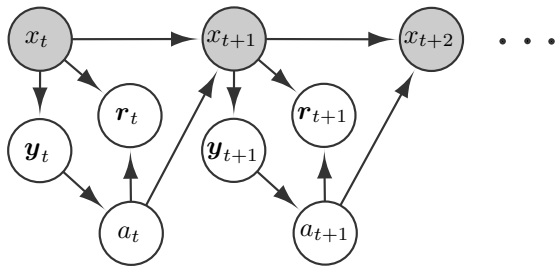}
\end{psfrags}
\end{center}
\vspace{-0.2in}
\caption{Graphical model of a POMDP under memoryless policies.}
\label{fig:pds2}
\end{figure}

A POMDP $M$ is a tuple $\langle \X, \A, \Y, \R, f_T, f_R, f_O\rangle$, where $\X$ is a finite state space with cardinality $|\X|=X$, $\A$ is a finite action space with cardinality $|\A|=A$, $\Y$ is a finite observation space with cardinality $|\Y|=Y$, and $\R$ is a finite reward space with cardinality $|\R|=R$ and largest reward $r_{\max}$. For notation convenience, we use a vector notation for the elements in $\Y$ and $\R$, so that $\vec{y}\in\Re^Y$ and $\vec{r}\in\Re^R$ are indicator vectors with entries equal to $0$ except a $1$ in the position corresponding to a specific element in the set (e.g., $\vec{y} = \vec{e}_n$ refers to the $n$-th element in $\Y$). We use $i,j\in[X]$ to index states, $k,l\in[A]$ for actions, $m\in[R]$ for rewards, and $n\in[Y]$ for observations. Finally, $f_T$ denotes the transition density, so that $f_T(x'|x,a)$ is the probability of transition to $x'$ given the state-action pair $(x,a)$, $f_R$ is the reward density, so that $f_R(\vec{r}|x,a)$ is the probability of receiving the reward in $\R$ corresponding to the value of the indicator vector $\vec{r}$ given the state-action pair $(x,a)$, and $f_O$ is the observation density, so that $f_O(\vec{y}|x)$ is the probability of receiving the observation in $\Y$ corresponding to the indicator vector $\vec{y}$ given the state $x$. Whenever convenient, we use tensor forms for the density functions such that
\begin{align*}
T_{i,j,l} &= \Prob[x_{t+1}=j | x_t = i, a_t = l] = f_T(j|i,l), & s.t.~~T\in\Re^{X\times X\times A}\\
O_{n,i} &= \Prob[\vec{y}=\vec{e}_n | x = i] = f_O(\vec{e}_n|i),& s.t.~~ O\in\Re^{Y\times X}\\
\Gamma_{i,l,m} &= \Prob[\vec{r} = \vec{e}_m | x = i, a = l] = f_R(\vec{e}_m|i,l),& s.t.~~\Gamma \in\Re^{X\times A\times R}.
\end{align*}
We also denote by $T_{:, j, l}$ the fiber (vector) in $\Re^{X}$ obtained by fixing the arrival state $j$ and action $l$ and by $T_{:,:,l}\in\Re^{X\times X}$ the transition matrix between states when using action $l$. The graphical model associated to the POMDP is illustrated in Fig.~\ref{fig:pds2}. 

We focus on stochastic memoryless policies which map observations to actions and for any policy $\pi$ we denote by $f_\pi(a|\vec{y})$ its density function. We denote by $\calP$ the set of all stochastic memoryless policies that have a non-zero probability to explore all actions:
\begin{align*}
\calP = \{\pi: \min_{\vec{y}} \min_{a} f_\pi(a|\vec{y}) > \pi_{\min}\}.
\end{align*}
Acting according to a policy $\pi$ in a POMDP $M$ defines a Markov chain characterized by a transition density
\begin{align*}
f_{T,\pi}(x'|x) = \sum_a \sum_{\vec{y}} f_\pi(a|\vec{y}) f_O(\vec{y}|x) f_T(x'|x,a),
\end{align*}
and a stationary distribution $\omega_\pi$ over states such that $\omega_\pi(x) = \sum_{x'} f_{T,\pi}(x'|x)\omega_\pi(x')$. The expected average reward performance of a policy $\pi$ is
\begin{align*}
\eta(\pi; M) = \sum_x \omega_\pi(x) \wb{r}_\pi(x),
\end{align*}

where $\wb{r}_\pi(x)$ is the expected reward of executing policy $\pi$ in state $x$ defined as
\begin{align*}
\wb{r}_\pi(x) = \sum_a \sum_{\vec{y}} f_O(\vec{y}|x) f_\pi(a|\vec{y}) \wb{r}(x,a),
\end{align*}
and $\wb{r}(x,a) = \sum_{r} r f_R(r|x,a)$ is the expected reward for the state-action pair $(x,a)$.
The best stochastic memoryless policy in $\calP$ is $\pi^+ = \displaystyle\arg\max_{\pi\in\calP} \eta(\pi; M)$ and we denote by $\eta^+ = \eta(\pi^+; M)$ its average reward.\footnote{We use $\pi^+$ rather than $\pi^*$ to recall the fact that we restrict the attention to $\calP$ and the actual optimal policy for a POMDP in general should be constructed on the belief-MDP.} Throughout the paper we assume that we have access to an optimization oracle returning the optimal policy $\pi^+$ in $\calP$ for any POMDP $M$. We need the following assumptions on the POMDP $M$.

\begin{assumption}[Ergodicity]\label{asm:ergodicity}
For any policy $\pi\in\calP$, the corresponding Markov chain $f_{T,\pi}$ is ergodic, so $\omega_\pi(x)>0$ for all states $x\in\X$.
\end{assumption}

We further characterize the Markov chains that can be generated by the policies in $\calP$. For any ergodic Markov chain with stationary distribution $\omega_\pi$, let $f_{1\rightarrow t}(x_t|x_1)$ by the distribution over states reached by a policy $\pi$ after $t$ steps starting from an initial state $x_1$. The inverse mixing time $\rho_{\text{mix},\pi}(t)$ of the chain is defined as
\begin{align*}
\rho_{\text{mix},\pi}(t)=\sup_{x_1}\left\|f_{1\rightarrow t}(\cdot|x_1)-\omega_\pi\right\|_{\text{TV}},
\end{align*}
where $\| \cdot\|_{\text{TV}}$ is the total-variation metric. \citet{kontorovich2014uniform} show that for any ergodic Markov chain the mixing time can be bounded as
\begin{align*}
\rho_{\text{mix},\pi}(t)\leq G(\pi)\theta^{t-1}(\pi),
\end{align*}
where $1\leq G(\pi)<\infty$ is the \textit{geometric ergodicity} and $0\leq \theta(\pi)<1$ is the \textit{contraction coefficient} of the Markov chain generated by policy $\pi$.

\begin{assumption}[Full Column-Rank]\label{asm:observation}
The observation matrix $O\in\Re^{Y\times X}$ is full column rank.
\end{assumption}
and define

This assumption guarantees that the distribution $f_O(\cdot|x)$ in a state $x$ (i.e., a column of the matrix $O$) is not the result of a linear combination of the distributions over other states. We show later that this is a sufficient condition to recover $f_O$ since it makes all states \textit{distinguishable} from the observations and it also implies that $Y \geq X$. Notice that POMDPs have been often used in the opposite scenario ($X \gg Y$) in applications such as robotics, where imprecise sensors prevents from distinguishing different states. On the other hand, there are many domains in which the number of observations may be much larger than the set of states that define the dynamics of the system. A typical example is the case of spoken dialogue systems~\citep{atrash2006efficient,png2012building}, where the observations (e.g., sequences of words uttered by the user) is much larger than the state of the conversation (e.g., the actual meaning that the user intended to communicate). A similar scenario is found in medical applications~\citep{hauskrecht2000planning}, where the state of a patient (e.g., sick or healthy) can produce a huge body of different (random) observations. In these problems it is crucial to be able to reconstruct the underlying small state space and the actual dynamics of the system from the observations.

\begin{assumption}[Invertible]\label{asm:transition}
For any action $a\in[A]$, the transition matrix $T_{:,:,a}\in\Re^{X\times X}$ is invertible.
\end{assumption}

Similar to the previous assumption, this means that for any action $a$ the distribution $f_T(\cdot|x,a)$ cannot be obtained as linear combination of distributions over other states, and it is a sufficient condition to be able to recover the transition tensor. Both Asm.~\ref{asm:observation} and~\ref{asm:transition} are strictly related to the assumptions introduced by~\citet{anandkumar2014tensor} for tensor methods in HMMs. In Sect.~\ref{s:learning} we discuss how they can be partially relaxed.

%%%%%%%%%%%%%%%%%%%%%%%%%%%%%%%%%%%%%%%%%%%%%%%%%%%%%%%%%%%%%%%%
%%%%%%%%%%%%%%%%%%%%%%%%%%%%%%%%%%%%%%%%%%%%%%%%%%%%%%%%%%%%%%%%
%% LEARNING POMDP
%%%%%%%%%%%%%%%%%%%%%%%%%%%%%%%%%%%%%%%%%%%%%%%%%%%%%%%%%%%%%%%%
%%%%%%%%%%%%%%%%%%%%%%%%%%%%%%%%%%%%%%%%%%%%%%%%%%%%%%%%%%%%%%%%

% !TEX root = master.tex

%%%%%%%%%%%%%%%%%%%%%%%%%%%%%%%%%%%%%%%%%%%%%%%%%%%%%%%%%%%%%%%%
%%%%%%%%%%%%%%%%%%%%%%%%%%%%%%%%%%%%%%%%%%%%%%%%%%%%%%%%%%%%%%%%
%% LEARNING POMDP
%%%%%%%%%%%%%%%%%%%%%%%%%%%%%%%%%%%%%%%%%%%%%%%%%%%%%%%%%%%%%%%%
%%%%%%%%%%%%%%%%%%%%%%%%%%%%%%%%%%%%%%%%%%%%%%%%%%%%%%%%%%%%%%%%

\section{Learning the Parameters of the POMDP}\label{s:learning.pomdp}

%\begin{figure}
%\tiny
%\hspace*{-0.9cm}
%\begin{psfrags}
%\psfrag{x1}[][1]{$x_i$}
%\psfrag{x2}[][1]{$x_{i+1}$}
%\psfrag{x3}[][1]{$x_{i+2}$}
%\psfrag{y1}[][1]{$\vecy_i$}
%\psfrag{y2}[][1]{$\vecy_{i+1}$}
%\psfrag{r1}[][1]{$r_i$}
%\psfrag{r2}[][1]{$r_{i+1}$}
%\psfrag{a1}[][1]{$a_{i}$}
%\psfrag{a2}[][1]{$a_{i+1}$}
%\psfrag{y3}[][1]{$\vecy_{i+2}$}
%\begin{subfigure}{.29\textwidth}
%\hspace*{1cm}
%  \centering
%  \includegraphics[width=.8\linewidth]{figs/HMM.eps}
%  \caption{HMM}
%  \label{fig:HMM}
%\end{subfigure}%
%\begin{subfigure}{.29\textwidth}
%\hspace*{-1cm}
%  \centering
%  \includegraphics[width=.8\linewidth]{figs/MV.eps}
%  \caption{Multi-View}
%  \label{fig:sfig2}
%\end{subfigure}
%\end{psfrags}
%\caption{HMM to Multi-View}
%\label{fig:HMM-MV}
%\end{figure}

In this section we introduce a novel spectral method to estimate the POMDP parameters $f_T$, $f_O$, and $f_R$. A stochastic policy $\pi$ is used to generate a trajectory $(\vec{y}_1, a_1, \vec{r}_1, \ldots, \vec{y}_N, a_N, \vec{r}_N)$ of $\TN$ steps. We need the following assumption that, together with Asm.~\ref{asm:ergodicity}, guarantees that all states and actions are constantly visited.

\begin{assumption}[Policy Set]\label{asm:stochastic.policy}
The policy $\pi$ belongs to $\calP$.
\end{assumption}

Similar to the case of HMMs, the key element to apply the spectral methods is to construct a multi-view model for the hidden states. Despite its similarity, the spectral method developed for HMM by~\cite{anandkumar2014tensor} cannot be directly employed here. In fact, in HMMs the state transition and the observations only depend on the current state. On the other hand, in POMDPs the probability of a transition to state $x'$ not only depends on $x$, but also on action $a$. Since the action is chosen according to a memoryless policy $\pi$ based on the current observation, this creates an indirect dependency of $x'$ on observation $\vec{y}$, which makes the model more intricate.

%%%%%%%%%%%%%%%%%%%%%%%%%%%%%%%%%%%%%%%%%%%%%%%%%%%%

\subsection{The multi-view model}\label{ss:multi-view}

We estimate POMDP parameters for each action $l\in [A]$ separately. Let $t\in [2,N-1]$ be a step at which $a_t=l$, we construct three views $(a_{t-1}, \vec{y}_{t-1}, \vec{r}_{t-1})$, $(\vec{y}_t, \vec{r}_t)$, and $(\vec{y}_{t+1})$ which all contain observable elements. As it can be seen in Fig.~\ref{fig:pds2}, all three views provide some information about the hidden state $x_t$ (e.g., the observation $\vec{y}_{t-1}$ triggers the action $a_{t-1}$, which influence the transition to $x_t$). A careful analysis of the graph of dependencies shows that conditionally on $x_t, a_t$ all the views are independent. For instance, let us consider $\vec{y}_{t}$ and $\vec{y}_{t+1}$. These two random variables are clearly dependent since $\vec{y}_t$ influences action $a_t$, which triggers a transition to $x_{t+1}$ that emits an observation $\vec{y}_{t+1}$. Nonetheless, it is sufficient to condition on the action $a_t = l$ to break the dependency and make $\vec{y}_t$ and $\vec{y}_{t+1}$ independent. Similar arguments hold for all the other elements in the views, which can be used to recover the latent variable $x_t$. More formally, we encode the triple $(a_{t-1},\vec{y}_{t-1},\vec{r}_{t-1})$ into a vector $\vec{v}_{1,t}^{(l)}\in \Re^{A\cdot Y\cdot R}$, so that view $\vec{v}_{1,t}^{(l)} = \vec{e}_s$ whenever $a_{t-1}=k$, $\vec{y}_{t-1}=\vec{e}_n$, and $\vec{r}_{t-1}=\vec{e}_m$ for a suitable mapping between the index $s\in\{1,\ldots,A\cdot Y\cdot R\}$ and the indices $(k,n,m)$ of the action, observation, and reward. Similarly, we proceed for $\vec{v}_{2,t}^{(l)}\in\Re^{Y \cdot R}$ and  $\vec{v}_{3,t}^{(l)}\in\Re^{Y}$. We introduce the three view matrices $V^{(l)}_\nu$ with $\nu\in\{1,2,3\}$ associated with action $l$ defined as $V^{(l)}_1\in\Re^{A\cdot Y\cdot R\times X}$, $V^{(l)}_2\in\Re^{Y\cdot R\times X}$, and $V^{(l)}_3\in\Re^{Y\times X}$ such that
\begin{align*}%\label{eq:multi.views}
[V_1^{(l)}]_{s,i}&=\Prob\big(\vec{v}_1^{(l)}=\vec{e}_s | x_2=i\big) = [V_1^{(l)}]_{(n,m,k),i} = \Prob\big(\vecy_1=\vec{e}_n,\vec{r}_1=\vec{e}_m,a_1=k|x_2=i\big),\\[0.05in]
[V_2^{(l)}]_{s,i}&=\Prob\big(\vec{v}_2^{(l)}=\vec{e}_s |x_2=i,a_2=l\big) = [V_2^{(l)}]_{(n',m'),i} = \Prob\big(\vecy_2=\vec{e}_{n'},\vec{r}_2=\vec{e}_{m'}|x_2=i,a_2=l\big),\\[0.05in]
[V_3^{(l)}]_{s,i}&=\Prob\big(\vec{v}_3^{(l)}=\vec{e}_s |x_2=i,a_2=l\big) = [V_3^{(l)}]_{n'',i} = \Prob\big(\vecy_3=\vec{e}_{n''}|x_2=i,a_2=l\big).
\end{align*}
%
%%
%\begin{align*}%\label{eq:multi.views}
%[V_1^{(l)}]_{s,i}&=\Prob\big(\vec{v}_1^{(l)}=\vec{e}_s | x_2=i\big) = [V_1^{(l)}]_{(n,m,k),i} \\
%&= \Prob\big(\vecy_1=\vec{e}_n,\vec{r}_1=\vec{e}_m,a_1=k|x_2=i\big),\\
%[V_2^{(l)}]_{s,i}&=\Prob\big(\vec{v}_2^{(l)}=\vec{e}_s |x_2=i,a_2=l\big) = [V_2^{(l)}]_{(n',m'),i}\\
%&= \Prob\big(\vecy_2=\vec{e}_{n'},\vec{r}_2=\vec{e}_{m'}|x_2=i,a_2=l\big),\\
%[V_3^{(l)}]_{s,i}&=\Prob\big(\vec{v}_3^{(l)}=\vec{e}_s |x_2=i,a_2=l\big) = [V_3^{(l)}]_{n'',i}\\
%&= \Prob\big(\vecy_3=\vec{e}_{n''}|x_2=i,a_2=l\big).
%\end{align*}
%%
In the following we denote by $\mu^{(l)}_{\nu,i} = [V_\nu^{(l)}]_{:,i}$ the $i$th column of the matrix $V_\nu^{(l)}$ for any $\nu\in\{1,2,3\}$. Notice that Asm.~\ref{asm:observation} and Asm.~\ref{asm:transition} imply that all the view matrices are full column rank. As a result, we can construct a multi-view model that relates the spectral decomposition of the second and third moments of the (modified) views with the columns of the third view matrix.

%\aacomment{better to put this as a lemma and move derivations to appendix. Many COLT reviewers will know this result. and if they don't they can refer to it in the proof section. Same with recovering parameters of POMDP model. that should be the first lemma saying we can do it and then giving derivations as part of the proof. This is a much better way of explaining.}

\begin{proposition}[Thm. 3.6 in~\citep{anandkumar2014tensor}]\label{lem:spectral.decomposition}
Let $K_{\nu,\nu'}^{(l)} = \E\big[\vec{v}^{(l)}_\nu \otimes \vec{v}^{(l)}_{\nu'}\big]$ be the correlation matrix between views $\nu$ and $\nu'$ and $K^\dagger$ is its pseudo-inverse. We define a modified version of the first and second views as
\begin{align}\label{eq:mod.views}
\wt{\vec{v}}_1^{(l)}:=K_{3,2}^{(l)}(K_{1,2}^{(l)})^{\dagger}\vec{v}_1^{(l)}, \quad \wt{\vec{v}}_2^{(l)}:=K_{3,1}^{(l)}(K_{2,1}^{(l)})^{\dagger}\vec{v}_2^{(l)}.
\end{align}
Then the second and third moment of the modified views have a spectral decomposition as
\begin{align}
M_2^{(l)} &\!\!= \E\big[\wt{\vec{v}}_1^{(l)} \!\otimes \wt{\vec{v}}_2^{(l)}\big] \!=\! \sum_{i=1}^X \omega_\pi^{(l)}(i) \mu^{(l)}_{3,i} \otimes \mu^{(l)}_{3,i},\label{eq:m2}\\
M_3^{(l)} &\!\!= \E\big[\wt{\vec{v}}_1^{(l)} \!\otimes \wt{\vec{v}}_2^{(l)} \!\otimes \vec{v}_3^{(l)}\big] \!=\!\sum_{i=1}^X \omega_\pi^{(l)}(i) \mu^{(l)}_{3,i} \otimes \mu^{(l)}_{3,i} \otimes \mu^{(l)}_{3,i},\label{eq:m3}
\end{align}
where $\otimes$ is the tensor product and $\omega_\pi^{(l)}(i) = \Prob[x=i|a=l]$ is the state stationary distribution of $\pi$ conditioned on action $l$ being selected by policy $\pi$.
\end{proposition}

%Although the views are conditionally independent given the middle state and action, we cannot apply the multi-view spectral method yet, because constructing the third moment of the process would lead to a non-symmetric tensor. For this reason, we first need to \textit{symmetrize} the process

Notice that under Asm.~\ref{asm:ergodicity} and~\ref{asm:stochastic.policy}, $\omega_\pi^{(l)}(i)$ is always bounded away from zero.
Given $M_2^{(l)}$ and $M_3^{(l)}$ we can recover the columns of the third view $\mu^{(l)}_{3,i}$ directly applying the standard spectral decomposition method of~\citet{anandkumar2012method}.
% which first uses $M_2^{(l)}$ to whiten the tensor $M_3^{(l)}$ and then proceeds through successive steps of power iteration method and deflation to recover its eigenvalues and eigenvectors. Finally, $\mu^{(l)}_{3,i}$ is recovered by inversion of the whitening matrix.
We need to recover the other views from $V_3^{(l)}$. From the definition of modified views in Eq.~\ref{eq:mod.views} we have
\begin{align}\label{eq:expect.mod.views}
\begin{aligned}
\mu^{(l)}_{3,i} = \E\big[\wt{\vec{v}}_1 | x_2 = i,a_2=l \big] = K_{3,2}^{(l)}(K_{1,2}^{(l)})^{\dagger}\E\big[\vec{v}_1| x_2 = i,a_2=l\big] = K_{3,2}^{(l)}(K_{1,2}^{(l)})^{\dagger}\mu^{(l)}_{1,i}, \\ \mu^{(l)}_{3,i} = \E\big[\wt{\vec{v}}_2 | x_2 = i,a_2=l \big] = K_{3,1}^{(l)}(K_{2,1}^{(l)})^{\dagger}\E\big[\vec{v}_2| x_2 = i,a_2=l\big] = K_{3,1}^{(l)}(K_{2,1}^{(l)})^{\dagger}\mu^{(l)}_{2,i}.
\end{aligned}
\end{align}
Thus, it is sufficient to invert (pseudo invert) the two equations above to obtain the columns of both the first and second view matrices. This process could be done in any order, e.g., we could first estimate the second view by applying a suitable symmetrization step (Eq.~\ref{eq:mod.views}) and recovering the first and the third views by reversing similar equations to Eq.~\ref{eq:expect.mod.views}. On the other hand, we cannot repeat the symmetrization step multiple times and estimate the views independently (i.e., without inverting Eq.~\ref{eq:expect.mod.views}). In fact, the estimates returned by the spectral method are consistent ``up to a suitable permutation'' on the indexes of the states. While this does not pose any problem in computing one single view, if we estimated two views independently, the permutation may be different, thus making them non-consistent and impossible to use in recovering the POMDP parameters. On the other hand, estimating first one view and recovering the others by inverting Eq.~\ref{eq:expect.mod.views} guarantees the consistency of the labeling of the hidden states.

%%%%%%%%%%%%%%%%%%%%%%%%%%%%%%%%%%%%%%%%%%%%%%%%%%%%

\subsection{Recovery of POMDP parameters}\label{ss:recovery}

Once the views $\{V_\nu^{(l)}\}_{\nu=2}^3$ are computed from $M_2^{(l)}$ and $M_3^{(l)}$, we can derive $f_T$, $f_O$, and $f_R$. In particular, all parameters of the POMDP can be obtained by manipulating the second and third view as illustrated in the following lemma.

\begin{lemma}\label{lem:pomdp.parameters}
Given the views $V_2^{(l)}$ and $V_3^{(l)}$, for any state $i\in[X]$ and action $l\in[A]$, the POMDP parameters are obtained as follows. For any reward $m\in[R]$ the reward density is
\begin{align}\label{eq:rew.recovery}
f_R(\vec{e}_{m'}|i,l) = \sum_{n'=1}^Y [V_2^{(l)}]_{(n',m'),i};
\end{align}
for any observation $n'\in[Y]$ the observation density is
\begin{align}
f^{(l)}_O(\vec{e}_{n'}|i) = \sum_{m'=1}^R \frac{[V_2^{(l)}]_{(n',m'),i}}{f_\pi(l|\vec{e}_{n'})\rho(i,l)},\label{eq:obs.recovery}
\end{align}
with
\begin{align*}%\label{eq:rho.term}
\rho(i,l) = \sum_{m'=1}^R \sum_{n'=1}^Y \frac{[V_2^{(l)}]_{(n',m'),i}}{f_\pi(l|\vec{e}_{n'})} = \frac{1}{\Prob(a_2=l|x_2=i)}.
\end{align*}
Finally, each second mode of the transition tensor $T\in\Re^{X\times X\times A}$ is obtained as
\begin{align}\label{eq:transition.recovery.inv}
 [T]_{i,:,l} = O^\dagger [V_3^{(l)}]_{:,i},
\end{align}
where $O^\dagger$ is the pseudo-inverse of matrix observation $O$ and $f_T(\cdot|i,l) = [T]_{i,:,l}$.
\end{lemma}

In the previous statement we use $f^{(l)}_O$ to denote that the observation model is recovered from the view related to action $l$. While in the exact case, all $f^{(l)}_O$ are identical, moving to the empirical version leads to $A$ different estimates, one for each action view used to compute it. Among them, we will select the estimate with the better accuracy.

\begin{algorithm}[t]
\setstretch{1.1}
\begin{small}
\begin{algorithmic}
\STATE \textbf{Input:}
\STATE $\quad$ Policy density $f_\pi$, number of states $X$
\STATE $\quad$ Trajectory $\langle (\vec{y}_1,a_1,\vec{r}_1), (\vec{y}_2,a_2,\vec{r}_2),\ldots, (\vec{y}_N,a_N, \vec{r}_N)\rangle$

\STATE \textbf{Variables:}
\STATE $\quad$ Estimated second and third views $\wh{V}_2^{(l)}$, and $\wh{V}_3^{(l)}$ for any action $l\in[A]$
\STATE $\quad$ Estimated observation, reward, and transition models $\wh{f}_O$, $\wh{f}_R$, $\wh{f}_T$
\STATE
\FOR{$l=1,\ldots,A$}
\STATE Set $\T(l) = \{t \in [N-1]: a_t = l\}$ and $N(l) = |\T(l)|$
\STATE Construct views $\enspace\vec{v}_{1,t}^{(l)} = (a_{t-1},\vec{y}_{t-1},\vec{r}_{t-1}),$ $\enspace$ $\vec{v}_{2,t}^{(l)} = (\vec{y}_{t},\vec{r}_{t}),$ $\enspace$ $\vec{v}_{3,t}^{(l)} = \vec{y}_{t+1}\quad$ for any $t\in\T(l)$
\STATE Compute covariance matrices $\wh{K}_{3,1}^{(l)}$, $\wh{K}_{2,1}^{(l)}$, $\wh{K}_{3,2}^{(l)}$ as
$$
\wh{K}_{\nu,\nu'}^{(l)} = \frac{1}{N(l)} \sum_{t\in \T(l)} \vec{v}_{\nu,t}^{(l)} \otimes \vec{v}_{\nu',t}^{(l)};  \enspace \nu,\nu'\in\{1,2,3\}
$$
\STATE Compute modified views $\enspace \wt{\vec{v}}_{1,t}^{(l)}:=\wh{K}_{3,2}^{(l)}(\wh{K}_{1,2}^{(l)})^{\dagger}\vec{v}_1, \quad \wt{\vec{v}}_{2,t}^{(l)}:=\wh{K}_{3,1}^{(l)}(\wh{K}_{2,1}^{(l)})^{\dagger}\vec{v}_{2,t}^{(l)}\quad$ for any $t\in\T(l)$

\STATE Compute second and third moments
$$
\wh{M}_2^{(l)} = \frac{1}{N(l)} \sum_{t\in\T_l} \wt{\vec{v}}_{1,t}^{(l)} \otimes \wt{\vec{v}}_{2,t}^{(l)}, \quad
\wh{M}_3^{(l)} = \frac{1}{N(l)} \sum_{t\in\T_l} \wt{\vec{v}}_{1,t}^{(l)} \otimes \wt{\vec{v}}_{2,t}^{(l)} \otimes \vec{v}_{3,t}^{(l)}
$$
\STATE Compute $\wh{V}_3^{(l)}$ = \textsc{\small TensorDecomposition}($\wh{M}_2^{(l)}$, $\wh{M}_3^{(l)}$)
\STATE Compute $\wh{\mu}_{2,i}^{(l)} = \wh{K}_{1,2}^{(l)} (\wh{K}_{3,2}^{(l)})^\dagger \wh{\mu}_{3,i}^{(l)}\quad$ for any $i\in[X]$
\STATE Compute $\wh{f}(\vec{e}_m | i,l) = \sum_{n'=1}^Y [\wh{V}_2^{(l)}]_{(n',m), i}\quad$ for any $i \in [X]$, $m\in [R]$
\STATE Compute $\rho(i,l) = \sum_{m'=1}^R \sum_{n'=1}^Y \frac{[V_2^{(l)}]_{(n',m'),i}}{f_\pi(l | \vec{e}_{n'})}\quad$ for any $i \in [X]$, $n\in[Y]$
\STATE Compute $\wh{f}_O^{(l)}(\vec{e}_{n} | i) = \sum_{m'=1}^R \frac{[V_2^{(l)}]_{(n,m'),i}}{f_\pi(l | \vec{e}_{n}) \rho(i,l)}\quad$ for any $i \in [X]$, $n\in[Y]$
\ENDFOR
\STATE Compute bounds $\B_O^{(l)}$
\STATE Set $l^* = \arg\min_l \B_O^{(l)}$, $\wh{f}_O = \wh{f}_O^{l^*}$ and construct matrix $[\wh{O}]_{n,j} = \wh{f}_O(\vec{e}_n | j)$ 
\STATE Reorder columns of matrices $\wh{V}_2^{(l)}$ and $\wh{V}_3^{(l)}$ such that matrix $O^{(l)}$ and $O^{(l^*)}$ match, $\forall l\in [A]$\footnotemark
\FOR{$i\in[X]$, $l\in[A]$}
\STATE Compute $[T]_{i,:,l} = \wh{O}^\dagger [\wh{V}_3^{(l)}]_{:,i}$
\ENDFOR
\STATE \textbf{Return:} $\wh{f}_R$, $\wh{f}_T$, $\wh{f}_O$, $\B_R$, $\B_T$, $\B_O$
\end{algorithmic}
\end{small}
\caption{Estimation of the POMDP parameters. The routine \textsc{\small TensorDecomposition} refers to the spectral tensor decomposition method of~\citet{anandkumar2012method}.}
\label{alg:spectral.pomdp}
\end{algorithm}
\footnotetext{Each column of $O^{(l)}$ corresponds to $\ell 1$-closest column of $O^{(l^*)}$  }

\paragraph{Empirical estimates of POMDP parameters.}
In practice, $M_2^{(l)}$ and $M_3^{(l)}$ are not available and need to be estimated from samples. Given a trajectory of $N$ steps obtained executing policy $\pi$, let $\T(l) = \{t\in[2,N-1]: a_t = l\}$ be the set of steps when action $l$ is played, then we collect all the triples $(a_{t-1}, \vec{y}_{t-1}, \vec{r}_{t-1})$, $(\vec{y}_t, \vec{r}_t)$ and $(\vec{y}_{t+1})$ for any $t\in\T(l)$ and construct the corresponding views $\vec{v}_{1,t}^{(l)}$, $\vec{v}_{2,t}^{(l)}$, $\vec{v}_{3,t}^{(l)}$. Then we symmetrize the views using empirical estimates of the covariance matrices and build the empirical version of Eqs.~\ref{eq:m2} and~\ref{eq:m3} using $N(l) = |\T(l)|$ samples, thus obtaining
\begin{align}
\wh{M}_2^{(l)} = \frac{1}{N(l)} \sum_{t\in\T_l} \wt{\vec{v}}_{1,t}^{(l)} \otimes \wt{\vec{v}}_{2,t}^{(l)}, \quad\quad \wh{M}_3^{(l)} = \frac{1}{N(l)} \sum_{t\in\T_l} \wt{\vec{v}}_{1,t}^{(l)} \otimes \wt{\vec{v}}_{2,t}^{(l)} \otimes \vec{v}_{3,t}^{(l)}.\label{eq:est.m}
\end{align}%\aacomment{Is $N(l)$ not number of triplets where middle action is $l$. }
%
%%
%\begin{align}
%\wh{M}_2^{(l)} &= \frac{1}{N(l)} \sum_{t\in\T_l} \wt{\vec{v}}_{1,t}^{(l)} \otimes \wt{\vec{v}}_{2,t}^{(l)},\label{eq:est.m2}\\
%\wh{M}_3^{(l)} &= \frac{1}{N(l)} \sum_{t\in\T_l} \wt{\vec{v}}_{1,t}^{(l)} \otimes \wt{\vec{v}}_{2,t}^{(l)} \otimes \vec{v}_{3,t}^{(l)}.\label{eq:est.m3}
%\end{align}%\aacomment{Is $N(l)$ not number of triplets where middle action is $l$. }
%%
Given the resulting $\wh{M}_2^{(l)}$ and $\wh{M}_3^{(l)}$, we apply the spectral tensor decomposition method to recover an empirical estimate of the third view $\wh{V}_3^{(l)}$ and invert Eq.~\ref{eq:expect.mod.views} (using estimated covariance matrices) to obtain $\wh{V}_2^{(l)}$. Finally, the estimates $\wh{f}_O$, $\wh{f}_T$, and $\wh{f}_R$ are obtained by plugging the estimated views $\wh{V}_\nu$ in the process described in Lemma~\ref{lem:pomdp.parameters}.\\
Spectral methods indeed recover the factor matrices up to a permutation of the hidden states. In this case, since we separately carry out spectral decompositions for different actions, we recover permuted factor matrices. Since the observation  matrix $O$ is common to all the actions, we use it to align these decompositions.
Let's define $\deltaO$
\begin{align*}
\deltaO=:~\min_{x,x'}\|f_O(\cdot|x)-f_O(\cdot|x')\|_1
\end{align*}

Actually, $\deltaO$ is the minimum separability level of matrix $O$. When the estimation error over columns of matrix O are less than $4\deltaO$, then one can come over the permutation issue by matching columns of $O^{l}$ matrices. In T condition is reflected as a condition that the number of samples for each action has to be larger some number.

 The overall method is summarized in Alg.~\ref{alg:spectral.pomdp}. The empirical estimates of the POMDP parameters enjoy the following guarantee.

\begin{theorem}[Learning Parameters]\label{thm:estimates}
Let $\wh{f}_O$, $\wh{f}_T$, and $\wh{f}_R$ be the estimated POMDP models using a trajectory of $N$ steps. We denote by $\sigma^{(l)}_{\nu,\nu'} = \sigma_X(K_{\nu,\nu'}^{(l)})$ the smallest non-zero singular value of the covariance matrix $K_{\nu,\nu'}$, with $\nu,\nu'\in\{1,2,3\}$, and by $\sigma_{\min}(V_\nu^{(l)})$ the smallest singular value of the view matrix $V_{\nu}^{(l)}$ (strictly positive under Asm.~\ref{asm:observation} and Asm.~\ref{asm:transition}), and we define $\omega_{\min}^{(l)} = \min_{x\in\X} \omega_\pi^{(l)}(x)$ (strictly positive under Asm.~\ref{asm:ergodicity}). If for any action $l\in[A]$, the number of samples $N(l)$ satisfies the condition
\begin{align}\label{eq:nl.condition}
N(l) \geq \max\bigg\{\frac{4}{(\sigma_{3,1}^{(l)})^2}, \frac{16C_O^2YR}{{\lambda^{(l)}}^2\deltaO^2}, \left(\frac{G(\pi)\frac{2\sqrt{2}+1}{1-\theta(\pi)}}{{\omega^{(l)}_{\min}\min\limits_{\nu\in\{1,2,3\}}\lbrace\sigma^2_{\min}(V^{(l)}_\nu)\rbrace}}\right)^2\Theta^{(l)}\bigg\} \log\Big(\frac{2(Y^2+AYR)}{\delta}\Big),
\end{align}
with $\Theta^{(l)}$, defined in Eq~\ref{eq:nl.condition1}\footnote{We do not report the explicit definition of $\Theta^{(l)}$ here because it contains exactly the same quantities, such as $\omega^{(l)}_{\min}$, that are already present in other parts of the condition of Eq.~\ref{eq:nl.condition}.}, and $G(\pi),\theta(\pi)$ are the geometric ergodicity and the contraction coefficients of the corresponding Markov chain induced by $\pi$, then for any $\delta\in(0,1)$ and for any state $i\in[X]$ and action $l\in[A]$ we have
%
%\begin{small}
%
\begin{align}\label{eq:obs.bound}
\|\wh{f}_O^{(l)}(\cdot|i) \!-\! f_O(\cdot|i)\|_1 \leq \B_O^{(l)} \!:=\frac{ C_O}{\lambda^{(l)}} \sqrt{\frac{YR\log(1/\delta)}{N(l)}},
\end{align}
\begin{align}\label{eq:rew.bound}
\|\wh{f}_R(\cdot|i,l) - f_R(\cdot|i,l)\|_1 \leq \B_R^{(l)} := \frac{C_R}{\lambda^{(l)}} \sqrt{\frac{YR\log(1/\delta)}{N(l)}},
\end{align}
\begin{align}\label{eq:transition.bound}
\| \wh{f}_T(\cdot|i,l) \!-\! f_T(\cdot|i,l) \|_2 \leq \B_T^{(l)} := \frac{C_T }{\lambda^{(l)}} \sqrt{\frac{YRX^2\log(1/\delta)}{N(l)}},
\end{align}
%
%\end{small}
%
\noindent with probability $1-6(Y^2+AYR)A\delta$ (w.r.t. the randomness in the transitions, observations, and policy), where $C_O$, $C_R$, and $C_T$ are numerical constants and
\begin{align}\label{eq:lambda}
\lambda^{(l)} =\sigma_{\min}(O)(\pi_{\min}^{(l)})^2\sigma_{1,3}^{(l)} (\omega^{(l)}_{\min}\min\limits_{\nu\in\{1,2,3\}}\lbrace\sigma^2_{\min}(V^{(l)}_\nu)\rbrace)^{3/2}.
\end{align}
Finally, we denote by $\wh{f}_O$ the most accurate estimate of the observation model, i.e., the estimate $\wh{f}_O^{(l^*)}$ such that $l^* = \arg\min_{l\in[A]} \B_O^{(l)}$ and we denote by $\B_O$ its corresponding bound.
\end{theorem}

\noindent\paragraph{Remark 1 (consistency and dimensionality).}
All previous errors decrease with a rate $\wt{O}(1/\sqrt{N(l)})$, showing the consistency of the spectral method, so that if all the actions are repeatedly tried over time, the estimates converge to the true parameters of the POMDP. This is in contrast with EM-based methods which typically get stuck in local maxima and return biased estimators, thus preventing from deriving confidence intervals.

The bounds in Eqs.~\ref{eq:obs.bound}, \ref{eq:rew.bound}, \ref{eq:transition.bound} on $\wh{f}_O$, $\wh{f}_R$ and  $\wh{f}_T$ depend on $X$, $Y$, and $R$ (and the number of actions only appear in the probability statement). The bound in Eq.~\ref{eq:transition.bound} on  $\wh{f}_T$ is worse than the bounds for $\wh{f}_R$ and $\wh{f}_O$ in Eqs.~\ref{eq:obs.bound}, \ref{eq:rew.bound} by a factor of $X^2$.
 This seems unavoidable since $\wh{f}_R$ and $\wh{f}_O$ are the results of the manipulation of the matrix $V_2^{(l)}$ with $Y\cdot R$ columns, while estimating  $\wh{f}_T$ requires working on both $V_2^{(l)}$ and $V_3^{(l)}$. In addition, to come up with upper bound for $\wh{f}_T$, more complicated bound derivation is needed and it has one step of Frobenious norms to $\ell2$ norm transformation. The derivation procedure for  $\wh{f}_T$ is more complicated compared to $\wh{f}_O$ and $\wh{f}_R$ and adds the term $X$ to the final bound. (Appendix.~\ref{app:proof1})

\noindent\paragraph{Remark 2 (POMDP parameters and policy $\pi$).}
In the previous bounds, several terms depend on the structure of the POMDP and the policy $\pi$ used to collect the samples:
\begin{itemize}
\item $\lambda^{(l)}$ captures the main problem-dependent terms. While $K_{1,2}$ and $K_{1,3}$ are full column-rank matrices (by Asm.~\ref{asm:observation} and~\ref{asm:transition}), their smallest non-zero singular values influence the accuracy of the (pseudo-)inversion in the construction of the modified views in Eq.~\ref{eq:mod.views} and in the computation of the second view from the third using Eq.~\ref{eq:expect.mod.views}. Similarly the presence of $\sigma_{\min}(O)$ is justified by the pseudo-inversion of $O$ used to recover the transition tensor in Eq.~\ref{eq:transition.recovery.inv}. Finally, the dependency on the smallest singular values $\sigma^2_{\min}(V^{(l)}_\nu)$ is due to the tensor decomposition method (see App.~\ref{s:whitening} for more details).
\item A specific feature of the bounds above is that they do not depend on the state $i$ and the number of times it has been explored. Indeed, the inverse dependency on $\omega_{\min}^{(l)}$ in the condition on $N(l)$ in Eq.~\ref{eq:nl.condition} implies that if a state $j$ is poorly visited, then the empirical estimate of any other state $i$ may be negatively affected. This is in striking contrast with the fully observable case where the accuracy in estimating, e.g., the reward model in state $i$ and action $l$, simply depends on the number of times that state-action pair has been explored, even if some other states are never explored at all. This difference is intrinsic in the partial observable nature of the POMDP, where we reconstruct information about the states (i.e., reward, transition, and observation models) only from indirect observations. As a result, in order to have accurate estimates of the POMDP structure, we need to rely on the policy $\pi$ and the ergodicity of the corresponding Markov chain to guarantee that the whole state space is covered.
\item Under Asm.~\ref{asm:ergodicity} the Markov chain $f_{T,\pi}$ is ergodic for any $\pi\in\calP$. Since no assumption is made on the fact that the samples generated from $\pi$ being sampled from the stationary distribution, the condition on $N(l)$ depends on how fast the chain converge to $\omega_\pi$ and this is characterized by the parameters $G(\pi)$ and $\theta(\pi)$.
\item If the policy is deterministic, then some actions would not be explored at all, thus leading to very inaccurate estimations (see e.g., the dependency on $f_\pi(l|\vec{y})$ in Eq.~\ref{eq:obs.recovery}). The inverse dependency on $\pi_{\min}$ (defined in $\calP$) accounts for the amount of exploration assigned to every actions, which determines the accuracy of the estimates. Furthermore, notice that also the singular values $\sigma_{1,3}^{(l)}$ and $\sigma_{1,2}^{(l)}$ depend on the distribution of the views, which in turn is partially determined by the policy $\pi$.
\end{itemize}

Notice that the first two terms are basically the same as in the bounds for spectral methods applied to HMM~\citep{song2013nonparametric}, while the dependency on $\pi_{\min}$ is specific to the POMDP case. On the other hand, in the analysis of HMMs usually there is no dependency on the parameters $G$ and $\theta$ because the samples are assumed to be drawn from the stationary distribution of the chain. Removing this assumption required developing novel results for the tensor decomposition process itself using extensions of matrix concentration inequalities for the case of Markov chain (not yet in the stationary distribution). The overall analysis is reported in App.~\ref{app:ConBound} and~\ref{s:whitening}. It worth to note that, \cite{kontorovich2013learning}, without stationary assumption, proposes new method to learn the transition matrix of HMM model given factor matrix $O$, and it provides theoretical bound over estimation errors.

%%%%%%%%%%%%%%%%%%%%%%%%%%%%%%%%%%%%%%%%%%%%%%%%%%%%%%%%%%%%%%%%
%%%%%%%%%%%%%%%%%%%%%%%%%%%%%%%%%%%%%%%%%%%%%%%%%%%%%%%%%%%%%%%%
%% LEARNING
%%%%%%%%%%%%%%%%%%%%%%%%%%%%%%%%%%%%%%%%%%%%%%%%%%%%%%%%%%%%%%%%
%%%%%%%%%%%%%%%%%%%%%%%%%%%%%%%%%%%%%%%%%%%%%%%%%%%%%%%%%%%%%%%%

% !TEX root = master.tex

%%%%%%%%%%%%%%%%%%%%%%%%%%%%%%%%%%%%%%%%%%%%%%%%%%%%%%%%%%%%%%%%
%%%%%%%%%%%%%%%%%%%%%%%%%%%%%%%%%%%%%%%%%%%%%%%%%%%%%%%%%%%%%%%%
%% LEARNING
%%%%%%%%%%%%%%%%%%%%%%%%%%%%%%%%%%%%%%%%%%%%%%%%%%%%%%%%%%%%%%%%
%%%%%%%%%%%%%%%%%%%%%%%%%%%%%%%%%%%%%%%%%%%%%%%%%%%%%%%%%%%%%%%%

\section{Spectral \ucrl}\label{s:learning}

%In the previous section we introduced a novel spectral method for the estimation of the POMDP parameters.
The most interesting aspect of the estimation process illustrated in the previous section is that it can be applied when samples are collected using any policy $\pi$ in the set $\calP$. As a result, it can be integrated into any exploration-exploitation strategy where the policy changes over time in the attempt of minimizing the regret.

\begin{algorithm}[t]
\begin{small}
\setstretch{1.1}
\begin{algorithmic}
\STATE \textbf{Input:} Confidence $\delta'$
\STATE \textbf{Variables:}
\STATE $\quad$ Number of samples $N^{(k)}(l)$
\STATE $\quad$ Estimated observation, reward, and transition models $\wh{f}^{(k)}_O$, $\wh{f}^{(k)}_R$, $\wh{f}^{(k)}_T$
\STATE \textbf{Initialize:} $t=1$, initial state $x_1$, $\delta = \delta'/N^6$, $k=1$
\WHILE{$t < N$}
\STATE Compute the estimated POMDP $\wh{M}^{(k)}$ with the Alg.~\ref{alg:spectral.pomdp} using $N^{(k)}(l)$ samples per action
\STATE Compute the set of admissible POMDPs $\M^{(k)}$ using bounds in Thm.~\ref{thm:estimates}
\STATE Compute the optimistic policy
$\wt{\pi}^{(k)} = \arg\max\limits_{\pi\in\calP}\max\limits_{M\in\M^{(k)}} \eta(\pi; M)$
\STATE Set $v^{(k)}(l) = 0$ for all actions $l\in[A]$
\WHILE{$\forall l\in[A], v^{(k)}(l) < 2N^{(k)}(l)$}
\STATE Execute $a_t \sim f_{\wt{\pi}^{(k)}}(\cdot|\vec{y}_t)$
\STATE Obtain reward $\vec{r}_t$, observe next observation $\vec{y}_{t+1}$, and set $t=t+1$
\ENDWHILE
\STATE Store $N^{(k+1)}(l) = \max_{k'\leq k} v^{(k')}(l)$ samples for each action $l\in[A]$
\STATE Set $k=k+1$
\ENDWHILE
\end{algorithmic}
\caption{The \smucrl algorithm.}
\label{alg:sm.ucrl}
\end{small}
\end{algorithm}

\paragraph{The algorithm.} The \smucrl algorithm illustrated in Alg.~\ref{alg:sm.ucrl} is the result of the integration of the spectral method into a structure similar to \ucrl~\citep{jaksch2010near-optimal} designed to optimize the exploration-exploitation trade-off. The learning process is split into episodes of increasing length. At the beginning of each episode $k>1$ (the first episode is used to initialize the variables), an estimated POMDP $\wh{M}^{(k)} = (X, A, Y, R, \wh{f}_T^{(k)}, \wh{f}_R^{(k)}, \wh{f}_O^{(k)})$ is computed using the spectral method of Alg.~\ref{alg:spectral.pomdp}.
%Through the algorithm, there is no need to know any information about model parameter to define the length of initial episodes.
Unlike in \ucrl, \smucrl cannot use all the samples from past episodes. In fact, the distribution of the views $\vec{v}_1, \vec{v}_2, \vec{v}_3$ depends on the policy used to generate the samples. As a result, whenever the policy changes, the spectral method should be re-run using only the samples collected by that specific policy. Nonetheless we can exploit the fact that the spectral method is applied to each action separately. In \smucrl at episode $k$ for each action $l$ we use the samples coming from the past episode which returned the largest number of samples for that action. Let $v^{(k)}(l)$ be the number of samples obtained during episode $k$ for action $l$, we denote by $N^{(k)}(l)  = \max_{k'< k} v^{(k')}(l)$ the largest number of samples available from past episodes for each action separately and we feed them to the spectral method to compute the estimated POMDP $\wh{M}^{(k)}$ at the beginning of each episode $k$.

Given the estimated POMDP $\wh{M}^{(k)}$ and the result of Thm.~\ref{thm:estimates}, we construct the set $\M^{(k)}$ of \textit{admissible} POMDPs $\wt{M} = \langle \X, \A, \Y, \R, \wt{f}_T, \wt{f}_R, \wt{f}_O\rangle$ whose transition, reward, and observation models belong to the confidence intervals (e.g., $\|\wh{f}^{(k)}_O(\cdot|i) \!-\! \wt{f}_O(\cdot|i)\|_1 \leq \B_O$ for any state $i$).
By construction, this guarantees that the true POMDP $M$ is included in $\M^{(k)}$ with high probability.
Following the \textit{optimism in face of uncertainty} principle used in \ucrl, we compute the optimal memoryless policy corresponding to the most optimistic POMDP within $\M^{(k)}$. More formally, we compute\footnote{The computation of the optimal policy (within $\calP$) in the optimistic model may not be trivial. Nonetheless, we first notice that given an horizon $N$, the policy needs to be recomputed at most $O(\log N)$ times (i.e., number of episodes). Furthermore, if an optimization oracle to $\eta(\pi; M)$ for a given POMDP $M$ is available, then it is sufficient to randomly sample multiple POMDPs from $\M^{(k)}$ (which is a computationally cheap operation), find their corresponding best policy, and return the best among them. If \textit{enough} POMDPs are sampled, the additional regret caused by this approximately optimistic procedure can be bounded as $\wt{O}(\sqrt{N})$.}
\begin{align}\label{eq:optimistic.policy}
\wt{\pi}^{(k)} = \arg\max_{\pi\in\calP}\max_{M\in\M^{(k)}} \eta(\pi; M).
\end{align}
Intuitively speaking, the optimistic policy implicitly balances exploration and exploitation. Large confidence intervals suggest that $\wh{M}^{(k)}$ is poorly estimated and further exploration is needed. Instead of performing a purely explorative policy, \smucrl still exploits the current estimates to construct the set of admissible POMDPs and selects the policy that maximizes the performance $\eta(\pi; M)$ over all POMDPs in $\M^{(k)}$. The choice of using the optimistic POMDP guarantees the $\wt{\pi}^{(k)}$ explores more often actions corresponding to large confidence intervals, thus contributing the improve the estimates over time. After computing the optimistic policy, $\wt{\pi}^{(k)}$ is executed until the number of samples for one action is doubled, i.e., $v^{(k)}(l) \geq 2N^{(k)}(l)$. This stopping criterion avoids switching policies too often and it guarantees that when an episode is terminated, enough samples are collected to compute a new (better) policy. This process is then repeated over episodes and we expect the optimistic policy to get progressively closer to the best policy $\pi^+\in\calP$ as the estimates of the POMDP get more and more accurate.

%%%%%%%%%%%%%%%%%%%%%%%%%%%%%%%%%%%%%%%%%%%%%%%%%%%%%%%%%%%%%%%%
%%%%%%%%%%%%%%%%%%%%%%%%%%%%%%%%%%%%%%%%%%%%%%%%%%%%%%%%%%%%%%%%
%% REGRET ANALYSIS
%%%%%%%%%%%%%%%%%%%%%%%%%%%%%%%%%%%%%%%%%%%%%%%%%%%%%%%%%%%%%%%%
%%%%%%%%%%%%%%%%%%%%%%%%%%%%%%%%%%%%%%%%%%%%%%%%%%%%%%%%%%%%%%%%

\paragraph{Regret analysis.}
We now study the regret \smucrl w.r.t.\ the best policy in $\calP$. While in general $\pi^+$ may not be optimal, $\pi_{\min}$ is usually set to a small value and oftentimes the optimal memoryless policy itself is stochastic and it may actually be contained in $\calP$. Given an horizon of $\TN$ steps, the regret is defined as
\begin{align}\label{eq:regret}
\text{Reg}_\TN = \TN\eta^+ - \sum_{t=1}^\TN r_t,
\end{align}
where $r_t$ is the random reward obtained at time $t$ according to the reward model $f_R$ over the states traversed by the policies performed over episodes on the actual POMDP.
To restate, similar to the MDP case, the complexity of learning in a POMDP $M$ is partially determined by its diameter, defined as
\begin{align}\label{eq:diameter}
%D=\max_{x,x'\in\X,a,a'\in\A}\min_{\pi\in\calP}\mathbb{E}_{\pi}\big[\tau((x',a')\rightarrow(x,a))\big],
D:=\max_{x,x'\in\X,a,a'\in\A}\min_{\pi\in\calP}\mathbb{E}\big[\tau(x',a'|x,a; \pi)\big],
\end{align}
which corresponds to the expected passing time from a state $x$ to a state $x'$ starting with action $a$ and terminating with action $a'$ and following the most effective memoryless policy $\pi\in\calP$. The main difference w.r.t. to the diameter of the underlying MDP (see e.g.,~\cite{jaksch2010near-optimal}) is that it considers the distance between state-action pairs using memoryless policies instead of state-based policies. %Whenever the exepcted passing time between each pair of state-action is small, then the diameter is small and viceversa.

Before stating our main result, we introduce the worst-case version of the parameters characterizing Thm.~\ref{thm:estimates}. Let $\wb{\sigma}_{1,2,3} := \min\limits_{l\in[A]} \min\limits_{\pi\in\calP} \omega^{(l)}_{\min}\min\limits_{\nu\in\{1,2,3\}}\sigma^2_{\min}(V^{(l)}_\nu)$ be the worst smallest non-zero singular value of the views for action $l$ when acting according to policy $\pi$ and let $\wb{\sigma}_{1,3} := \min\limits_{l\in[A]} \min\limits_{\pi\in\calP} \sigma_{\min}(K_{1,3}^{(l)}(\pi))$ be the worst smallest non-zero singular value of the covariance matrix $K_{1,3}^{(l)}(\pi)$ between the first and third view for action $l$ when acting according to policy $\pi$. Similarly, we define $\wb{\sigma}_{1,2}$. We also introduce $\wb{\omega}_{\min} := \min\limits_{l\in[A]}\min\limits_{x\in[X]}\min\limits_{\pi\in\calP} \omega_{\pi}^{(l)}(x)$ and 
\begin{align}\label{eq:nl.condition.max}
\wb{N} := \max_{l\in[A]}\max_{\pi\in\calP}\max
\bigg\{ \frac{4}{(\wb{\sigma}_{3,1}^2)}, \frac{16C_O^2YR}{{\lambda^{(l)}}^2\deltaO^2},\left(\frac{G(\pi)\frac{2\sqrt{2}+1}{1-\theta(\pi)}}{\wb{\omega}_{\min}\wb{\sigma}_{1,2,3}}\right)^2\wb{\Theta}^{(l)}\bigg\} \log\bigg(2\frac{(Y^2+AYR)}{\delta}\bigg),
\end{align}
which is a sufficient number of samples for the statement of Thm.~\ref{thm:estimates} to hold for any action and any policy. Here $\wb{\Theta}^{(l)}$ is also model related parameter which is defined in Eq.~\ref{eq:thetabar}. Then we can prove the following result.

%\todoa{This should be fixed according to the last version of the bound for the spectral method.}

\begin{theorem}[Regret Bound]\label{thm:regret}
Consider a POMDP $M$ with $X$ states, $A$ actions, $Y$ observations, $R$ rewards, characterized by a diameter $D$ and with an observation matrix $O\in\Re^{Y\times X}$ with smallest non-zero singular value $\sigma_X(O)$. We consider the policy space $\calP$, such that the worst smallest non-zero value is $\wb{\sigma}_{1,2,3}$ (resp. $\wb{\sigma}_{1,3}$) and the worst smallest probability to reach a state is $\wb{\omega}_{\min}$.
%We also define the mixing time $\wb{t}_{\text{mix}}$ of $M$ as the largest mixing time over all the Markov chains induced by the policies in $\calP$ on $M$.
If \smucrl is run over $N$ steps and the confidence intervals of Thm.~\ref{thm:estimates} are used with $\delta = \delta'/N^{6}$ in constructing the plausible POMDPs $\wt{\mathcal M}$, then under Asm.~\ref{asm:ergodicity}, \ref{asm:observation}, and~\ref{asm:transition} it suffers from a total regret
%
%\begin{align}\label{eq:regret.bound1}
%\text{Reg}_N\leq C_1\frac{r_{\max}}{\wb{\lambda}}DX^{3/2}\sqrt{AYRN\log (N/\delta')} + C_2 r_{\max}\wb{N}\wb{t}_{\text{mix}}\log(N/\delta'),
%\end{align}
\begin{align}\label{eq:regret.bound1}
\text{Reg}_N\leq C_1\frac{r_{\max}}{\wb{\lambda}}DX^{3/2}\sqrt{AYRN\log (N/\delta')}
\end{align}
with probability $1-\delta'$, where $C_1$ is numerical constants, and $\wb{\lambda}$ is the worst-case equivalent of Eq.~\ref{eq:lambda} defined as
\begin{align}\label{eq:worst.lambda}
\wb{\lambda} = \sigma_{\min}(O)\pi_{\min}^2\wb{\sigma}_{1,3} \wb{\sigma}_{1,2,3}^{3/2}
\end{align}

\end{theorem}

\noindent\paragraph{Remark 1 (comparison with MDPs).}
If \ucrl could be run directly on the underlying MDP (i.e., as if the states where directly observable), then it would obtain a regret~\citep{jaksch2010near-optimal}
\begin{align*}
\text{Reg}_N\leq C_{\text{MDP}}D_{\text{MDP}}X\sqrt{AN\log{N}},
\end{align*}
where
\begin{align*}
D_{\text{MDP}}:=\max_{x,x'\in\X}\min_{\pi}\mathbb{E}[\tau(x'|x; \pi)],
\end{align*}
with high probability. We first notice that the regret is of order $\wt{O}(\sqrt{N})$ in both MDP and POMDP bounds. This means that despite the complexity of POMDPs, \smucrl has the same dependency on the number of steps as in MDPs and it has a vanishing per-step regret. Furthermore, this dependency is known to be minimax optimal. The diameter $D$ in general is larger than its MDP counterpart $D_{\text{MDP}}$, since it takes into account the fact that a memoryless policy, that can only work on observations, cannot be as efficient as a state-based policy in moving from one state to another. Although no lower bound is available for learning in POMDPs, we believe that this dependency is unavoidable since it is strictly related to the partial observable nature of POMDPs.

\noindent\paragraph{Remark 2 (dependency on POMDP parameters).}
The dependency on the number of actions is the same in both MDPs and POMDPs. On the other hand, moving to POMDPs naturally brings the dimensionality of the observation and reward models ($Y$,$X$, and $R$ respectively) into the bound. The dependency on $Y$ and $R$ is directly inherited from the bounds in Thm.~\ref{thm:estimates}. The term $X^{3/2}$ is indeed the results of two terms; $X$ and $X^{1/2}$. The first term is the same as in MDPs, while the second comes from the fact that the transition tensor is derived from Eq.~ \ref{eq:transition.recovery.inv}. Finally, the term $\wb{\lambda}$ in Eq.~\ref{eq:regret.bound1} summarizes a series of terms which depend on both the policy space $\calP$ and the POMDP structure. These terms are directly inherited from the spectral decomposition method used at the core of \smucrl and, as discussed in Sect.~\ref{s:learning.pomdp}, they are due to the partial observability of the states and the fact that all (unobservable) states need to be visited often enough to be able to compute accurate estimate of the observation, reward, and transition models. 
%Finally, \smucrl suffers from an additional regret (the second term in Eq.~\ref{eq:regret.bound1}) which accounts from the difficulty of applying the spectral method over different policies. In fact, the bounds in Thm.~\ref{thm:estimates} need enough samples before being accurate, hence the additional dependency on $\wt{K}$ fist episode in Eq.~\ref{eq:constantterm} in the regret bound. %) and need samples to be drawn from a stationary distribution as required in Thm.~\ref{thm:estimates}, which requires the burn-in phase and a constant additional regret depending on $\wb{t}_{\text{mix}}$.

\noindent\paragraph{Remark 3 (computability of the confidence intervals).} While it is a common assumption that the dimensionality $X$ of the hidden state space is known as well as the number of actions, observations, and rewards, it is not often the case that the terms $\lambda^{(l)}$ appearing in Thm.~\ref{thm:estimates} are actually available. While this does not pose any problem for a \textit{descriptive} bound as in Thm.~\ref{thm:estimates}, in \smucrl we actually need to compute the bounds $\B_O^{(l)}$, $\B_R^{(l)}$, and $\B_T^{(l)}$ to explicitly construct confidence intervals. This situation is relatively common in many exploration--exploitation algorithms that require computing confidence intervals containing the range of the random variables or the parameters of their distributions in case of sub-Gaussian variables. In practice these values are often replaced by parameters that are tuned by hand and set to much smaller values than their theoretical ones. As a result, we can run \smucrl with the terms $\lambda^{(l)}$ replaced by a fixed parameter. Notice that any inaccurate choice in setting $\lambda^{(l)}$ would mostly translate into bigger multiplicative constants in the final regret bound or in similar bounds but with smaller probability. \\
In general, computing confidence bound is a hard problem, even for simpler cases such as Markov chains \cite{hsu2015mixing}. Therefore finding upper confidence bounds for POMDP is challenging if we do not know its mixing properties. As it mentioned, another parameter is needed to compute upper confidence bound is $\lambda^{(l)}$ \ref{eq:lambda}.  As it is described in, in practice, one can replace the coefficient $\lambda^{(l)}$ with some constant which causes bigger multiplicative constant in final regret bound. Alternatively, one can estimate $\lambda^{(l)}$ from data. In this case, we  add a lower order term to the regret which decays as $\frac{1}{N}$. 

\noindent\paragraph{Remark 4 (relaxation on assumptions).}
Both Thm.~\ref{thm:estimates} and~\ref{thm:regret} rely on the observation matrix $O\in\Re^{Y\times X}$ being full column rank (Asm.~\ref{asm:observation}). As discussed in Sect.~\ref{s:preliminaries} may not be verified in some POMDPs where the number of states is larger than the number of observations ($X>Y$). Nonetheless, it is possible to correctly estimate the POMDP parameters when $O$ is not full column-rank by exploiting the additional information coming from the reward and action taken at step $t+1$. In particular, we can use the triple $(a_{t+1}, \vec{y}_{t+1}, r_{t+1})$ and redefine the third view $V^{(l)}_3\in\Re^{d\times X}$ as
\begin{align*}%\label{eq:multi.views}
[V_3^{(l)}]_{s,i}&=\Prob(\vec{v}_3^{(l)}=\vec{e}_s |x_2=i,a_2=l) = [V_3^{(l)}]_{(n,m,k),i}\\
&= \Prob(\vecy_3=\vec{e}_n,\vec{r}_3=\vec{e}_m,a_3=k|x_2=i,a_2=l),
\end{align*}
and replace Asm.~\ref{asm:observation} with the assumption that the view matrix $V^{(l)}_3$ is full column-rank, which basically requires having rewards that jointly with the observations are informative enough to reconstruct the hidden state. While this change does not affect the way the observation and the reward models are recovered in Lemma~\ref{lem:pomdp.parameters}, (they only depend on the second view $V^{(l)}_2$), for the reconstruction of the transition tensor, we need to write the third view $V_3^{(l)}$ as
\begin{align*}
&[V_3^{(l)}]_{s,i} = [V_3^{(l)}]_{(n,m,k),i}\\
&= \sum_{j=1}^X \Prob\big(\vecy_3=\vec{e}_n,\vec{r}_3=\vec{e}_m,a_3=k|x_2=i,a_2=l, x_3=j\big) \Prob\big(x_3=j|x_2=i,a_2=l\big)\\
%&\quad\quad\quad\cdot\Prob(x_3=j|x_2=i,a_2=l) \\
&= \sum_{j=1}^X \Prob\big(\vec{r}_3=\vec{e}_m|x_3=j,a_3=k) \Prob(a_3=k|\vec{y}_3=\vec{e}_n\big) \Prob\big(\vecy_3=\vec{e}_n|x_3=j\big)\Prob\big(x_3=j|x_2=i,a_2=l\big) \\
%&\quad\quad\quad\cdot\Prob(\vecy_3=\vec{e}_n|x_3=j)\Prob(x_3=j|x_2=i,a_2=l) \\
&= f_\pi(k| \vec{e}_n) \sum_{j=1}^X f_R(\vec{e}_m | j, k) f_O(\vec{e}_n|j) f_T(j| i,l),
\end{align*}
\noindent where we factorized the three components in the definition of $V_3^{(l)}$ and used the graphical model of the POMDP to consider their dependencies. We introduce an auxiliary matrix $W\in\Re^{d\times X}$ such that
\begin{align*}
[W]_{s,j} = [W]_{(n,m,k), j} = f_\pi(k| \vec{e}_n) f_R(\vec{e}_m | j, k) f_O(\vec{e}_n|j),
\end{align*}
which contain all known values, and for any state $i$ and action $l$ we can restate the definition of the third view as
\begin{align}\label{eq:transition.recovery2}
W [T]_{i,:,l} = [V_3^{(l)}]_{:,i},
\end{align}
which allows computing the transition model as $[T]_{i,:,l} = W^\dagger [V_3^{(l)}]_{:,i}$, where $W^\dagger$ is the pseudo-inverse of $W$. While this change in the definition of the third view allows a significant relaxation of the original assumption, it comes at the cost of potentially worsening the bound on $\wh{f}_T$ in Thm.~\ref{thm:estimates}. In fact, it can be shown that
\begin{align}\label{eq:transition.bound2}
\| \wt{f}_T(\cdot|i,l) - f_T(\cdot|i,l) \|_F \!\leq\! \B_T' \!:= \!\!\!\max_{l'=1,\ldots,A}\frac{C_TAYR}{\lambda^{(l')}} \sqrt{\frac{XA\log(1/\delta)}{N(l')}}.
\end{align}
Beside the dependency on multiplication of $Y$, $R$, and $R$, which is due to the fact that now $V_3^{(l)}$ is a larger matrix, the bound for the transitions triggered by an action $l$ scales with the number of samples from the least visited action. This is due to the fact that now the matrix $W$ involves not only the action for which we are computing the transition model but all the other actions as well. As a result, if any of these actions is poorly visited, $W$ cannot be accurately estimated is some of its parts and this may negatively affect the quality of estimation of the transition model itself. This directly propagates to the regret analysis, since now we require all the actions to be repeatedly visited enough. The immediate effect is the introduction of a different notion of diameter. Let $\tau_{M,\pi}^{(l)}$ the mean passage time between two steps where action $l$ is chosen according to policy $\pi\in\calP$, we define
\begin{align}\label{eq:diameter-ratio}
\Dr=\max_{\pi\in\mathcal{P}}\frac{\max_{l\in\mathcal{A}}\tau_{M,\pi}^{(l)}}{\min_{l\in\mathcal{A}}\tau_{M,\pi}^{(l)}}
\end{align}
as the diameter ratio, which defines the ratio between maximum mean passing time between choosing an action and choosing it again, over its minimum. As it mentioned above, in order to have an accurate estimate of $f_T$ all actions need to be repeatedly explored. The $\Dr$ is small when each action is executed frequently enough and it is large when there is at least one action that is executed not as many as others. Finally, we obtain%\todoa{We should check whether the proof of this result is actually up to date.}
\begin{align*}
\text{Reg}_N\leq \wt{O}\Big(\frac{r_{\max}}{\wb{\lambda}}\sqrt{YR\Dr N\log N}X^{{3}/{2}}A(D+1)\Big).
\end{align*}
While at first sight this bound is clearly worse than in the case of stronger assumptions, notice that $\wb{\lambda}$ now contains the smallest singular values of the newly defined views. In particular, as $V_3^{(l)}$ is larger, also the covariance matrices $K_{\nu,\nu'}$ are bigger and have larger singular values, which could significantly alleviate the inverse dependency on $\wb{\sigma}_{1,2}$ and $\wb{\sigma}_{2,3}$. As a result, relaxing Asm.~\ref{asm:observation} may not necessarily worsen the final bound since the bigger diameter may be compensated by better dependencies on other terms. We leave a more complete comparison of the two configurations (with or without Asm.~\ref{asm:observation}) for future work.

%%%%%%%%%%%%%%%%%%%%%%%%%%%%%%%%%%%%%%%%%%%%%%%%%%%%%%%%%%%%%%%%
%%%%%%%%%%%%%%%%%%%%%%%%%%%%%%%%%%%%%%%%%%%%%%%%%%%%%%%%%%%%%%%%
%% DISCUSSION
%%%%%%%%%%%%%%%%%%%%%%%%%%%%%%%%%%%%%%%%%%%%%%%%%%%%%%%%%%%%%%%%
%%%%%%%%%%%%%%%%%%%%%%%%%%%%%%%%%%%%%%%%%%%%%%%%%%%%%%%%%%%%%%%%

% !TEX root = master.tex

%%%%%%%%%%%%%%%%%%%%%%%%%%%%%%%%%%%%%%%%%%%%%%%%%%%%%%%%%%%%%%%%
%%%%%%%%%%%%%%%%%%%%%%%%%%%%%%%%%%%%%%%%%%%%%%%%%%%%%%%%%%%%%%%%
%% DISCUSSION
%%%%%%%%%%%%%%%%%%%%%%%%%%%%%%%%%%%%%%%%%%%%%%%%%%%%%%%%%%%%%%%%
%%%%%%%%%%%%%%%%%%%%%%%%%%%%%%%%%%%%%%%%%%%%%%%%%%%%%%%%%%%%%%%%

\section{Conclusion}\label{s:conclusions}

We introduced a novel RL algorithm for POMDPs which relies on a spectral method to consistently identify the parameters of the POMDP and an optimistic approach for the solution of the exploration--exploitation problem. For the resulting algorithm we derive confidence intervals on the parameters and a minimax optimal bound for the regret. 

This work opens several interesting directions for future development. \textbf{1)} \smucrl cannot accumulate samples over episodes since Thm.~\ref{thm:estimates} requires samples to be drawn from a fixed policy. While this does not have a very negative impact on the regret bound, it is an open question how to apply the spectral method to all samples together and still preserve its theoretical guarantees. \textbf{2)} While memoryless policies may perform well in some domains, it is important to extend the current approach to bounded-memory policies. \textbf{3)} The POMDP is a special case of the predictive state representation (PSR) model~\cite{littman2001predictive}, which allows representing more sophisticated dynamical systems. Given the spectral method developed in this paper, a natural extension is to apply it to the more general PSR model and integrate it with an exploration--exploitation algorithm to achieve bounded regret.

\newpage
% !TEX root = master.tex

\section{Table of Notation}\label{s:notation}

%\vspace{-1cm}
\begin{table}[htbp]%\caption{Main notation used throughout the paper.}
\begin{center}% used the environment to augment the vertical space
% between the caption and the table
\begin{small}
\renewcommand{\arraystretch}{1.1}
\begin{tabular}{r c p{12.3cm} }
\toprule

\multicolumn{3}{c}{\textbf{POMDP Notation} (Sect.~\ref{s:preliminaries})}\\
\midrule
$\vec{e}$ &  & indicator vector \\
$M$ &  & POMDP model\\
$\X, X, x, (i,j)$ &  & state space, cardinality, element, indices \\
$\Y, Y, \vec{y}, n$ &  & observation space, cardinality, indicator element, index \\
$\A, A, a, (l, k)$ &  & action space, cardinality, element, indices \\
$\R, R, r, \vec{r}, m, r_{\max}$ &  & reward space, cardinality, element, indicator element, index, largest value \\
$f_T(x'|x,a), T$ &  & transition density from state $x$ to state $x'$ given action $a$ and transition tensor\\  
$f_O(\vec{y}|x), O$ &  & observation density of indicator $\vec{y}$ given state $x$ and observation matrix\\  
$f_R(\vec{r}|x,a), \Gamma$ &  & reward density of indicator $\vec{r}$ given pair of state-action and reward tensor\\  
$\pi, f_\pi(a|\vec{y}), \Pi$ &  & policy, policy density of action $a$ given observation indicator $\vec{y}$ and policy matrix\\  
$\pi_{\min}$, $\calP$ &  & smallest element of policy matrix and set of stochastic memoryless policies\\  
$f_{\pi,T}(x'|x)$ &  & Markov chain transition density for policy $\pi$ on a POMDP with transition density $f_T$\\  
$\omega_\pi$, $\omega_\pi^{(l)}$ &  & stationary distribution over states given policy $\pi$ and conditional on action $l$ \\  
$\eta(\pi,M)$ &  & expected average reward of policy $\pi$ in POMDP $M$\\    
$\eta^+$ &  & best expected average reward over policies in $\calP$\\    
& & \\
\multicolumn{3}{c}{\textbf{POMDP Estimation Notation} (Sect.~\ref{s:learning.pomdp})}\\
\midrule
$\nu\in\{1,2,3\}$ & & index of the views\\
$\vec{v}_{\nu,t}^{(l)}$, $V_{\nu}^{(l)}$ &  & $\nu$th view and view matrix at time $t$ given $a_t=l$\\
$K_{\nu,\nu'}^{(l)}$, $\sigma_{\nu,\nu'}^{(l)}$ &  & covariance matrix of views $\nu, \nu'$ and its smallest non-zero singular value given action $l$\\
%$\sigma_{\nu,\nu'}^{(l)}$ &  &  of $K_{\nu,\nu'}^{(l)}$\\
$M_2^{(l)}$, $M_3^{(l)}$ &  & second and third order moments of the views given middle action $l$\\  
$\wh{f}_O^{(l)}$, $\wh{f}_R^{(l)}$, $\wh{f}_T^{(l)}$ &  & estimates of observation, reward, and transition densities for action $l$\\  
$N$, $N(l)$ &  & total number of samples and number of samples from action $l$\\  
$C_O, C_R, C_T$ &  & numerical constants\\  
$\B_O$, $\B_R$, $\B_T$ &  & upper confidence bound over error of estimated $f_O$, $f_R$, $f_T$\\  
& & \\
\multicolumn{3}{c}{\textbf{\smucrl} (Sect.~\ref{s:learning})}\\
\midrule
$\text{Reg}_N$ & & cumulative regret \\
$D$ &  & POMDP diameter\\
$k$ & & index of the episode\\
$\wh{f}_T^{(k)}, \wh{f}_R^{(k)}, \wh{f}_O^{(k)}$, $\wh{M}^{(k)}$ & & estimated parameters of the POMDP at episode $k$ \\
$\M^{(k)}$ &  & set of plausible POMDPs at episode $k$\\
$v^{(k)}(l)$ &  & number of samples from action $l$ in episode $k$ \\
$N^{(k)}(l)$ &  & maximum number of samples from action $l$ over all episodes before $k$ \\
$\wt{\pi}^{(k)}$ &  & optimistic policy executed in episode $k$ \\
$\wb{N}$ & & min. number of samples to meet the condition in Thm.~\ref{thm:estimates} for any policy and any action \\
$\wb{\sigma}_{\nu,\nu'}$ & & worst smallest non-zero singular value of covariance $K_{\nu,\nu'}^{(l)}$ for any policy and action \\
$\wb{\omega}_{\min}$ & & smallest stationary probability over actions, states, and policies \\

\bottomrule
\end{tabular}
\end{small}
\end{center}
\label{tab:notation.small}
\end{table}

\newpage
\bibliography{ref}

\newpage
\appendix
{\setstretch{1.0}

%%%%%%%%%%%%%%%%%%%%%%%%%%%%%%%%%%%%%%%%%%%%%%%%%%%%%%%%%%%%%%%%
%%%%%%%%%%%%%%%%%%%%%%%%%%%%%%%%%%%%%%%%%%%%%%%%%%%%%%%%%%%%%%%%
%% PROOF OF SPECTRAL METHOD BOUNDS
%%%%%%%%%%%%%%%%%%%%%%%%%%%%%%%%%%%%%%%%%%%%%%%%%%%%%%%%%%%%%%%%
%%%%%%%%%%%%%%%%%%%%%%%%%%%%%%%%%%%%%%%%%%%%%%%%%%%%%%%%%%%%%%%%

% !TEX root = master.tex

%%%%%%%%%%%%%%%%%%%%%%%%%%%%%%%%%%%%%%%%%%%%%%%%%%%%%%%%%%%%%%%%
%%%%%%%%%%%%%%%%%%%%%%%%%%%%%%%%%%%%%%%%%%%%%%%%%%%%%%%%%%%%%%%%
%% PROOF OF SPECTRAL METHOD BOUNDS
%%%%%%%%%%%%%%%%%%%%%%%%%%%%%%%%%%%%%%%%%%%%%%%%%%%%%%%%%%%%%%%%
%%%%%%%%%%%%%%%%%%%%%%%%%%%%%%%%%%%%%%%%%%%%%%%%%%%%%%%%%%%%%%%%

\section{Organization of the Appendix}\label{app:organization}

\begin{figure}[h!]
\small
\begin{center}
\includegraphics[width=0.8\textwidth]{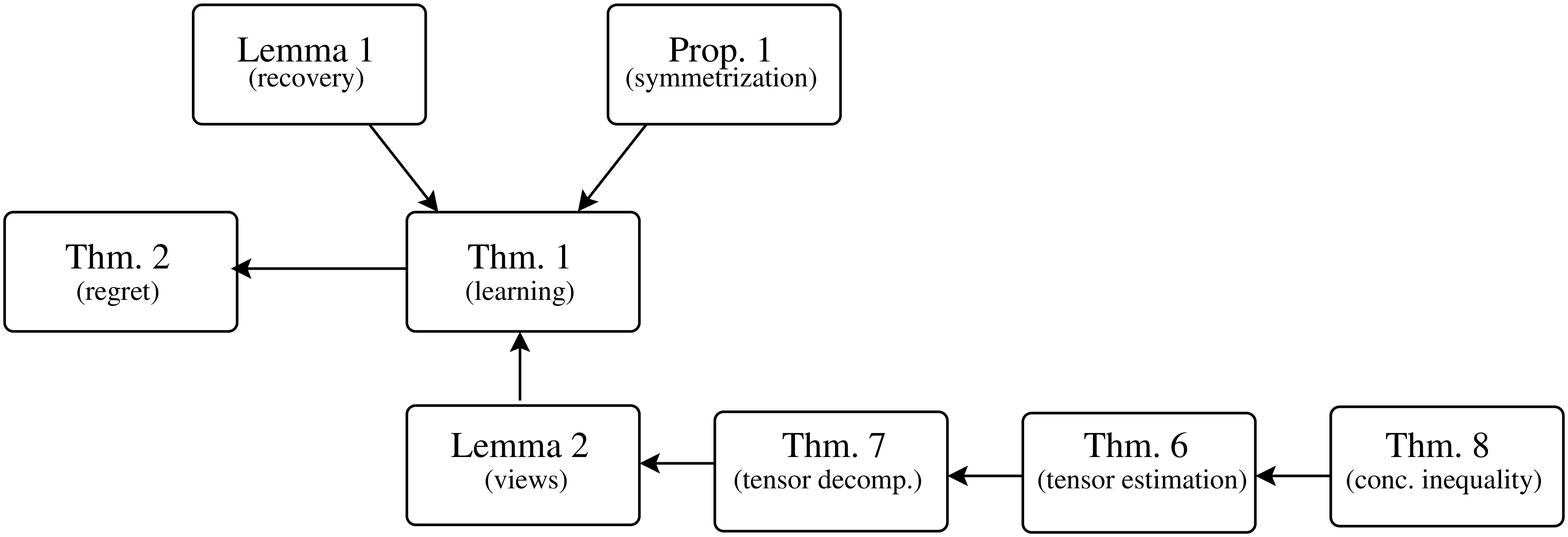}
\end{center}
\vspace{-0.2in}
\caption{Organization of the proofs.}
\label{fig:proofs}
\end{figure}

We first report the proofs of the main results of the paper in sections~\ref{s:proof.lem:pomdp.parameters}, \ref{app:proof1}, \ref{app:proof2}, \ref{s:proof.remark} and we postpone the technical tools used to derive them from Section~\ref{app:ConBound} on right after preliminary empirical results in Sect.~\ref{s:experiments}. In particular, the main lemmas and theorems of the paper are organized as in Fig.~\ref{fig:proofs}.

Furthermore, we summarize the additional notation used throughout the appendices in the following table.

%\begin{table}[h]\caption{Notation used in the appendix.}
\begin{center}
\begin{tabular}{r c p{12cm} }
\toprule
$\Delta_{n^{(l)}}$ &  & Concentration matrix\\  
$\eta_{i,j}^{(l)}(\cdot,\cdot,\cdot)$ &  & mixing coefficient\\  
$p(i,l)$ &  & translator of $i'th$ element in sequence of samples given middle action $l$ to the actual sequence number \\  
$\HH^{i|l}$ &  & $i'th$ quadruple consequence of states random variable given second action $l$\\  
$\hhh^{i|l}$ &  & $i'th$ quadruple consequence of states given second action $l$\\  
$\HH^{j|l}_i$ &  & sequence of all $\HH^{i'|l}$ for $i'\in\{i,\ldots,j\}$ \\  
$\hhh^{i|l}$ &  & sequence of all $\hhh^{i'|l}$ for $i'\in\{i,\ldots,j\}$\\  
$\VV^{i|l}$ &  & $i'th$ triple consequence of views random variable given second action $l$\\  
$\vv^{i|l}$ &  & $i'th$ triple consequence of observation given second action $l$\\  
$\VV^{j|l}_i$ &  & sequence of all $\VV^{i'|l}$ for $i'\in\{i,\ldots,j\}$ \\  
$\vv^{i|l}$ &  & sequence of all $\vv^{i'|l}$ for $i'\in\{i,\ldots,j\}$\\  
\bottomrule
\end{tabular}
\end{center}
%\label{tab:TableOfNotationForMyResearch3}
%\end{table}
For the tensor $A\in\mathbb{R}^{d_1\times d_2\ldots\times d_p}$, and matrices $\lbrace V_i\in\mathbb{R}^{d_i,n_i}:i\in\lbrace 1,\ldots,p\rbrace\rbrace$, the tensor multi-linear operator is defined as follows\\
For the ${i_1,i_2,\ldots,i_p}-th$ element
\begin{align*}
\left[A(V_1,V_2, \ldots, V_p)\right]_{i_1,i_2,\ldots,i_p}\displaystyle\sum_{j_1,j_2,\ldots,j_p\in\lbrace 1,2,\ldots,p\rbrace}A_{j_1,j_2,\ldots,j_p\in\lbrace 1,2,\ldots,p\rbrace}[V_1]_{j_1,i_1}[V_2]_{j_2,i_2}\cdots[V_p]_{j_p,i_p}
\end{align*}

\newpage

% !TEX root = master.tex

%%%%%%%%%%%%%%%%%%%%%%%%%%%%%%%%%%%%%%%%%%%%%%%%%%%%%%%%%%%%%%%%
%%%%%%%%%%%%%%%%%%%%%%%%%%%%%%%%%%%%%%%%%%%%%%%%%%%%%%%%%%%%%%%%
%% LEARNING POMDP
%%%%%%%%%%%%%%%%%%%%%%%%%%%%%%%%%%%%%%%%%%%%%%%%%%%%%%%%%%%%%%%%
%%%%%%%%%%%%%%%%%%%%%%%%%%%%%%%%%%%%%%%%%%%%%%%%%%%%%%%%%%%%%%%%

\section{Proof of Lemma~\ref{lem:pomdp.parameters}}\label{s:proof.lem:pomdp.parameters}

The proof proceeds by construction. First notice that the elements of the second view can be written as
\begin{align*}
[V_2^{(l)}]_{s,i} &= [V_2^{(l)}]_{(n',m'),i}\\
&= \Prob(\vecy_2=\vec{e}_{n'}|x_2=i,a_2=l)\Prob(\vec{r}_2=\vec{e}_{m'}|x_2=i,a_2=l)\\
&= \Prob(\vecy_2=\vec{e}_{n'}|x_2=i,a_2=l) f_R(\vec{e}_{m'}|i,l),
\end{align*}
where we used the independence between observations and rewards. As a result, summing up over all the observations $n'$, we can recover the reward model as
\begin{align}
f_R(\vec{e}_{m'}|i,l) = \sum_{n'=1}^Y [V_2^{(l)}]_{(n',m'),i},
\end{align}
for any combination of states $i\in [X]$ and actions $l\in[A]$.
In order to compute the observation model, we have to further elaborate the definition of $V_2^{(l)}$ as
\begin{align*}
&[V_2^{(l)}]_{s,i} = [V_2^{(l)}]_{(n',m'),i}\\
&= \frac{\Prob(a_2=l|x_2=i,\vecy_2=\vec{e}_{n'}) \Prob(\vecy_2=\vec{e}_{n'}|x_2=i)}{ \Prob(a_2=l|x_2=i) } \cdot \Prob(\vec{r}_2=\vec{e}_{m'}|x_2=i,a_2=l) \\
&= \frac{f_\pi(l|\vec{e}_{n'}) f_O(\vec{e}_{n'}|i) f_R(\vec{e}_{m'}|i,l)}{\Prob(a_2=l|x_2=i)}.
\end{align*}
Since the policy $f_\pi$ is known, if we divide the previous term by $f_\pi(l|\vec{e}_{n'})$ and sum over observations and rewards, we obtain the denominator of the previous expression as
\begin{align*}%\label{eq:rho.term}
\sum_{m'=1}^R \sum_{n'=1}^Y \frac{[V_2^{(l)}]_{(n',m'),i}}{f_\pi(l|\vec{e}_{n'})} = \frac{1}{\Prob(a_2=l|x_2=i)}.
\end{align*}
Let $\rho(i,l) = 1/\Prob(a_2=l|x_2=i)$ as computed above, then the observation model is
\begin{align}
f^{(l)}_O(\vec{e}_{n'}|i) = \sum_{m'=1}^R \frac{[V_2^{(l)}]_{(n',m'),i}}{f_\pi(l|\vec{e}_{n'})\rho(i,l)}.
\end{align}
Repeating the procedure above for each $n'$ gives the full observation model $f_O^{(l)}$. We are left with the transition tensor, for which we need to resort to the third view $V_3^{(l)}$, that can be written as
\begin{align}\label{eq:v3.transition}
[V_3^{(l)}]_{s,i} &= [V_3^{(l)}]_{n'',i}\nonumber\\
&= \sum_{j=1}^X \Prob(\vecy_3=\vec{e}_{n''}|x_2=i,a_2=l, x_3=j)\cdot\Prob(x_3=j|x_2=i,a_2=l) \nonumber\\
&= \sum_{j=1}^X\Prob(\vecy_3=\vec{e}_{n''}|x_3=j)\Prob(x_3=j|x_2=i,a_2=l) \nonumber\\
&= \sum_{j=1}^X f_O(\vec{e}_n|j) f_T(j| i,l),
\end{align}
\noindent where we used the graphical model of the POMDP to introduce the dependency on $x_3$. Since the policy $f_\pi$ is known and the observation model is obtained from the second view with Eq.~\ref{eq:obs.recovery}, it is possible to recover the transition model. 
We recall that the observation matrix $O\in\Re^{Y\times X}$ is such that $[O]_{n,j} = f_O(\vec{e}_n|j)$, then we can restate Eq.~\ref{eq:v3.transition} as
\begin{align}\label{eq:transition.recovery}
O [T]_{i,:,l} = [V_3^{(l)}]_{:,i}
\end{align}
where $[T]_{i,:,l}$ is the second mode of the transition tensor $T\in\Re^{X\times X\times A}$. Since all the terms in $O$ are known, we finally obtain $[T]_{i,:,l} = O^\dagger [V_3^{(l)}]_{:,i}$, where $O^\dagger$ is the pseudo-inverse of $O$. Repeating for all states and actions gives the full transition model $f_T$.

% !TEX root = master.tex

%%%%%%%%%%%%%%%%%%%%%%%%%%%%%%%%%%%%%%%%%%%%%%%%%%%%%%%%%%%%%%%%
%%%%%%%%%%%%%%%%%%%%%%%%%%%%%%%%%%%%%%%%%%%%%%%%%%%%%%%%%%%%%%%%
%% PROOF OF SPECTRAL METHOD BOUNDS
%%%%%%%%%%%%%%%%%%%%%%%%%%%%%%%%%%%%%%%%%%%%%%%%%%%%%%%%%%%%%%%%
%%%%%%%%%%%%%%%%%%%%%%%%%%%%%%%%%%%%%%%%%%%%%%%%%%%%%%%%%%%%%%%%

\section{Proof of Thm.~\ref{thm:estimates}}\label{app:proof1}

The proof builds upon previous results on HMM by~\cite{anandkumar2012method}, \cite{song2013nonparametric}, Thm.~\ref{thm:GeneralConcentrationBound}, Appendix~ \ref{app:ConBound}, . All the following statements hold under the assumption that the samples are drawn from the stationary distribution induced by the policy $\pi$ on the POMDP (i.e., $f_{T,\pi}$). In proving Thm.~\ref{thm:regret}, we will consider the additional error coming from the fact that samples are not necessarily drawn from $f_{T,\pi}$.

We denote by $\sigma_1(A)\geq \sigma_2(A) \geq \ldots $ the singular values of a matrix $A$ and we recall that the covariance matrices $K_{\nu,\nu'}^{(l)}$ have rank $X$ under Asm.~\ref{asm:observation} and we denote by $\sigma^{(l)}_{\nu,\nu'} = \sigma_X(K_{\nu,\nu'}^{(l)})$ its smallest non-zero singular value, where $\nu,\nu'\in\{1,2,3\}$. Adapting the result by~\cite{song2013nonparametric}, we have the following performance guarantee when the spectral method is applied to recover each column of the third view.

%%%%%%%%%%%%%%%% LEMMA

\begin{lemma}\label{lem:view.bound}
Let $\wh{\mu}_{3,i}^{(l)}\in\Re^{d_3}_3$ and $\wh{\omega}_\pi^{(l)}(i)$ be the estimated third view and the conditional distribution computed in state $i\in\X$ using the spectral method in Sect.~\ref{s:learning.pomdp} using $N(l)$ samples. Let $\omega_{\min}^{(l)} = \min_{x\in\X} \omega_\pi^{(l)}(x)$ and the number of samples $N(l)$ is such that

\begin{align}\label{eq:nl.condition1}
N(l)>\left(\frac{G(\pi)\frac{2\sqrt{2}+1}{1-\theta(\pi)}}{{\omega^{(l)}_{\min}\min\limits_{\nu\in\{1,2,3\}}\lbrace\sigma^2_{\min}(V^{(l)}_\nu)\rbrace}}\right)^2\log(2\frac{(d_1d_2+d_3)}{\delta})\Theta^{(l)}\\ \Theta^{(l)}=\max\left\{\frac{16X^{\frac{1}{3}}}{C_1^{\frac{2}{3}}(\omega^{(l)}_{\min})^{\frac{1}{3}}} ,4,\frac{2\sqrt{2}X}{C_1^2\omega^{(l)}_{\min}\min\limits_{\nu\in\{1,2,3\}}\lbrace\sigma^2_{\min}(V^{(l)}_\nu)\rbrace}\right\},
\end{align}
where $C_1$ is numerical constants and $d_1,d_2$ are dimensions of first and second views. Then under Thm.~\ref{thm:Whitening} for any $\delta\in(0,1)$ we have
%Then under Asm.~\ref{asm:transition} for any $\delta\in(0,1)$ we have
\footnote{More precisely, the statement should be phrased as ``there exists a suitable permutation on the label of the states such that''. This is due to the fact that the spectral method cannot recover the exact \textit{identity} of the states but if we properly relabel them, then the estimates are accurate. In here we do not make explicit the permutation in order to simplify the notation and readability of the results.}
\begin{align*}
&\big\|[\wh{V}_3^{(l)}]_{:,i} - [V_3^{(l)}]_{:,i}\big\|_2 \leq \epsilon_3
\end{align*}
with probability $1-\delta$ (w.r.t. the randomness in the transitions, observations, and policy), where\footnote{Notice that $\epsilon_3(l)$ does not depend on the specific state (column) $i$.}
\begin{equation}
\label{eq:epst1}
\epsilon_3(l):=G(\pi)\frac{4\sqrt{2}+2}{(\omega^{(l)}_{\min})^{\frac{1}{2}}(1-\theta(\pi))}\sqrt{\frac{\log(2\frac{(d_1+d_2)}{\delta})}{n}}+\frac{8\widetilde{\epsilon}_M}{\omega^{(l)}_{\min}}
\end{equation}
and
\begin{align*}
\widetilde{\epsilon}_M(l) \leq \frac{2\sqrt{2} G(\pi)\frac{2\sqrt{2}+1}{1-\theta(\pi)}\sqrt{\frac{\log(\frac{2(d_1d_2+d_3)}{\delta})}{N(l)}}}{((\omega^{(l)}_{\min})^{\frac{1}{2}}\min\limits_{\nu\in\{1,2,3\}}\lbrace\sigma_{\min}(V^{(l)}_\nu)\rbrace)^3}+\frac{\left(64 G(\pi)\frac{2\sqrt{2}+1}{1-\theta(\pi)}\right)}{{\min\limits_{\nu\in\{1,2,3\}}\lbrace\sigma^2_{\min}(V^{(l)}_\nu)\rbrace}(\omega^{(l)}_{\min})^{1.5}}\sqrt{\frac{\log(2\frac{(d_1d_2+d_3)}{\delta})}{N(l)}},
\end{align*}

\end{lemma}

Notice that although not explicit in the notation, $\epsilon_3(l)$ depends on the policy $\pi$ through the term $\omega_{\min}^{(l)}$. 

\begin{proof}
We now proceed with simplifying the expression of $\epsilon_3(l)$. Rewriting the condition on $N(l)$ in Eq.~\ref{eq:nl.condition1} we obtain
\begin{align*}
&\frac{\log(2\frac{(d_1d_2+d_3)}{\delta})}{N(l)}\leq \left(\frac{{\omega^{(l)}_{\min}\min\limits_{\nu\in\{1,2,3\}}\lbrace\sigma^2_{\min}(V^{(l)}_\nu)\rbrace}}{G(\pi)\frac{2\sqrt{2}+1}{1-\theta(\pi)}}\right)^2\\
\end{align*}
Substituting this bound on a factor $\log(2\frac{Y^2+YAR}{\delta})/N(l)$ in the second term of Eq.~\ref{eq:epst1}, we obtain
\begin{align*}
\widetilde{\epsilon}_M(l) \leq \frac{2\sqrt{2} G(\pi)\frac{2\sqrt{2}+1}{1-\theta(\pi)}\sqrt{\frac{\log(\frac{2(d_1d_2+d_3)}{\delta})}{N(l)}}}{((\omega^{(l)}_{\min})^{\frac{1}{2}}\min\limits_{\nu\in\{1,2,3\}}\lbrace\sigma_{\min}(V^{(l)}_\nu)\rbrace)^3}+\frac{\left(64 G(\pi)\frac{2\sqrt{2}+1}{1-\theta(\pi)}\right)}{{(\min\limits_{\nu\in\{1,2,3\}}\lbrace\sigma^2_{\min}(V^{(l)}_\nu)\rbrace)}(\omega^{(l)}_{\min})^{1.5}}\sqrt{\frac{\log(2\frac{(d_1d_2+d_3)}{\delta})}{N(l)}},
\end{align*}
which leads to the final statement after a few trivial bounds on the remaining terms.%(\omega^{(l)}_{\min})^{\frac{3}{2}}\min_i\sigma^3_X(V^{(l)}_i))
\end{proof}

While the previous bound does hold for both the first and second views when computed independently with a suitable symmetrization step, as discussed in Section~\ref{s:learning.pomdp}, this leads to inconsistent state indexes. As a result, we have to compute the other views by inverting Eq.~\ref{eq:expect.mod.views}. Before deriving the bound on the accuracy of the corresponding estimates, we introduce two propositions which will be useful later.

\begin{proposition}\label{prop:norm2}
%(Accuracy of empirical probabilities) 
Fix $\vec{\varsigma}=(\varsigma_1,\varsigma_2,\ldots,\varsigma_{(Y^2RA)})$ a point in $(Y^2)RA-1$ simplex.\footnote{Such that $\forall i=1,\ldots,d^2$, $\varsigma_i>0$ and $\sum_i \varsigma_i=1$.} Let $\vec{\xi}$ be a random one-hot vector such that $\mathbb{P}(\vec{\xi}=\vec{e}_i)=\varsigma_i$ for all $i\in \{1,\ldots,(Y)^2RA\}$ and let $\vec{\xi}_1,\vec{\xi}_2,\ldots,\vec{\xi}_N$ be $N$ i.i.d.\ copies of $\vec{\xi}$ and $\hat{\varsigma}=\frac{1}{N}\sum\limits_j^N\vec{\xi}_j$ be their empirical average, then
%\begin{align*}
%\|\hat{p}-\vec{p}\|_2\leq \sqrt{\frac{1}{N}}+\sqrt{\frac{\log{(\frac{1}{\delta})}}{N}}
%\end{align*}
%by probability at least $1-\delta$. Asymptotically, the first part is negligible compared to second part, then
\begin{align*}
\|\hat{\varsigma}-\vec{\varsigma}~\|_2\leq \sqrt{\frac{\log{(1/\delta)}}{N}},
\end{align*}
with probability $1-\delta$.
\end{proposition}

\begin{proof}
See Lemma F.1. in~\cite{anandkumar2012method}.
\end{proof}

\begin{proposition}\label{prop.inverse.K}
Let $\wh{K}_{3,1}^{(l)}$ be an empirical estimate of $K_{3,1}^{(l)}$ obtained using $N(l)$ samples. Then if 
\begin{align}\label{eq:nl.condition2}
N(l) \geq 4\frac{\log(1/\delta)}{(\sigma_{3,1}^{(l)})^2},
\end{align}
then
\begin{align*}
\|(K_{3,1}^{(l)})^{\dagger}-(\wh{K}_{3,1}^{(l)})^{\dagger}\|_2\leq\frac{\sqrt{\frac{\log{(1/\delta)}}{N(l)}}}{\sigma_{3,1}^{(l)}-\sqrt{\frac{\log{(1/\delta)}}{N(l)}}} \leq \frac{2}{\sigma_{3,1}^{(l)}}\sqrt{\frac{\log{(\frac{1}{\delta})}}{N(l)}},
\end{align*}
with probability $1-\delta$.
\end{proposition}

\begin{proof}
Since $K_{3,1}^{(l)} = \E\big[\vec{v}^{(l)}_3 \otimes \vec{v}^{(l)}_{1}\big]$ and the views are one-hot vectors, we have that each entry of the matrix is indeed a probability (i.e., a number between 0 and 1) and the sum of all the elements in the matrix sums up to 1. As a result, we can apply Proposition~\ref{prop:norm2} to $K_{3,1}^{(l)}$ and obtain
\begin{align}
%\label{eq:k1}&\|K_{2,1}^{(l)}-\hat{K}_{2,1}^{(l)}\|_2\leq\sqrt{\frac{\log{(\frac{1}{\delta})}}{N(l)}}\\
\label{eq:k2}&\|K_{3,1}^{(l)}-\wh{K}_{3,1}^{(l)}\|_2\leq\sqrt{\frac{\log{(1/\delta)}}{N(l)}},
\end{align}
with probability $1-\delta$. Then the statement follows by applying Lemma E.4. in~\cite{anandkumar2012method}.
\end{proof}

The previous proposition holds for $K_{2,1}^{(l)}$ as well with $\sigma_{2,1}^{(l)}$ replacing $\sigma_{3,1}^{(l)}$. We are now ready to state and prove the accuracy of the estimate of the second view (a similar bound holds for the first view).

\begin{lemma}\label{lemma:view.bound2}
Let $\wh{V}_2^{(l)}$ be the second view estimated inverting Eq.~\ref{eq:expect.mod.views} using estimated covariance matrices $K$ and $V_3^{(l)}$, then if $N(l)$ satisfies the conditions in Eq.~\ref{eq:nl.condition1} and Eq.~\ref{eq:nl.condition2}  with probability $1-3\delta$
\begin{align*}
\big\|[\wh{V}_2^{(l)}]_{:,i} - [V_2^{(l)}]_{:,i}\big\|_2 = \epsilon_2(l) := \frac{21}{\sigma_{3,1}^{(l)}}\epsilon_3(l).
\end{align*}
\end{lemma}

\begin{proof}
For any state $i\in\X$ and action $l\in A$, we obtain the second view by inverting Eq.~\ref{eq:expect.mod.views}, that is by computing
\begin{align*}
 [V_2^{(l)}]_{:,i}=K_{2,1}^{(l)}(K_{3,1}^{(l)})^{\dagger} [V_3^{(l)}]_{:,i}.
\end{align*}
To derive a confidence bound on the empirical version of $\mu_{2,i}^{(l)}$, we proceed by first upper bounding the error as
\begin{align*}
\big\|[\wh{V}_2^{(l)}]_{:,i} - [V_2^{(l)}]_{:,i}\big\|_2\leq& \|K_{2,1}^{(l)}-\wh{K}_{2,1}^{(l)}\|_2 \|(K_{3,1}^{(l)})^{\dagger}\|_2 \| [V_3^{(l)}]_{:,i}\|_2\\
&+\|K_{2,1}^{(l)}\|_2\|(K_{3,1}^{(l)})^{\dagger}-(\wh{K}_{3,1}^{(l)})^{\dagger}\|_2\| [V_3^{(l)}]_{:,i}\|_2\\
&+\|K_{2,1}^{(l)}\|_2\|(K_{3,1}^{(l)})^{\dagger}\|_2\big\|[\wh{V}_3^{(l)}]_{:,i} - [V_3^{(l)}]_{:,i}\big\|_2.
\end{align*}
The error $\|K_{2,1}^{(l)}-\wh{K}_{2,1}^{(l)}\|_2$ can be bounded by a direct application of Proposition~\ref{prop:norm2} (see also Eq.~\ref{eq:k2}). Then we can directly use Proposition~\ref{prop.inverse.K} to bound the second term and Lemma~\ref{lem:view.bound} for the third term, and obtain
\begin{align*}
\big\|[\wh{V}_2^{(l)}]_{:,i} - [V_2^{(l)}]_{:,i}\big\|_2 &\leq\frac{3}{\sigma_{3,1}^{(l)}}\sqrt{\frac{\log{(\frac{1}{\delta})}}{N(l)}}+\frac{18\epsilon_3(l)}{\sigma_{3,1}^{(l)}}\leq \frac{21\epsilon_3(l)}{\sigma_{3,1}^{(l)}},
\end{align*}
where we used $\|(K_{3,1}^{(l)})^{\dagger}\|_2 \leq 1/\sigma_{3,1}^{(l)}$, $\|K_{2,1}^{(l)}\|_2 \leq 1$ and $\|[V_3^{(l)}]_{:,i}\|_2 \leq 1$. Since each of the bounds we used hold with probability $1-\delta$, the final statement is valid with probability at least $1-3\delta$.
\end{proof}

%%%%%%%%%%%%% THEOREM
We are now ready to derive the bounds in Thm.~\ref{thm:estimates}.

%\noindent \textbf{Theorem~\ref{thm:estimates}}
%\textit{
%Let $\wh{O}$, $\wh{T}$, and $\wh{\Gamma}$ be the estimated POMDP parameters using a trajectory of $T$ steps. For any $\delta\in(0,1)$ and for any state $i=1,\ldots,X$ we have
%%
%\begin{align}\label{eq:obs.bound}
%\text{TV}&(\wh{f}_O(\cdot|i) - f_O(\cdot|i)) \leq \sqrt{Y} \| [\wh{O}]_{:,i} - [O]_{:,i} \|_2 \\
%&\leq \B_O := \min_{l=1,\ldots,A}\bigg\{ \frac{C_O^{(l)}}{\lambda_2^{(l)}} \sqrt{d\frac{\log(1/\delta)}{N(l)}}\bigg\} \nonumber
%\end{align}
%%
%and for any state $i=1,\ldots,X$ and any action $l=1,\ldots,A$
%\begin{align}\label{eq:rew.bound}
%\text{TV}&(\wh{f}_R(\cdot|i,l) - f_R(\cdot|i,l)) \leq \sqrt{R} \| [\wh{\Gamma}]_{i,l,:} - [\Gamma]_{i,l,:} \|_2 \\
%& \leq \B_R^{(l)} := \frac{C_R^{(l)}}{\lambda_2^{(l)}} \sqrt{d\frac{\log(1/\delta)}{N(l)}},\nonumber
%\end{align}
%%
%and
%%
%\begin{align}\label{eq:transition.bound}
%\text{TV}&(\wh{f}_T(\cdot|i,l) - f_T(\cdot|i,l)) \leq \sqrt{X} \| [\wh{T}]_{i,:,l} - [T]_{i,:,l} \|_2 \\
%&\leq \B_T := \sqrt{\frac{\log(1/\delta)}{N(l)}},\nonumber
%\end{align}
%%
%with probability $1-\delta$ (w.r.t. the randomness in the transitions, observations, and policy).
%}

\begin{proof}[Proof of Thm.~\ref{thm:estimates}]
We first recall that the estimates $\wh{f}_R$, $\wh{f}_O$, and $\wh{f}_T$ are obtained by working on the second and third views only, as illustrated in Sect.~\ref{s:learning.pomdp}.

\noindent\textbf{Step 1 (bound on $f_R$).} Using the empirical version of Eq.~\ref{eq:rew.recovery}, the reward model in state $i$ for action $l$ is computed as
\begin{align*}
\wh{f}_R(\vec{e}_{m'}|i,l) = \sum_{n'=1}^Y [\wh{V}_2^{(l)}]_{(n',m'),i}.
\end{align*}
Then the $\ell_1$-norm of the error can be bounded as
\begin{align*}
\|\wh{f}_R(.|i,l) - f_R(.|i,l)\|_1 &= \sum_{m'=1}^R |\wh{f}_R(\vec{e}_{m'}|i,l) - f_R(\vec{e}_{m'}|i,l)| \\
&\leq \sum_{m'=1}^R \bigg|\sum_{n'=1}^Y [\wh{V}_2^{(l)}]_{(n',m'),i} - \sum_{n'=1}^Y [V_2^{(l)}]_{(n',m'),i}\bigg|\\
&\leq \sum_{m'=1}^R\sum_{n'=1}^Y \bigg|[\wh{V}_2^{(l)}]_{(n',m'),i} -  [V_2^{(l)}]_{(n',m'),i}\bigg| \\
&\leq \sqrt{YR} \bigg(\sum_{m'=1}^R \sum_{n'=1}^Y \Big([\wh{V}_2^{(l)}]_{(n',m'),i} -  [V_2^{(l)}]_{(n',m'),i}\Big)^2\bigg)^{1/2} \\
&= \sqrt{YR} \big\|[\wh{V}_2^{(l)}]_{:,i} - [V_2^{(l)}]_{:,i}\big\|_2,
\end{align*}
where we use $\|v\|_1 \leq \sqrt{YR} \|v\|_2$ for any vector $v\in\Re^{Y\cdot R}$. Applying Lemma~\ref{lemma:view.bound2} we obtain
%%
%\begin{align*}
%\|\wh{f}_R(.|i,l) - f_R(.|i,l)\|_1 \leq 18 \sqrt{YR} \epsilon_2(l) \leq \B_R := \frac{C}{\sigma_{3,1}^{(l)}\min\{(\omega^{(l)}_{\min})^{\frac{3}{2}}\min_i\sigma^3_X(V^{(l)}_i)), (\sigma_{1,2}^{(l)})^{3}(\omega^{(l)})^{1/2}\}} \sqrt{\frac{YR\log(1/\delta)}{N(l)}}
%\end{align*}
%
\begin{align*}
\|\wh{f}_R(.|i,l) - f_R(.|i,l)\|_1 \leq \B_R := \frac{C_R}{\sigma_{3,1}^{(l)}(\omega^{(l)}_{\min})^{\frac{3}{2}}\min\limits_{\nu\in\{1,2,3\}}\lbrace\sigma^3_{\min}(V^{(l)}_\nu)\rbrace} \sqrt{\frac{ YR\log(2\frac{Y^2+YAR}{\delta})}{N(l)}},
\end{align*}
where $C_R$ is a numerical constant.

\noindent \textbf{Step 2 (bound on $\rho(i,l)$).}
We proceed by bounding the error of the estimate the term $\rho(i,l) = 1/\Prob(a_2=l|x_2=i)$ which is computed as
\begin{align*}
\wh{\rho}(i,l) = \sum_{m'=1}^R \sum_{n'=1}^Y \frac{[\wh{V}_2^{(l)}]_{(n',m'),i}}{f_\pi(l|\vec{e}_{n'})},
\end{align*}
and it is used to estimate the observation model. Similarly to the bound for $f_R$ we have
\begin{align}
|\rho(i,l)-\wh{\rho}(i,l)| &\leq \sum_{m'=1}^R \sum_{n'=1}^Y \frac{| [V_2^{(l)}]_{(n',m'),i} - [\wh{V}_2^{(l)}]_{(n',m'),i}|}{f_\pi(l|\vec{e}_{n'})} \leq \frac{1}{\pi_{\min}^{(l)}} \big\|[V_2^{(l)}]_{:,i}-[\wh{V}_2^{(l)}]_{:,i}\big\|_1 \nonumber \\
&\leq \frac{\sqrt{YR}}{\pi_{\min}^{(l)}} \big\|[V_2^{(l)}]_{:,i}-[\wh{V}_2^{(l)}]_{:,i}\big\|_2 \leq 21\frac{\sqrt{YR}}{\sigma_{3,1}^{(l)}\pi_{\min}^{(l)}} \epsilon_{3}(i) =: \epsilon_{\rho}(i,l),
\label{eq:rhobound}
\end{align}
where $\pi_{\min}^{(l)} = \min_{\vec{y}\in\Y} f_\pi(l|\vec{y})$ is the smallest non-zero probability of taking an action according to policy $\pi$.

\noindent\textbf{Step 3 (bound on $f_O$).} 
The observation model in state $i$ for action $l$ can be recovered by plugging the estimates into Eq.~\ref{eq:rew.recovery} and obtain
\begin{align*}
\wh{f}_O^{(l)}(\vec{e}_{n'}|i) = \sum_{m'=1}^R \frac{[\wh{V}_2^{(l)}]_{(n',m'),i}}{f_\pi(l|\vec{e}_{n'})\wh{\rho}(i,l)},
\end{align*}
where the dependency on $l$ is due do the fact that we use the view computed for action $l$. As a result, the $\ell_1$-norm of the estimation error is bounded as follows
\begin{align*}
\sum_{n'=1}^Y &|\wh{f}_O^{(l)}(\vec{e}_{n'}|i) - f_O(\vec{e}_{n'}|i)| \leq \sum_{n'=1}^Y\sum_{m'=1}^R \bigg| \frac{1}{f_\pi(l|\vec{e}_{n'})} \bigg(\frac{[\wh{V}_2^{(l)}]_{(n',m'),i}}{\wh{\rho}(i,l)} - \frac{[V_2^{(l)}]_{(n',m'),i}}{\rho(i,l)}\bigg)\bigg|\\
& \leq \frac{1}{\pi_{\min}^{(l)}}\sum_{n'=1}^Y\sum_{m'=1}^R \bigg|  \frac{\rho(i,l)\big([\wh{V}_2^{(l)}]_{(n',m'),i} - [V_2^{(l)}]_{(n',m'),i}\big) + [V_2^{(l)}]_{(n',m'),i} \big(\rho(i,l)-\wh{\rho}(i,l)\big) }{\wh{\rho}(i,l)\rho(i,l)}\bigg|\\
& \leq \frac{1}{\pi_{\min}^{(l)}}\bigg(\sum_{n'=1}^Y\sum_{m'=1}^R \frac{\big|[\wh{V}_2^{(l)}]_{(n',m'),i} - [V_2^{(l)}]_{(n',m'),i}\big|}{\wh{\rho}(i,l)} + \frac{\big|\rho(i,l)-\wh{\rho}(i,l)\big| }{\wh{\rho}(i,l)\rho(i,l)}\Big(\sum_{n'=1}^Y\sum_{m'=1}^R [V_2^{(l)}]_{(n',m'),i}\Big)  \bigg)\\
& \stackrel{(a)}{\leq} \frac{1}{\pi_{\min}^{(l)}}\bigg(\frac{\sqrt{YR}}{\wh{\rho}(i,l)}\big\|[\wh{V}_2^{(l)}]_{:,i} - [V_2^{(l)}]_{:,i}\big\|_2 + \frac{\big|\rho(i,l)-\wh{\rho}(i,l)\big| }{\wh{\rho}(i,l)\rho(i,l)}\Big(\sum_{m'=1}^R [V_2^{(l)}]_{(n',m'),i}\Big)  \bigg)\\
& \stackrel{(b)}{\leq} \frac{1}{\pi_{\min}^{(l)}}\bigg(\frac{\sqrt{YR}}{\wh{\rho}(i,l)}\epsilon_{2}(i) + \frac{\epsilon_{\rho}(i,l) }{\wh{\rho}(i,l)\rho(i,l)} \bigg)\\
& \stackrel{(c)}{\leq} \frac{1}{\pi_{\min}^{(l)}}\bigg(21\sqrt{YR}\frac{\epsilon_{3}(i)}{\sigma_{3,1}^{(l)}} + \epsilon_{\rho}(i,l) \bigg),
\end{align*}
where in $(a)$ we used the fact that we are only summing over $R$ elements (instead of the whole $YR$ dimensionality of the vector $[V_2^{(l)}]_{:,i}$), in $(b)$ we use Lemmas ~\ref{lem:view.bound}, ~\ref{lemma:view.bound2}, and in $(c)$ the fact that $1/\rho(i,l) = \Prob[a_2=l| x_2=i] \leq 1$ (similar for $1/\wh{\rho}(i,l)$). Recalling the definition of $\epsilon_\rho(i,l)$ and Lemma~\ref{lem:view.bound} and Lemma~\ref{lemma:view.bound2} we obtain

\begin{align*}
\|\wh{f}_O^{(l)}(\cdot|i) - f_O(\cdot |i)\|_1 &\leq \frac{62}{(\pi_{\min}^{(l)})^2}\sqrt{YR}\epsilon_{3}(l) \\
&\leq \B_O^{(l)} := \frac{C_O}{(\pi_{\min}^{(l)})^2\sigma_{1,3}^{(l)} (\omega^{(l)}_{\min})^{\frac{3}{2}}\min\limits_{\nu\in\{1,2,3\}}\lbrace\sigma^3_{\min}(V^{(l)}_\nu)\rbrace} \sqrt{\frac{YR\log(2\frac{Y^2+YAR}{\delta})}{N(l)}},
\end{align*}
where $C_O$ is a numerical constant.
As mentioned in Sect.~\ref{s:learning.pomdp}, since we obtain one estimate per action, in the end we define $\wh{f}_O$ as the estimate with the smallest confidence interval, that is
\begin{align*}
\wh{f}_O = \wh{f}_O^{(l^*)}, \quad l^*=\arg\min_{\{\wh{f}_O^{(l)}\}} \B_O^{(l)},
\end{align*}
whose corresponding error bound is
\begin{align*}
\|\wh{f}_O(\vec{e}_{n'} |i) - f_O(\vec{e}_{n}|i)\| \leq \B_O := \!\min_{l=1,\ldots,A}\! \frac{C_O}{(\pi_{\min}^{(l)})^2\sigma_{1,3}^{(l)} (\omega^{(l)}_{\min})^{\frac{3}{2}}\min\limits_{\nu\in\{1,2,3\}}\lbrace\sigma^3_{\min}(V^{(l)}_\nu)\rbrace} \sqrt{\frac{YR\log(2\frac{Y^2+YAR}{\delta})}{N(l)}}.
\end{align*}
The columns of estimated $O^{(l)}$ matrices are up to different permutations over states, i.e. these matrices have different columns ordering. Let's assume that the number of samples for each action is such a way that satisfies $\B_O^{(l)}\leq \frac{\deltaO}{4}$, $\forall l\in[A]$. Then, one can exactly match each matrix $O^{(l)}$ with $O^{(l^*)}$ and then propagate these orders to matrices $V_2^{(l)}$ and $V_3^{(l)}$, $\forall l\in[A]$. The condition $\B_O^{(l)}\leq \frac{\deltaO}{4}$, $\forall l\in[A]$ can be represented as follow
\begin{align*}
N(l)\geq \frac{16C_O^2YR}{{\lambda^{(l)}}^2\deltaO^2}~,~~ \forall l\in[A]
\end{align*}
%
%In the following we use $\lambda^{(l)} = (\pi_{\min}^{(l)})^2\sigma_{1,3}^{(l)} (\sigma_{1,2}^{(l)})^{3/2}$ to summarize the dependency on all the parameters of the problem not related to the dimensionality of the POMDP.

\textbf{Step 4 (bound on $f_T$).}
The derivation of the bound for $\wh{f}_T$ is more complex since each distribution $\wh{f}_T(\cdot|x,a)$ is obtained as the solution of the linear system of equations in Eq.~\ref{eq:transition.recovery}, that is for any state $i$ and action $l$ we compute
\begin{align}\label{eq:est.transition}
[\wh{T}]_{i,:,l} = \wh{O}^\dagger [\wh{V}_3^{(l)}]_{:,i},
\end{align}
where $\wh{O}$ is obtained plugging in the estimate $\wh{f}_O$.\footnote{We recall that $\wh{f}_O$ corresponds to the estimate $\wh{f}_O^{(l)}$ with the tightest bound $B_O^{(l)}$.}
We first recall the following general result for the pseudo-inverse of a matrix and we instance it in our case. Let $\w$ and $\wh{\w}$ be any pair of matrix such that $\wh{\w} = \w+E$ for a suitable error matrix $E$, then we have~\cite{meng2010optimal}
\begin{equation}\label{eq:bound.inverse}
\|\w^\dagger-\wh{\w}^\dagger\|_2\leq\frac{1+\sqrt{5}}{2}\max{\bigg\{\|\w^\dagger\|_2,\|\wh{\w}^\dagger\|_2\bigg\}}\|E\|_2,
\end{equation}
where $\|\cdot\|_2$ is the spectral norm. 
%From the definition of $W$ and $V_2^{(l)}$ we have
%%
%\begin{align*}
%[\w]_{s,j} = [\w]_{n, j} &=f_O(\vec{e}_n|j),\\
%[V_2^{(l)}]_{s,i} = [V_2^{(l)}]_{(n',m'),i} &= \rho(i,l) f_\pi(l|\vec{e}_{n'}) f_O(\vec{e}_{n'}|i) f_R(\vec{e}_{m'}|i,l).
%\end{align*}
%%
Since Lemma~\ref{lemma:view.bound2} provides a bound on the error for each column of $V_2^{(l)}$ for each action and a bound on the error of $\rho(i,l)$ is already developed in Step 2, we can bound the $\ell_2$ norm of the estimation error for each column of $O$ and $\wh{O}$ as
\begin{align}\label{eq:bound.O}
\|\wh{O} - O\|_2\leq \|\wh{O} - O\|_F\leq \sqrt{X}\min_{l\in [A]}\B_O^{(l)}.
\end{align}
We now focus on the maximum in Eq.~\ref{eq:bound.inverse}, for which we need to bound the spectral norm of the pseudo-inverse of the estimated $W$. We have $\|\wh{O}^\dagger\|_2 \leq (\sigma_X(\wh{O}))^{-1}$ where $\sigma_X(\wh{O})$ is the $X$-th singular value of matrix $\wh{O}$ whose perturbation is bounded by $\|\wh{O}-O\|_2$. Since matrix $O$ has rank $X$ from Asm.~\ref{asm:observation} then
\begin{align*}
\|\wh{O}^\dagger\|_2 \leq (\sigma_X(\wh{O}))^{-1}\leq \frac{1}{\sigma_X(O)}\bigg(1+\frac{\|\wh{O}-O\|_2}{\sigma_X(O)}\bigg)\leq  \frac{1}{\sigma_X(O)}\bigg(1+\frac{\|\wh{O}-O\|_F}{\sigma_X(O)}\bigg).
\end{align*}
We are now ready to bound the estimation error of the transition tensor. From the definition of Eq.~\ref{eq:est.transition} we have that for any state $i=1,\ldots,X$ the error is bounded as
\begin{align*}
\|T_{i,:,l}-\wh{T}_{i,:,l}\|_2\leq \|T_{:,:,l}-\wh{T}_{:,:,l}\|_2\leq\|\wh{O}^\dagger - O^\dagger\|_2\|V_3^{(l)}\|_2+\|\wh{V}_3^{(l)}-V_3^{(l)}\|_2\|\wh{O}^\dagger\|_2.
\end{align*}
In Lemma~\ref{lem:view.bound} we have a bound on the $\ell_2$-norm of the error for each column of $V_3^{(l)}$, thus we have $\|\wh{V}_3^{(l)}-V_3^{(l)}\|_2 \leq\|\wh{V}_3^{(l)}-V_3^{(l)}\|_F \leq 18\sqrt{X}\epsilon_{3}(l)$. Using the bound on Eq.~\ref{eq:bound.inverse} and denoting $\|V_3^{(l)}\|_2=\sigma_{\max}(V_3^{(l)})$ we obtain
\begin{align*}
\|T_{i,:,l}-&\wh{T}_{i,:,l}\|_2 \\
&\leq \frac{1+\sqrt{5}}{2}\frac{\|\wh{O}-O\|_F}{\sigma_X(O)}\bigg(1+\frac{\|\wh{O}-O\|_F}{\sigma_X(O)}\bigg)\sigma_{\max}(V_3^{(l)}) + 18\sqrt{X}\epsilon_3(l) \frac{1}{\sigma_X(O)}\bigg(1+\frac{\|\wh{O}-O\|_F}{\sigma_X(O)}\bigg) \\
&\leq \frac{2}{\sigma_X(O)}\bigg(1+\frac{\|\wh{O}-O\|_F}{\sigma_X(O)}\bigg) \Big( \sigma_{\max}(V_3^{(l)})\|\wh{O}-O\|_F + 18\sqrt{X}\epsilon_3(l)\Big).
\end{align*}
Finally, using the bound in Eq.~\ref{eq:bound.O} and bounding $\sigma_{\max}(V_3^{(l)}) \leq \sqrt{X}$,\footnote{This is obtained by $\|V_3^{(l)}\|_2 \leq \sqrt{X} \|V_3^{(l)}\|_1 = \sqrt{X}$, since the sum of each column of $V_3^{(l)}$ is one.}
\begin{align*}
\|T_{i,:,l}-\wh{T}_{i,:,l}\|_2 &\leq \frac{4}{\sigma_X(O)} \Big( X \min_{l\in[A]}\B_O^{(l)} + 18\sqrt{X}\epsilon_3(l)\Big) \\
&\leq \frac{C_T}{\sigma_X(O)(\pi_{\min}^{(l)})^2\sigma_{1,3}^{(l)} (\omega^{(l)}_{\min})^{\frac{3}{2}}\min\limits_{\nu\in\{1,2,3\}}\lbrace\sigma^3_{\min}(V^{(l)}_\nu)\rbrace} \sqrt{\frac{X^2YR\log(8/\delta)}{N(l)}},
\end{align*}
thus leading to the final statement. Since we require all these bounds to hold simultaneously for all actions, the probability of the final statement is $1-3A\delta$. Notice that for the sake of readability in the final expression reported in the theorem we use the denominator of the error of the transition model to bound all the errors and we report the statement with probability $1-24A\delta$ are change the logarithmic term in the bounds accordingly.

\end{proof}

%%%%%%%%%%%%%%%%%%%%%%%%%%%%%%%%%%%%%%%%%%%%%%%%%%%%%%%%%%%%%%%%
%%%%%%%%%%%%%%%%%%%%%%%%%%%%%%%%%%%%%%%%%%%%%%%%%%%%%%%%%%%%%%%%
%% PROOF OF REGRET BOUND
%%%%%%%%%%%%%%%%%%%%%%%%%%%%%%%%%%%%%%%%%%%%%%%%%%%%%%%%%%%%%%%%
%%%%%%%%%%%%%%%%%%%%%%%%%%%%%%%%%%%%%%%%%%%%%%%%%%%%%%%%%%%%%%%%

% !TEX root = master.tex

%%%%%%%%%%%%%%%%%%%%%%%%%%%%%%%%%%%%%%%%%%%%%%%%%%%%%%%%%%%%%%%%
%%%%%%%%%%%%%%%%%%%%%%%%%%%%%%%%%%%%%%%%%%%%%%%%%%%%%%%%%%%%%%%%
%% PROOF OF REGRET BOUND
%%%%%%%%%%%%%%%%%%%%%%%%%%%%%%%%%%%%%%%%%%%%%%%%%%%%%%%%%%%%%%%%
%%%%%%%%%%%%%%%%%%%%%%%%%%%%%%%%%%%%%%%%%%%%%%%%%%%%%%%%%%%%%%%%

\section{Proof of Theorem~\ref{thm:regret}}\label{app:proof2}

\begin{proof}[Proof of Theorem~\ref{thm:regret}]
While the proof is similar to UCRL~\cite{jaksch2010near-optimal}, each step has to be carefully adapted to the specific case of POMDPs and the estimated models obtained from the spectral method.

\noindent\paragraph{Step 1 (regret decomposition).}
We first rewrite the regret making it explicit the regret accumulated over episodes, where we remove the burn-in phase
\begin{align*}
\text{Reg}_N &\leq \sum_{k=1}^K \bigg(\sum_{t=t^{(k)}}^{t^{(k+1)}-1} \Big(\eta^+ - r_t(x_t, \wt{\pi}_k(\vec{y}_t))\Big) + r_{\max}\psi\bigg)\\
&= \sum_{k=1}^K \sum_{t=t^{(k)}}^{t^{(k+1)}-1} \Big(\eta^+ - r_t(x_t, \wt{\pi}_k(\vec{y}_t))\Big) + r_{\max}K \psi,
\end{align*}
where $r_t(x_t, \wt{\pi}_k(\vec{y}_t))$ is the random reward observed when taking the action prescribed by the optimistic policy $\wt{\pi}_k$ depending on the observation triggered by state $x_t$.
We introduce the time steps $\T^{(k)} = \big\{t: t^{(k)}\leq t < t^{(k+1)}\big\}$, $\T^{(k)}(l) = \big\{t\in\T^{(k)}: l_t=l\}$, $\T^{(k)}(x,l) = \big\{t\in\T^{(k)}:x_t=x ,a_t=l\}$ and the counters $v^{(k)} = |\T^{(k)}|$, $v^{(k)}(l) = |\T^{(k)}(l)|$, $v^{(k)}(x,l)=|\T^{(k)}(x,l)|$, while we recall that $N^{(k)}(l)$ denotes the number of samples of action $l$ available at the beginning of episodes $k$ used to compute the optimistic policy $\wt{\pi}_k$. We first remove the randomness in the observed reward by Hoeffding's inequality as
\begin{align*}
\mathbb{P}\Bigg[ \sum_{t=t^{(k)}}^{t^{(k)}+v^{(k)}-1} r_t(x_t, \wt{\pi}_k(\vecy_t))\leq \sum\limits_{x,l}v^{(k)}(x,l)\bar{r}(x,l)-r_{\max}\sqrt{\frac{v^{(k)}\log{\frac{1}{\delta}}}{2}} \; \bigg| \;\{N^{(k)}(l)\}_{l}\Bigg]\leq \delta,
\end{align*}
where the probability is taken w.r.t.\ the reward model $f_R(\cdot|x,a)$ and observation model $f_O(\cdot|x)$, $\wb{r}(x,l)$ is the expected reward for the state-action pair $x,l$. Recalling the definition of the optimistic POMDP $\wt{M}^{(k)} = \arg\max_{M\in \M^{(k)}}\max_{\pi\in\calP} \eta(\pi; M)$, we have that $\eta^+ \leq \eta(\wt{M}^{(k)}; \wt{\pi}^{(k)}) = \wt{\eta}^{(k)}$, then applying the previous bound in the regret definition we obtain
\begin{align*}
\text{Reg}_N \leq \sum_{k=1}^K \underbrace{\sum_{x=1}^X \sum_{l=1}^A v^{(k)}(x,l)\Big(\wt{\eta}^{(k)} - \bar{r}(x,l)\Big)}_{\Delta^{(k)}} + r_{\max}\sqrt{N\log 1/\delta} + r_{\max}K\psi,
\end{align*}
with high probability, where the last term follows from Jensen's inequality and the fact that $\sum_k v^{(k)} = N$.

\noindent\paragraph{Step 2 (condition on $N(l)$).} As reported in Thm.~\ref{thm:estimates}, the confidence intervals are valid only if for each action $l=1,\ldots,A$ enough samples are available. As a result, we need to compute after how many episodes the condition in Eq.~\ref{eq:nl.condition} is satisfied (with high probability). We first roughly simplify the condition by introducing $\wb{\omega}_{\min}^{(l)} = \min_{\pi\in\calP} \min_{x\in\X} \omega_{\pi}^{(l)}(x)$ and
\begin{align*}%\label{eq:nl.condition.max}
\wb{N} := \max_{l\in[A]}\max\bigg\{\frac{4}{(\sigma_{3,1}^{(l)})^2}, \frac{16C_O^2YR}{{\lambda^{(l)}}^2\deltaO^2}, \frac{C_2^2}{{(\wb{\omega}^{(l)}_{\min})^2\min\limits_{\nu\in\{1,2,3\}}\lbrace\sigma^4_{\min}(V^{(l)}_\nu)\rbrace}}\wb{\Theta}^{(l)}\bigg\} \log(2\frac{Y^2+YAR}{\delta}).
\end{align*}
\begin{align}\label{eq:thetabar}
\wb{\Theta}^{(l)}=\max\left\{\frac{16X^{\frac{1}{3}}}{C_1^{\frac{2}{3}}(\wb{\omega}^{(l)}_{\min})^{\frac{1}{3}}} ,4,\frac{2\sqrt{2}X}{C_1^2\omega^{(l)}_{\min}\min\limits_{\nu\in\{1,2,3\}}\lbrace\sigma^2_{\min}(V^{(l)}_\nu)\rbrace}\right\},
\end{align}
We recall that at the beginning of each episode $k$, the POMDP is estimated using $N^{(k)}(l)$ which is the largest number of samples collected for action $l$ in any episode prior to $k$, i.e., $N^{(k)}(l) = \max_{k'<k} v^{(k')}(l)$. Thus we first study how many samples are likely to be collected for any action $l$ in any episode of length $v$. Let $\tau_{M,\pi}^{(l)}$ is the mean passage time between two steps where action $l$ is chosen according to policy $\pi\in\calP$ then we define $\tau_M^{(l)} = \max_{\pi\in\calP} \tau_{M,\pi}^{(l)}= \max_{\pi\in\calP} \mathbb{E}[\mathcal{T}(l,l)]$, where $\mathcal{T}(l,l)$ is random variable and represent the passing time between two steps where action $l$ is chosen according to policy $\pi\in\calP$. By Markov inequality, the probability that it takes more than $2\tau_M^{(l)}$ to take the same action $l$ is at most $1/2$. If we divide the episode of length $v$ into $v/2\tau_{M}^{(l)}$ intervals of length $2\tau_{M}^{(l)}$, we have that within each interval we have a probability of 1/2 to observe a sample from action $l$, and thus on average we can have a total of $v/4\tau_{M}^{(l)}$ samples. Thus from Chernoff-Hoeffding, we obtain that the number of samples of action $l$ is such that
\begin{align*}
\Prob\bigg\{\exists l\in[A]: \; v(l)\geq \frac{v}{4\tau_{M}^{(l)}}-\sqrt{\frac{v\log(A/\delta)}{2\tau_{M}^{(l)}}}\bigg\}\geq 1-\delta.
\end{align*}
At this point we can derive a lower bound on the length of the episode that guarantee that the desired number of samples is collected. We solve
\begin{align*}
\frac{v}{4\tau_{M}^{(l)}}-\sqrt{\frac{v\log(A/\delta)}{2\tau_{M}^{(l)}}} \geq \wb{N},
\end{align*}
and we obtain the condition
\begin{align*}
\sqrt{v} \geq \sqrt{2\tau_{M}^{(l)} \log(A/\delta)} + \sqrt{2\tau_{M}^{(l)} \log(A/\delta) + 16\tau_{M}^{(l)}}\wb{N},
\end{align*}
which can be simplified to
\begin{align}\label{eq:v.condition}
v \geq \wb{v} := 24\tau_{M}^{(l)} \wb{N} \log(A/\delta).
\end{align}
Thus we need to find a suitable number of episodes $\wt{K}$ such that there exists an episode $k'<\wt{K}$ such that $v^{(k')}$ satisfies the condition in Eq.~\ref{eq:v.condition}. Since an episode is terminated when an action $l$ ($v^{(k)}(l)$) is selected twice the number of samples available at the beginning of the episode ($N^{(k)}(l)$), we have that at episode $k$ there was an episode in the past ($k'<k$) with at least $2^c$ steps with $c = \max\{n\in\mathbb{N}: An \leq k\}$, where $A$ is the number of actions (i.e., after $A c$ episodes there was at least one episode in which an action reached $2^c$ samples, which forced the episode to be at least that long). From condition in Eq.~\ref{eq:v.condition}, we need $2^c \geq \wb{v}$, which in turn gives $\wt{K} \geq A \log_2(\wb{v})$, which finally implies that $\wt{K} \leq A \log_2(\wb{v}) + 1$ is a sufficient condition on the number of episodes needed to guarantee that all the actions have been selected enough so that the condition of Thm.~\ref{thm:estimates} is satisfied. We are just left with measuring the regret accumulated over the first $\wt{K}$ episodes, that is
\begin{align}
\sum_{k=1}^{\wt{K}+1} \sum_{t=t^{(k)}}^{t^{(k+1)}-1} \Big(\eta^+ - r_t(x_t, \wt{\pi}_k(\vec{y}_t))\Big) \leq r_{\max}\sum_{k=1}^{\wt{K}+1} v^{(k)} \leq r_{\max}\sum_{k=1}^{\wt{K}+1} A2^k \leq 4r_{\max} A2^{\wt{K}} \leq 4Ar_{\max} (\wb{v}+1),
\label{eq:constantterm}
\end{align}
%\todo[inline]{need to add A, mentioned that it is good and there is no need to know any thing about the model parameters to start learning and the regret part would take care of that.}
%
where in the first step we maximize the per-step regret by $r_{\max}$ and then we use a rough upper bound for the length of each episode (as if the length is doubled at each episode) and finally we use the upper-bound on $\wt{K}$.

\noindent\paragraph{Step 3 (failing confidence intervals).}
Even after the first $\wt{K}$ episodes, the confidence intervals used to construct the set of POMDPs $\M^{(k)}$ may not be correct, which implies that the true POMDP $M$ is not contained in $\M^{(k)}$. We now bound the regret in the case of failing confidence intervals from Thm.~\ref{thm:estimates}. We have
\begin{align*}
R^{\text{fail}} = \sum_{k=1}^K \bigg(\sum_{t=t^{(k)}}^{t^{(k+1)}-1} \Big(\eta^+ - r_t(x_t, \wt{\pi}_k(\vec{y}_t))\Big)\I_{M\notin\M^{(k)}}\bigg)\leq r_{\max}\sum_{k=1}^K v^{(k)}\I_{M\notin\M^{(k)}}\leq r_{\max}\sum_{t=1}^Nt\I_{M\notin\M^{(t)}},
\end{align*}
where $\M^{(t)}$ denotes the set of admissible POMDPs according to the samples available at time $t$. We recall from Step 2 that the number of steps needed for the statement of Thm.~\ref{thm:estimates} to be valid is $\wb{t} = 4 (\wb{v}+1)$. If $N$ is large enough so that $\wb{t} \leq N^{1/4}$
%\todoaout{Maybe we should report this condition in the statement, although basically the result holds also in the other case}\todo{I dont think we need it, because we already said, in Thm 1, that N should be larger than some value. In the sense of regret, we take into account the regret due to this period 2 times.}
, then we bound the regret as
\begin{align*}
R^{\text{fail}} \leq \sum_{t=1}^{\lfloor N^{1/4}\rfloor}t\I_{M\notin\M^{(t)}} + \sum_{t=\lfloor N^{1/4}\rfloor+1}^Nt\I_{M\notin\M^{(t)}}\leq \sqrt{N}+\sum_{t=\lfloor N^{1/4}\rfloor+1}^Nt\I_{M\notin\M^{(t)}}.
\end{align*}
We are left with bounding the last term. We first notice that if we redefine the confidence intervals in Thm.~\ref{thm:estimates} by substituting the term $\log(1/\delta)$ by $\log(t^6/\delta)$, we obtain that at any time instants $t$, the statement holds with probability $1-24A\delta/t^6$. Since
\begin{align*}
\sum_{t=\lfloor N^{1/4}\rfloor+1}^N\frac{24A}{t^6}\leq \frac{24A}{N^{6/4}}+\int_{\lfloor N^{1/4}\rfloor}^{\infty}\frac{24A}{t^6}dt=\frac{24A}{N^{6/4}}+\frac{24A}{5N^{5/4}}\leq\frac{144A}{5N^{5/4}} \leq\frac{30A}{N^{5/4}},
\end{align*}
then $M$ is in the set of $\M^{(k)}$ at any time step $\lfloor N^{1/4}\rfloor\leq t\leq N$ with probability $1-30A\delta/N^{5/4}$. As a result, the regret due to failing confidence bound is bounded by $\sqrt{N}$ with probability $1-30A\delta/N^{5/4}$.

\noindent\paragraph{Step 4 (reward model).}
Now we focus on the per-episode regret $\Delta^{(k)}$ for $k> \wb{K}$ when $M$ is contained in $\wt{\mathcal{M}}_k$ and we decompose it in two terms
\begin{align*}
\Delta^{(k)} \leq \sum_{x=1}^X \sum_{l=1}^A v^{(k)}(x,l)\underbrace{\Big(\wt{\eta}^{(k)} - \wt{r}^{(k)}(x,l)\Big)}_{(a)} + \underbrace{\sum_{x=1}^X \sum_{l=1}^A v^{(k)}(x,l)\Big(\wt{r}^{(k)}(x,l) - \bar{r}(x,l)\Big)}_{(b)},
\end{align*}
where $\wt{r}^{(k)}$ is the state-action expected reward used in the optimistic POMDP $\wt{\M}^{(k)}$. We start by bounding the second term, which only depends on the size of the confidence intervals in estimating the reward model of the POMDP. We have
\begin{align*}
(b) &\leq \sum_{l=1}^A \sum_{x=1}^Xv^{(k)}(x,l)\max_{x'\in\X}\Big|\wt{r}^{(k)}(x',l) - \bar{r}(x',l)\Big| \\
&=\sum_{l=1}^A v^{(k)}(l)\max_{x\in\X}\Big|\wt{r}^{(k)}(x,l) - \bar{r}(x,l)\Big|\\
&\leq2\sum_{l=1}^A v^{(k)}(l) \B_R^{(k,l)},
\end{align*}
where $\B_R^{(k,l)}$ corresponds to the term $\B_R^{(l)}$ in Thm.~\ref{thm:estimates} computed using the $N^{(k)}(l)$ samples collected during episode $k^{(l)} = \arg\max_{k'<k} v^{(k')}(l)$.

\noindent\paragraph{Step 5 (transition and observation models).}
We now proceed with studying the first term $(a)$, which compares the (optimal) average reward in the optimistic model $\wt{\M}^{(k)}$ and the (optimistic) rewards collected on the states traversed by policy $\wt{\pi}^{(k)}$ in the true POMDP. We first recall the Poisson equation. For any POMDP $M$ and any policy $\pi$, the action value function $Q_{\pi,M}: \X\times \A\rightarrow \Re$ satisfies
\begin{align}\label{eq:poisson}
Q_{\pi,M}(x,a) &= \wb{r}(x,a) - \eta_\pi + \sum_{x'} f_{T}(x'|x,a) \Big(\sum_{a'} f_{\pi}(a'|x') Q_{\pi,M}(x',a')\Big)\\
&\Rightarrow \eta_\pi - \wb{r}(x,a) = \sum_{x'} f_{T}(x'|x,a) \Big(\sum_{a'} f_{\pi}(a'|x') Q_{\pi,M}(x',a')\Big) - Q_{\pi,M}(x,a),\nonumber
\end{align}
where $f_{\pi}(a'|x') = \sum_y f_O(y|x')f_\pi(a'|y)$ and terms such as $\wb{r}$ and $f_{T}$ depend on the specific POMDP. We define the function $\wb{Q}_{\pi,M}(x,l)$ as
\begin{align*}
\wb{Q}_{\pi,M}(x,l)={Q}_\pi(x,l)-\frac{\min\limits_{x,l}Q_{\pi,M}(x,l)-\max\limits_{x,l}Q_{\pi,M}(x,l)}{2},
\end{align*}
which is a centered version of $Q_{\pi,M}(x,l)$. In order to characterize $\wb{Q}$, we introduce a notion of diameter specific to POMDPs and the family of policies considered in the problem
\begin{align*}
D:=\max_{x,x'\in[X],~l,l'\in [A]}\min_{\pi\in\calP}\mathbb{E}[T(x',l'|M,\pi,x,l)],
\end{align*}
where $T(x',l'|M,\pi,x,l)$ is the (random) time that takes to move from state $x$ by first taking action $l$ and then following policy $\pi$ before reaching state $x'$ and performing action $l'$. An important feature of the diameter is that it can be used to upper bound the range of the function $Q_{\pi,M}$ computed using a policy derived from Eq~\ref{eq:optimistic.policy} in an optimistic model. The proof of this fact is similar to the case of the diameter for MDPs. We first recall the definition of the optimistic policy
\begin{align}
\wt{\pi}^{(k)} = \arg\max_{\pi\in\calP}\max_{M\in\M^{(k)}} \eta(\pi; M),
\end{align}
while $M_k$ is the optimistic model. The joint choice of the policy and the model can be seen as if a POMDP $\wt{M}^+$ with augmented action space $\A'$ is considered. Taking an action $a'\in\A'$ in a state $x$ corresponds to a basic action $a\in\A$ and a choice of transition, reward, and observation model from $\mathcal{M}^{(k)}$. We denote by $\calP^{(\M^{(k)})}$ the corresponding augmented policy space using $\calP$ and the set of admissible POMDPs $\mathcal{M}^{(k)}$. As a result, for any augmented policy $\wt{\pi}^+$ executed in $\wt{M}^+$ we obtain transitions, rewards, and observations that are equivalent to executing a (standard) policy $\wt{\pi}$ in a specific POMDP $\wt{M}\in\mathcal{M}^{(k)}$ and viceversa. As a result, computing $\wt{\pi}^{(k)}$ and the corresponding optimistic model $\wt{M}^{(k)}$ is equivalent to choosing the optimal policy in the POMDP $\wt{M}^+$. Since the true POMDP of diameter $D$ is in $\mathcal{M}^{(k)}$ with high-probability, then the diameter of the augmented POMDP $\wt{M}^+$ is at most $D$. Furthermore, we can show that the optimal policy has a Q-value with range bounded by $D$. Let us assume that there exists state-action pairs $(x,a), (x',a')$ such that $Q_{\wt{\pi}^{(k)},\wt{M}^{(k)}}(x,a)-Q_{\wt{\pi}^{(k)},\wt{M}^{(k)}}(x',a')\geq r_{\max}D$. Then it is easy to construct a policy different from $\wt{\pi}^{(k)}$ which achieves a better Q-value. We already know by definition of diameter that there exists a policy moving from $x$ to $x'$ in $D$ steps on average. If from $x'$ the optimal policy is followed, then only $r_{\max}D$ reward could have been missed at most and thus the difference in action-value function between $x$ and $x'$ cannot be larger than $r_{\max}D+1$, thus contradicting the assumption. As a result, we obtain
\begin{align}\label{eq:dinequ}
\max_{x,a}Q_{\wt{\pi}^{(k)},\wt{M}^{(k)}}(x,a)-\min_{x,a} Q_{\wt{\pi}^{(k)},\wt{M}^{(k)}}(x,a)\leq r_{\max}D
\end{align}
and thus
\begin{align*}
\wb{Q}_{\wt{\pi}^{(k)},\wt{M}^{(k)}}(x,a)(x,l)\leq r_{\max}\frac{(D+1)}{2}.
\end{align*}
By replacing $Q$ with $\wb{Q}$ in the Poisson equation for the optimistic POMDP characterized by the transition model $\wt{f}^{(k)}_{T}$ and where the observation model is such that the policy $\wt{\pi}_k$ takes actions according to the distribution $\wt{f}^{(k)}_{{\wt{\pi}^{(k)}}}(\cdot|x)$, we obtain
\begin{align*}
(a) &= \sum_{x'} \wt{f}^{(k)}_{T}(x'|x,l) \Big(\sum_{l'} \wt{f}^{(k)}_{{\wt{\pi}^{(k)}}}(a'|x') \wb{Q}_{\wt{\pi}^{(k)},\wt{M}^{(k)}}(x',a')\Big) - \wb{Q}_{\wt{\pi}^{(k)},\wt{M}^{(k)}}(x,l)\\
&= \sum_{x'} \wt{f}^{(k)}_{T}(x'|x,l) \Big(\sum_{l'} \wt{f}^{(k)}_{{\wt{\pi}^{(k)}}}(l'|x') \wb{Q}_{\wt{\pi}^{(k)},\wt{M}^{(k)}}(x',l')\Big)\\
&- \sum_{x'} f_{T}(x'|x,l) \Big(\sum_{l'} f_{{\wt{\pi}^{(k)}}}(l'|x') \wb{Q}_{\wt{\pi}^{(k)},\wt{M}^{(k)}}(x',l')\Big) \\
&\quad\quad+ \sum_{x'} f_{T}(x'|x,l) \Big(\sum_{l'} f_{{\wt{\pi}^{(k)}}}(l'|x') \wb{Q}_{\wt{\pi}^{(k)},\wt{M}^{(k)}}(x',l')\Big) - \wb{Q}_{\wt{\pi}^{(k)},\wt{M}^{(k)}}(x,l)\\
&= \underbrace{\sum_{x'} \sum_{l'} \big( \wt{f}^{(k)}_{T}(x'|x,l) \wt{f}^{(k)}_{{\wt{\pi}^{(k)}}}(l'|x') - f_{T}(x'|x,l) f_{{\wt{\pi}^{(k)}}}(l'|x')\big) \wb{Q}_{\wt{\pi}^{(k)},\wt{M}^{(k)}}(x',l')}_{(c)} \\
&\quad\quad+ \underbrace{\sum_{x'} f_{T}(x'|x,l) \Big(\sum_{l'} f_{{\wt{\pi}^{(k)}}}(l'|x') \wb{Q}_{\wt{\pi}^{(k)},\wt{M}^{(k)}}(x',l')\Big) - \wb{Q}_{\wt{\pi}^{(k)},\wt{M}^{(k)}}(x,l)}_{\zeta^{(k)}(x,l)}.
\end{align*}
The term $(c)$ can be further expanded as
\begin{align*}
(c) &= \sum_{x'} \sum_{l'}\Big( \big( \wt{f}^{(k)}_{T}(x'|x,l)-f_{T}(x'|x,l)\big) \wt{f}^{(k)}_{{\wt{\pi}^{(k)}}}(l'|x') - f_{T}(x'|x,l) \big(\wt{f}^{(k)}_{{\wt{\pi}^{(k)}}}(l'|x')-f_{{\wt{\pi}^{(k)}}}(l'|x')\big) \Big)\wb{Q}_{\wt{\pi}^{(k)},\wt{M}^{(k)}}(x',l')\\
&\leq \Big(\underbrace{\sum_{x'} \big| \wt{f}^{(k)}_{T}(x'|x,l)-f_{T}(x'|x,l)\big|}_{(d)} +\sum_{x'} f_{T}(x'|x,l) \underbrace{\sum_{l'}\big|\wt{f}^{(k)}_{{\wt{\pi}^{(k)}}}(l'|x')-f_{{\wt{\pi}^{(k)}}}(l'|x')\big|}_{(d')}  \Big) \|\wb{Q}_{\wt{\pi}^{(k)},\wt{M}^{(k)}}\|_\infty,
\end{align*}
where we used the fact that $\sum_{l'}\wt{f}^{(k)}_{{\wt{\pi}^{(k)}}}(l'|x')=1$.
For the first term we can directly apply the bound from Thm.~\ref{thm:estimates}, Eq.~\ref{eq:transition.bound} and obtain%\todoaout{Double-check is coherent with the statement in thm1.}
\begin{align*}
(d) &= \|\wt{f}^{(k)}_{T}(\cdot|x,l)-f_{T}(\cdot|x,l)\|_1 \leq2\sqrt{X} \B_T^{(k,l)}.
\end{align*}
As for $(d')$, the error in estimating the observation model is such that
\begin{align*}
&(d')=\sum_{l'}\sum_y |\wt{f}^{(k)}_O(y|x')- f_O(y|x')| f_{\wt{\pi}^{(k)}}^{(k)}(l'|y)= \sum_y |\wt{f}^{(k)}_O(y|x')- f_O(y|x')|\leq2 \B_O^{(k)}.
\end{align*}
Plugging back these two bounds into $(c)$ together with the bound on $\wb{Q}_{\wt{\pi}^{(k)},\wt{M}^{(k)}}(x,l)$, we obtain
\begin{align*}
(c) &\leq 2(\sqrt{X}\B_T^{(k,l)} + \B_O^{(k)})r_{\max}\frac{(D+1)}{2}.
\end{align*}
The term $(a)$ in the per-episode regret is thus bounded as
\begin{align*}
(a) \leq 2\big(\sqrt{X}\B_T^{(k,l)} + \B_O^{(k)}\big)r_{\max}\frac{(D+1)}{2} + \zeta^{(k)}(x,l).
\end{align*}

\noindent\paragraph{Step 6 (Residual error).}
We now bound the cumulative sum of the terms $\zeta^{(k)}(x,l)$. At each episode $k$ we have
\begin{align*}
\sum_{x=1}^X \sum_{l=1}^A v^{(k)}(x,l) \zeta^{(k)}(x,l) = \sum_{t=t^{(k)}}^{t^{(k+1)}}\sum_{x'} f_{T}(x'|x_t,l_t) \Big(\sum_{l'} f_{{\wt{\pi}^{(k)}}}(l'|x') \wb{Q}_{\wt{\pi}^{(k)},\wt{M}^{(k)}}(x',l')\Big) - \wb{Q}_{\wt{\pi}^{(k)},\wt{M}^{(k)}}(x_t,l_t),
\end{align*}
we introduce the term $\wb{Q}_{\wt{\pi}^{(k)},\wt{M}^{(k)}}(x_{t+1},l_{t+1})$ and we obtain two different terms
\begin{align*}
&\sum_{x=1}^X \sum_{l=1}^A v^{(k)}(x,l) \zeta^{(k)}(x,l) \\
&=\sum_{t=t^{(k)}}^{t^{(k+1)}}\sum_{x'} f_{T}(x'|x_t,l_t) \Big(\sum_{l'} f_{{\wt{\pi}^{(k)}}}(l'|x') \wb{Q}_{\wt{\pi}^{(k)},\wt{M}^{(k)}}(x',l')\Big) - \wb{Q}_{\wt{\pi}^{(k)},\wt{M}^{(k)}}(x_{t+1},l_{t+1}) \\
&\quad+ \wb{Q}_{\wt{\pi}^{(k)},\wt{M}^{(k)}}(x_{t+1},l_{t+1}) - \wb{Q}_{\wt{\pi}^{(k)},\wt{M}^{(k)}}(x_t,l_t)\\
&\leq \sum_{t=t^{(k)}}^{t^{(k+1)}} \underbrace{\sum_{x'} f_{T}(x'|x_t,l_t) \Big(\sum_{l'} f_{{\wt{\pi}^{(k)}}}(l'|x') \wb{Q}_{\wt{\pi}^{(k)},\wt{M}^{(k)}}(x',l')\Big) - \wb{Q}_{\wt{\pi}^{(k)},\wt{M}^{(k)}}(x_{t+1},l_{t+1})}_{Y_t} + r_{\max}D,
\end{align*}
where we use the fact that the range of $\wb{Q}_{\wt{\pi}^{(k)},\wt{M}^{(k)}}$ is bounded by the diameter $D$. We notice that $\E[Y_t|x_1,a_1,y_1,\ldots,x_t,a_t,y_t] = 0$ and $|Y_t| \leq r_{\max}D$, thus $Y_t$ is a martingale difference sequence and we can use Azuma's inequality to bound its cumulative sum. In fact, we have
\begin{align*}
\sum_{t=1}^N Y_t \leq D\sqrt{2N\log(N^{5/4}/\delta)}
\end{align*}
with probability $1-\delta/N^{5/4}$. As a result we can now bound the total sum of the terms $\zeta^{(k)}$ as
\begin{align*}
\sum_{x=1}^X \sum_{l=1}^A v^{(k)}(x,l) \zeta^{(k)}(x,l) \leq \sum_{t=1}^N Y_t + r_{\max}KD \leq r_{\max}D\sqrt{2N\log(N^{5/4}/\delta)} + r_{\max}KD.
\end{align*}

\noindent\paragraph{Step 7 (per-episode regret).}
We can now bound the per-episode regret as
\begin{align*}
\Delta^{(k)} \leq \sum_{l=1}^A v^{(k)}(l)2\bigg(\B_R^{(k,l)} + \big(\sqrt{X}\B_T^{(k,l)} + \B_O^{(k)}\big)r_{\max}\frac{(D+1)}{2}\bigg).
\end{align*}
Recalling the results from Thm.~\ref{thm:estimates}, we can bound the first term in the previous expression as
\begin{align*}
\Delta^{(k)} \leq {3r_{\max}(D+1)}\sqrt{d'\log(N^6/\delta)}\big(C_O+C_R+C_TX^{3/2}\big)\sum_{l=1}^A \frac{v^{(k)}(l)}{\lambda^{(k,l)}} \sqrt{\frac{1}{N^{(k)}(l)}}.
\end{align*}
Since the number of samples $N^{(k)}(l)$ collected in the previous episode is at most doubled in the current episode $k$, we have that $N^{(k)}(l) \geq v^{(k)}(l)/2$, then we obtain
\begin{align*}
\Delta^{(k)} \leq {9r_{\max}(D+1)}\sqrt{v^{(k)}d'\log(N^6/\delta)}\big(C_O+C_R+C_TX^{3/2}\big) \max_{l=1,\ldots,A}\frac{1}{\lambda^{(k,l)}}.
\end{align*}

\noindent\paragraph{Step 8 (bringing all together).}
Now we have to recollect all the previous terms: the number of episodes needed to use Thm.~\ref{thm:estimates} (Step 2), regret in case of failing confidence intervals (Step 3), and the per-episode regret (Step 7). The result is
\begin{align*}
&\text{Reg}_N \\
&\leq r_{\max}\Big(\underbrace{\sqrt{N\log(N^6/\delta)}}_{\text{Step 1}} + \underbrace{4Ar_{\max}(\wb{v}+1)}_{\text{Step 2}} + \underbrace{\sqrt{N}}_{\text{Step 3}} + \underbrace{D\sqrt{2N\log(N^{5/4}/\delta)} + r_{\max}KD}_{\text{Step 6}}\Big) + \sum_{k=\wt{K}+1}^K \Delta^{(k)}.
\end{align*}
The last term can be bounded as
\begin{align*}
\sum_{k=\wt{K}+1}^K \Delta^{(k)} \leq \frac{9r_{\max}(D+1)}{\wb{\lambda}}\sqrt{Nd'\log(N^6/\delta)}\big(C_O+C_R+C_TX^{3/2}\big).
\end{align*}
where $\wb{\lambda} = \min_{k,l}\lambda^{(k,l)}$ and it is defined as in the statement of the theorem.
Since $K$ is a random number, we need to provide an upper-bound on it. We can use similar arguments as in Step 2. Given the stopping criterion of each episode, at most every $A$ steps, then length of an episode is doubled. As a result, after $K$ episodes, we have these inequalities
\begin{align*}
N = \sum_{k=1}^K v^{(k)} \geq \sum_{k'=1}^{K/A}2^{k'} \geq 2^{K/A}.
\end{align*}
As a result, we obtain the upper bound $K \leq \wb{K}_N \leq A \log_2 N$.
Bringing all the bounds together we obtain the final statement with probability $1-\delta / (4N^{5/4})$.

\end{proof}

%%%%%%%%%%%%%%%%%%%%%%%%%%%%%%%%%%%%%%%%%%%%%%%%%%%%%%%%%%%%%%%%
%%%%%%%%%%%%%%%%%%%%%%%%%%%%%%%%%%%%%%%%%%%%%%%%%%%%%%%%%%%%%%%%
%% LIMITED MEMORY
%%%%%%%%%%%%%%%%%%%%%%%%%%%%%%%%%%%%%%%%%%%%%%%%%%%%%%%%%%%%%%%%
%%%%%%%%%%%%%%%%%%%%%%%%%%%%%%%%%%%%%%%%%%%%%%%%%%%%%%%%%%%%%%%%

%\input{Cextension}

% !TEX root = master.tex

%%%%%%%%%%%%%%%%%%%%%%%%%%%%%%%%%%%%%%%%%%%%%%%%%%%%%%%%%%%%%%%%
%%%%%%%%%%%%%%%%%%%%%%%%%%%%%%%%%%%%%%%%%%%%%%%%%%%%%%%%%%%%%%%%
%% PROOF OF REGRET BOUND
%%%%%%%%%%%%%%%%%%%%%%%%%%%%%%%%%%%%%%%%%%%%%%%%%%%%%%%%%%%%%%%%
%%%%%%%%%%%%%%%%%%%%%%%%%%%%%%%%%%%%%%%%%%%%%%%%%%%%%%%%%%%%%%%%

\section{Proof of Remark 2 in Section~\ref{s:learning}}\label{s:proof.remark}

We first prove the bound on the transition tensor, which requires re-deriving step 4 in the proof of Thm.~\ref{thm:estimates}.

\begin{proof}

\paragraph{Step 4 (bound on $f_T$).}
The derivation of the bound for $\wh{f}_T$ is more complex since each distribution $\wh{f}_T(\cdot|x,a)$ is obtained as the solution of the linear system of equations like Eq.~\ref{eq:transition.recovery.inv}, that is for any state $i$ and action $l$ we compute

\begin{align*}
[T]_{i,:,l} = W^\dagger [V_3^{(l)}]_{:,i},
\end{align*}
and derive transition tensor as follows

\begin{align}
[\wh{T}]_{i,:,l} = \wh{W}^\dagger [\wh{V}_3^{(l)}]_{:,i},
\end{align}
where $\wh{W}$ is obtained plugging in the estimates of $\wh{f}_O$ and $\wh{f}_R$ and the policy $f_\pi$.
We first recall the following general result for the pseudo-inverse of a matrix and we instance it in our case. Let $W$ and $\wh{W}$ be any pair of matrix such that $\wh{W} = W+E$ for a suitable error matrix $E$, then we have~\cite{meng2010optimal}
\begin{equation}\label{eq:bound.inverse2}
\|W^\dagger-\wh{W}^\dagger\|_2\leq\frac{1+\sqrt{5}}{2}\max{\bigg\{\|W^\dagger\|_2,\|\wh{W}^\dagger\|_2\bigg\}}\|E\|_2,
\end{equation}
where $\|\cdot\|_2$ is the spectral norm. From the definition of $W$ and $V_2^{(l)}$ we have
\begin{align*}
[W]_{s,j} = [W]_{(n,m,k), j} &= f_\pi(k| \vec{e}_n) f_R(\vec{e}_m | j, k) f_O(\vec{e}_n|j),\\
[V_2^{(l)}]_{s,i} = [V_2^{(l)}]_{(n',m'),i} &= \rho(i,l) f_\pi(l|\vec{e}_{n'}) f_O(\vec{e}_{n'}|i) f_R(\vec{e}_{m'}|i,l).
\end{align*}
Then it is clear that any column $j$ of $W$ is the result of stacking the matrices $V_2^{(l)}$ over actions properly re-weighted by $\rho(i,l)$, that is
\begin{align*}
[W]_{:,j} = \bigg[ \frac{[V_2^{(1)}]_{:,j}^{\top}}{\rho(j,1)} \enspace\cdots\enspace \frac{[V_2^{(l)}]_{:,j}^{\top}}{\rho(j,l)} \enspace\cdots\enspace \frac{[V_2^{(A)}]_{:,j}^{\top}}{\rho(j,A)} \bigg]^\top.
\end{align*}
The same relationship holds for $\wh{W}$ and $\wh{V}_2^{(l)}$. Since Lemmas ~\ref{lemma:view.bound2} and Eq.~\ref{eq:rhobound} provide a bound on the error for each column of $V_2^{(l)}$ for each action and a bound on the error of $\rho(i,l)$ is already developed in Step 2, we can bound the $\ell_2$ norm of the estimation error for each column of $W$ and $\wh W$ as
\begin{align*}
\|[\wh{W}]_{:,i}- [W]_{:,i}\|_2^2 = \sum\limits_l^A \sum\limits_{m',n'}^{R,Y}\bigg(\frac{[\wh{V}_2^{(l)}]_{(n',m'),i}}{\wh{\rho}(i,l)}-\frac{[{V}_2^{(l)}]_{(n',m'),i}}{{\rho(i,l)}}\bigg)^2.
\end{align*}
Following similar steps as in Step 3, each summand can be bounded as
\begin{align*}
&\bigg| \frac{[\wh{V}_2^{(l)}]_{(n',m'),i}}{\wh{\rho}(i,l)}-\frac{[{V}_2^{(l)}]_{(n',m'),i}}{{\rho(i,l)}} \bigg| \leq \Big| [\wh{V}_2^{(l)}]_{(n',m'),i}-[{V}_2^{(l)}]_{(n',m'),i} \Big| + \bigg|\frac{1}{\rho(i,l)}-\frac{1}{\wh{\rho}(i,l)}\bigg| [{V}_2^{(l)}]_{(n',m'),i}.
\end{align*}
Then the $\ell_2$-norm of the error is bounded as
\begin{align*}
\|[\wh{W}]_{:,i}- [W]_{:,i}\|_2 &\leq \sqrt{\sum_{l=1}^A \|[\wh{V}_2^{(l)}]_{:,i}-[{V}_2^{(l)}]_{:,i}\|_2^2} + \sqrt{\sum_{l=1}^A \bigg(\frac{1}{\rho(i,l)}-\frac{1}{\wh{\rho}(i,l)}\bigg)^2 \sum\limits_{m',n'}^{R,Y}[{V}_2^{(l)}]^2_{(n',m'),i}}\\
&\leq \sqrt{\sum_{l=1}^A \|[\wh{V}_2^{(l)}]_{:,i}-[{V}_2^{(l)}]_{:,i}\|_2^2} + \sqrt{\sum_{l=1}^A \bigg(\frac{1}{\rho(i,l)}-\frac{1}{\wh{\rho}(i,l)}\bigg)^2}\\
&\leq 20\sqrt{\sum_{l=1}^A \frac{\epsilon_3(l)}{\sigma_{3,1}^{(l)}}^2} + \sqrt{\sum_{l=1}^A \epsilon_\rho^2(i,l)}\\
&\leq \sum_{l=1}^A \big(20\frac{\epsilon_3(l)}{\sigma_{3,1}^{(l)}} + \epsilon_\rho(i,l)\big) \leq 40\sqrt{YR} \sum_{l=1}^A \frac{\epsilon_3(l)}{\sigma_{3,1}^{(l)}\pi_{\min}^{(l)}}.
\end{align*}
Now we can bound the spectral norm of the error in estimating $W$ as
\begin{align}\label{eq:bound.A}
\|\wh{W} - W\|_2\leq \|\wh{W} - W\|_F\leq 40\sqrt{XYR}\sum_{l=1}^A\frac{\epsilon_3(l)}{\sigma_{3,1}^{(l)}\pi_{min}^{(l)}}.
\end{align}
We now focus on the maximum in Eq.~\ref{eq:bound.inverse}, for which we need to bound the spectral norm of the pseudo-inverse of the estimated $W$. We have $\|\wh{W}^\dagger\|_2 \leq (\sigma_X(\wh{W}))^{-1}$ where $\sigma_X(\wh{W})$ is the $X$-th singular value of matrix $\wh{W}$ whose perturbation is bounded by $\|\wh{W}-W\|_2$. Since matrix $W$ is a rank $X$ matrix on Asm.~\ref{asm:observation} then
\begin{align*}
\|\wh{W}^\dagger\|_2 \leq (\sigma_X(\wh{W}))^{-1}\leq \frac{1}{\sigma_X(W)}\bigg(1+\frac{\|\wh{W}-W\|_2}{\sigma_X(W)}\bigg)\leq  \frac{1}{\sigma_X(W)}\bigg(1+\frac{\|\wh{W}-W\|_F}{\sigma_X(W)}\bigg).
\end{align*}
We are now ready to bound the estimation error of the transition tensor. From the definition of Eq.~\ref{eq:est.transition} we have that for any state $i=1,\ldots,X$ the error is bounded as
\begin{align*}
\|T_{i,:,l}-\wh{T}_{i,:,l}\|_2\leq \|T_{:,:,l}-\wh{T}_{:,:,l}\|_2\leq\|\wh{W}^\dagger - W^\dagger\|_2\|V_3^{(l)}\|_2+\|\wh{V}_3^{(l)}-V_3^{(l)}\|_2\|\wh{W}^\dagger\|_2.
\end{align*}
In Lemma~\ref{lem:view.bound} we have a bound on the $\ell_2$-norm of the error for each column of $V_3^{(l)}$, thus we have $\|\wh{V}_3^{(l)}-V_3^{(l)}\|_2 \leq\|\wh{V}_3^{(l)}-V_3^{(l)}\|_F \leq 18\sqrt{X}\epsilon_{3}(l)$. Using the bound on Eq.~\ref{eq:bound.inverse} and denoting $\|V_3^{(l)}\|_2=\sigma_{\max}(V_3^{(l)})$ we obtain
\begin{align*}
\|T_{i,:,l}-\wh{T}_{i,:,l}\|_2 &\leq \frac{1+\sqrt{5}}{2}\frac{\|\wh{W}-W\|_F}{\sigma_X(W)}\bigg(1+\frac{\|\wh{W}-W\|_F}{\sigma_X(W)}\bigg)\sigma_{\max}(V_3^{(l)}) + 18\sqrt{X}\epsilon_3(l) \frac{1}{\sigma_X(W)}\bigg(1+\frac{\|\wh{W}-W\|_F}{\sigma_X(W)}\bigg) \\
&\leq \frac{2}{\sigma_X(W)}\bigg(1+\frac{\|\wh{W}-W\|_F}{\sigma_X(W)}\bigg) \Big( \sigma_{\max}(V_3^{(l)})\|\wh{W}-W\|_F + 18\sqrt{X}\epsilon_3(l)\Big)
\end{align*}
Using the bound in Eq.~\ref{eq:bound.A} and $\sigma_{\max}(V_3^{(l)}) \leq \sqrt{X}$ we obtain
\begin{align*}
\|T_{i,:,l}-\wh{T}_{i,:,l}\|_2 &\leq \frac{4}{\sigma_X(W)}\Big( 40\sqrt{X^2YR}\sum_{l=1}^A\frac{\epsilon_3(l)}{\sigma_{3,1}^{(l)}\pi_{min}^{(l)}} + 18\sqrt{X}\epsilon_3(l)\Big)\\
&\leq C_T\sqrt{AX^2YR} \max_{l'=1,\ldots, A} \frac{1}{\lambda^{l'}} \sqrt{\frac{\log(8/\delta)}{N_{l'}}},
\end{align*}
thus leading to the final statement for the bound over confidence of transition tensor.
\end{proof}

We now move to analyzing how the new estimator for the transition tensor affects the regret of the algorithm. The proof is exactly the same as in Thm.~\ref{thm:regret} except for step 8.

\begin{proof}
\noindent\paragraph{per-episode regret:}
The per-episode regret is bounded as
\begin{align*}
\Delta^{(k)} \leq 2\sum_{l=1}^A v^{(k)}(l)\bigg(\B_R^{(k,l)} + \big(\B_T^{(k,l)} + \B_O^{(k)}\big)r_{\max}\frac{(D+1)}{2}\bigg).
\end{align*}
All the terms can be treated as before except for the cumulative regret due to the transition model estimation error. We define
$\Delta_N =\sum_{k=\wt{K}+1}^K \sum_{l=1}^A v^{(k)}(l)\big(\B_T^{(k,l)}\big)r_{\max}{(D+1)}$, which gives
\begin{align*}
\Delta_T=\sum_kr_{max}{(D+1)}\sum_{l=1}^Av^{(k)}(l)\sqrt{X} \max_{l'=1,\ldots,A}\frac{C_T^{(k)} \sqrt{X^2YR}}{\lambda^{(k)(l')}} \sqrt{\frac{\log(N^6/\delta)}{v^{(k)}(l')}}.
\end{align*}
Let $\tau_{M,\pi}^{(l)}$ the mean passage time between two steps where action $l$ is chosen according to policy $\pi\in\calP$ and restate a  $\pi$-diameter ration $\Dr^{\pi}$
\begin{align*}
\Dr^{\pi}=\frac{\max_{l\in\mathcal{A}}\tau_{M,\pi}^{(l)}}{\min_{l\in\mathcal{A}}\tau_{M,\pi}^{(l)}}
\end{align*}
and $\Dr$
\begin{align*}
\Dr=\displaystyle\max_{\pi\in\mathcal{P}}\Dr^{\pi}.
\end{align*}
We need the following lemma which directly follows from Chernoff-Heoffding inequality.

\begin{lemma}\label{lem:visits}
By Markovian inequality, the probability that during $2\tau_{M,\pi}^{(l)}$ the action $l$ is not visited is less than $\frac{1}{2}$. Then during episode $k$
\begin{align*}
\mathbb{P}\Big\{v^{(k)}(l)\leq \frac{1}{2}\frac{v^{(k)}}{2\tau_{M,\pi}^{(l)}}-\sqrt{v^{(k)}\log(\frac{1}{\delta})}\Big\}\leq \delta
\end{align*}
On the other hand we have
\begin{align*}
\mathbb{P}\Big\{v^{(k)}(l)\geq \frac{v^{(k)}}{2\tau_{M,\pi}^{(l)}}+\sqrt{v^{(k)}\log(\frac{1}{\delta})}\Big\}\leq \delta
\end{align*}
\end{lemma}

Let $C_T=\displaystyle\max_{l\in\mathcal{A},k\in\{k'|t_{k'}\leq N\}}\frac{C_T^{(k)} \sigma_{\max}(V_3^{(k)(l')})}{\lambda^{(k)(l')}}$ then
\begin{align*}
\Delta_T\leq X^{\frac{3}{2}}\sqrt{YR}r_{max}\sqrt{\log{\frac{1}{\delta}}}{(D+1)}\sum_k\sum_{l=1}^A\sqrt{v^{(k)}(l)}  \sqrt{\underbrace{\frac{v^{(k)}(l)}{\displaystyle\min_{l'\in\mathcal{A}}v^{(k)}(l')}}_{(a')}}
\end{align*}
From Lemma~\ref{lem:visits} we have that
\begin{align*}
&(a')=\frac{v^{(k)}(l)}{\displaystyle\min_{l'\in\mathcal{A}}v^{(k)}(l')}\leq \frac{\frac{v^{(k)}}{2\displaystyle\min_{l\in \mathcal{A}}{\tau_{M,\pi}^{(l)}}}+\sqrt{v^{(k)}\log(\frac{1}{\delta})}}{\frac{1}{2}\frac{v^{(k)}}{2\displaystyle\max_{l\in \mathcal{A}}{\tau_{M,\pi}^{(l)}}}-\sqrt{v^{(k)}\log(\frac{1}{\delta})}}\\
&\leq 2\Dr^{\pi}+8\Dr^{\pi}\max_{l\in \mathcal{A}}{\tau_{M,\pi}^{(l)}}\sqrt{\frac{\log(\frac{1}{\delta})}{v^{(k)}}}+4\max_{l\in \mathcal{A}}{\tau_{M,\pi}^{(l)}}\sqrt{\frac{\log{(\frac{1}{\delta})}}{v^{(k)}}}+16(\max_{l\in \mathcal{A}}{\tau_{M,\pi}^{(l)}})^2\frac{\log{(\frac{1}{\delta})}}{v^{(k)}},
\end{align*}
with probability $1-2\delta$. The first term dominates all the other terms and thus
\begin{align*}
(a')\leq 2\Dr^{\pi}\leq 2\Dr
\end{align*}
with probability at least $1-\delta$. Thus we finally obtain
\begin{align*}
\Delta_T = \sum_{k=\wt{K}+1}^K \sum_{l=1}^A v^{(k)}(l)\big(\B_T^{(k,l)}\big)r_{\max}{(D+1)}\leq X^{\frac{3}{2}}\sqrt{YR}r_{max}\sqrt{\log{\frac{1}{\delta}}}{(D+1)}\sqrt{2\Dr}\frac{\sqrt{2}}{\sqrt{2}-1}\sqrt{AN}
\end{align*}
with probability at least $1-2\wb{K}_\TN\delta$. Then with probability at least $1-\delta(8A+5\wb{K}_\TN)$ the regret is bounded by the final statement.
\end{proof}

%%%%%%%%%%%%%%%%%%%%%%%%%%%%%%%%%%%%%%%%%%%%%%%%%%%%%%%%%%%%%%%%
%%%%%%%%%%%%%%%%%%%%%%%%%%%%%%%%%%%%%%%%%%%%%%%%%%%%%%%%%%%%%%%%
%% DISCUSSION
%%%%%%%%%%%%%%%%%%%%%%%%%%%%%%%%%%%%%%%%%%%%%%%%%%%%%%%%%%%%%%%%
%%%%%%%%%%%%%%%%%%%%%%%%%%%%%%%%%%%%%%%%%%%%%%%%%%%%%%%%%%%%%%%%

% !TEX root = master.tex

%%%%%%%%%%%%%%%%%%%%%%%%%%%%%%%%%%%%%%%%%%%%%%%%%%%%%%%%%%%%%%%%
%%%%%%%%%%%%%%%%%%%%%%%%%%%%%%%%%%%%%%%%%%%%%%%%%%%%%%%%%%%%%%%%
%% EXPERIMENTS
%%%%%%%%%%%%%%%%%%%%%%%%%%%%%%%%%%%%%%%%%%%%%%%%%%%%%%%%%%%%%%%%
%%%%%%%%%%%%%%%%%%%%%%%%%%%%%%%%%%%%%%%%%%%%%%%%%%%%%%%%%%%%%%%%

\section{Experiments}\label{s:experiments}

Here, we illustrate the performance of our method on a simple synthetic  environment which follows a POMDP structure with $X=2$, $Y=4$, $A=2$, $R=4$, and $r_{max}=4$.  We find that spectral learning method converges quickly to the true model parameters, as seen in Fig.~[\ref{fig:learning}]. Estimation of the transition tensor $T$ takes longer compared to estimation of observation matrix $O$ and reward Tensor $R$. This is because the observation and reward matrices are first estimated through tensor decomposition, and the transition tensor is estimated subsequently through additional manipulations. Moreover, the transition tensor has more parameters since it is a tensor (involving observed, hidden and action states) while the observation and reward matrices involve fewer parameters. 

For planning, given the POMDP model parameters, we find the memoryless policy using a simple alternating minimization heuristic, which alternates between updates of the policy and the stationary distribution.  We find that in practice this converge to a good solution. The regret bounds  are shown in Fig.~[\ref{fig:performance}].  We compare against the following policies: (1) baseline random policies which simply selects random actions without looking at the observed data, (2) \ucrl-MDP \cite{auer2009near} which attempts to fit a MDP model to the observed data and runs the \ucrl policy, and (3) Q-Learning \cite{watkins1992q} which is a model-free method that updates policy based on the Q-function. We find that our method converges much faster than the competing methods. Moreover, it converges to a much better policy. Note that the MDP-based policies  \ucrl-MDP and Q-Learning perform very poorly, and are even worse than the baseline  are too far from $\smucrl$ policy. This is because the MDP policies try to fit data in high dimensional observed space, and therefore, have poor convergence rates. On the other hand, our spectral method efficiently finds the correct low dimensional hidden space quickly and therefore, is able to converge to an efficient policy.

\begin{figure}
%\subfloat[a][\small Learning by Spectral Method]
{
\label{fig:learning}\begin{minipage}{0.55\textwidth}
\begin{center}
\begin{psfrags}
\psfrag{episode}[][1]{Episodes}
\psfrag{average deviation from true parameter}[][1]{ Average Deviation from True Parameter}
\psfrag{Transition Tensor}[][1]{ ~~~  \footnotesize Transition Tensor}
\psfrag{Observation Matrix}[][1]{~~~~~~~\footnotesize Observation Matrix}
\psfrag{Reward Tensor}[][1]{\footnotesize ~~~~ Reward Tensor}
\psfrag{Random Policy}[][1]{Random Policy}
\includegraphics[width=12.5cm,trim={5cm 0 0 0},clip]{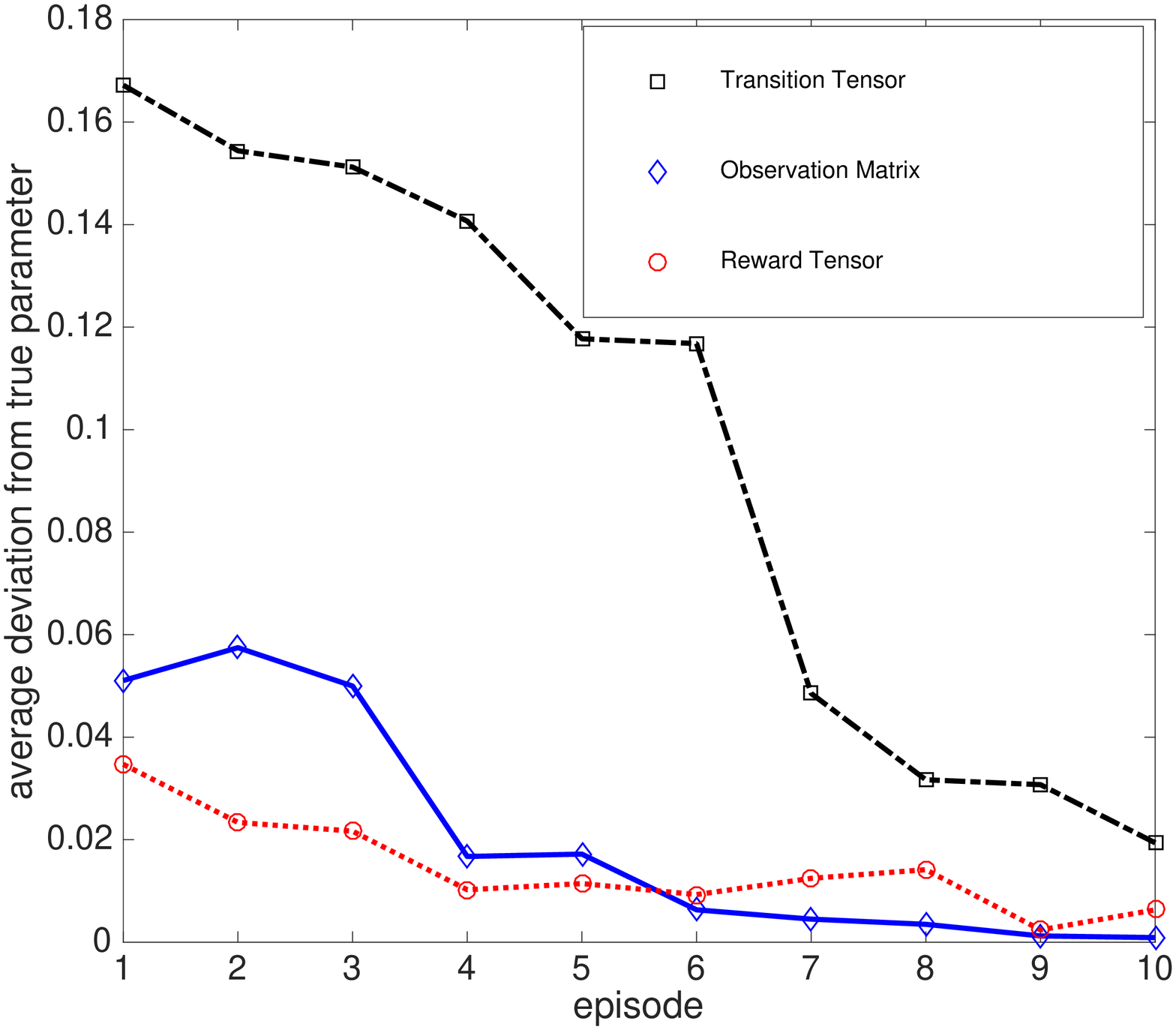}\\
\hspace{-2cm}\small $a)$Learning by Spectral Method
%\caption{\small Learning by Spectral Method}
\end{psfrags}
\end{center}
\end{minipage}
%\subfloat[b][\small Regret Performance]
\label{fig:performance}\begin{minipage}{0.55\textwidth}
\begin{psfrags}
\psfrag{Average Reward}[][1]{   Average Reward}
\psfrag{Number of Trials}[][1]{   Number of Trials}
\psfrag{SM-UCRL-POMDP}[][1]{ \small ~~~~~~  SM-UCRL-POMDP}
\psfrag{UCRL-MDP}[][1]{ \small ~~ UCRL-MDP}
\psfrag{Q-learning}[][1]{  \small ~~ Q-Learning}
\psfrag{Random Policy}[][1]{\small ~~~~  Random Policy}
\hspace{-1cm}
\includegraphics[width=9cm,trim={17cm 0 0cm 0},clip]{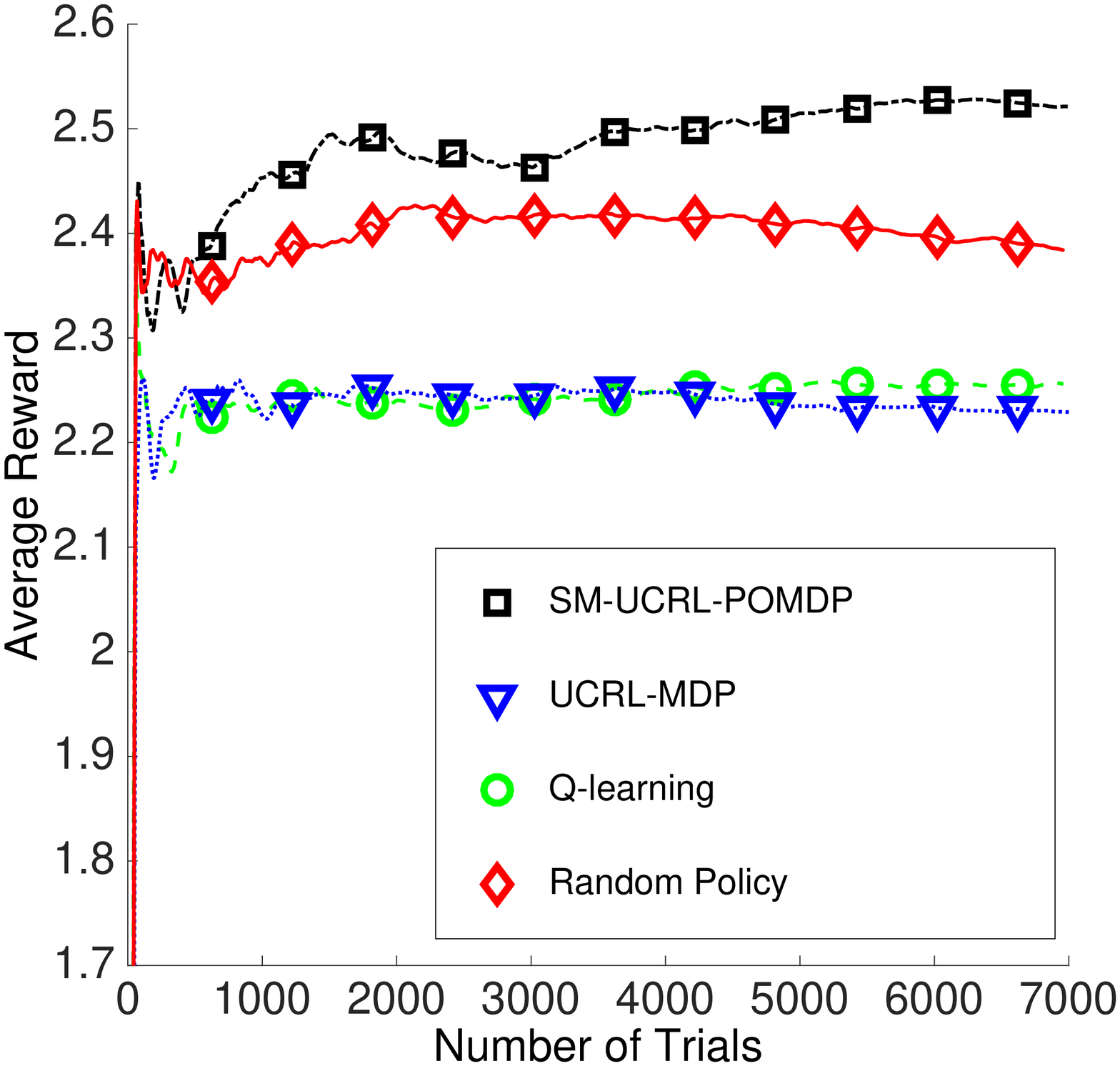}\\
~~~~~~~~~~~~\small            $ ~~~~~~~~~~~~~~~$ $~ b)$  Regret Performance
%\caption{\small Regret Performance}
\end{psfrags}
\end{minipage}} 
\caption{(a)Accuracy of estimated model parameter through tensor decomposition.  See  h Eqs.~\ref{eq:rew.bound},\ref{eq:obs.bound} and~\ref{eq:transition.bound}. (b)Comparison of $\smucrl$-POMDP is our method, \ucrl-MDP which attempts to fit a MDP model under \ucrl policy, $\epsilon-greedy$ Q-Learning, and a Random Policy.}
\end{figure}

%
%\begin{figure}
%\begin{center}
%\begin{psfrags}{tag}[]
%\psfrag{Average Reward}[][1]{\Huge   Average Reward}
%\psfrag{Number of Trials}[][1]{\Huge   Number of Trials}
%\psfrag{SM-UCRL-POMDP}[][1]{\Large   SM-UCRL-POMDP}
%\psfrag{UCRL-MDP}[][1]{\Large   UCRL-MDP}
%\psfrag{Q-learning}[][1]{\Large   Q-Learning}
%\psfrag{Random Policy}[][1]{\Large   Random Policy}
%\psfrag{r2}[][1]{$r_{t+1}$}
%\psfrag{a1}[][1]{$a_{t}$}
%\psfrag{a2}[][1]{$a_{t+1}$}
%\includegraphics[width=20cm,trim={15cm 0 0 0},clip]{figs/Performance.eps}
%\end{psfrags}
%\end{center}
%\end{figure}

%
%\begin{figure}
%\begin{center}
%\begin{psfrags}
%\psfrag{episode}[][1]{\Huge   Episodes}
%\psfrag{average deviation from true parameter}[][1]{\Huge   Average Deviation from True Parameter}
%\psfrag{Transition Tensor}[][1]{\Large                        ~~~  Transition Tensor}
%\psfrag{Observation Matrix}[][1]{\large                ~~~             Observation Matrix}
%\psfrag{Reward Tensor}[][1]{\large         Reward Tensor}
%\psfrag{Random Policy}[][1]{\large         Random Policy}
%\psfrag{r2}[][1]{$r_{t+1}$}
%\psfrag{a1}[][1]{$a_{t}$}
%\psfrag{a2}[][1]{$a_{t+1}$}
%\includegraphics[width=25cm,trim={0cm 0 0 0},clip]{figs/fig-learning1.eps}
%\end{psfrags}
%\end{center}
%\end{figure}

\newpage
\section{Concentration Bound}\label{app:HMMConst}
\noindent\paragraph{Concentration of functions of a HMM}  We now provide concentration bounds for any matrix valued function $\phi(\cdot)$ over samples drawn from a HMM. This extends the result on scalar functions by~\citet{kontorovich2014uniform}.

For any ergodic Markov chain, lets consider $\omega$ as its stationary distribution and $f_{1\rightarrow t}(x_t|x_1)$ as a distribution over states at time $t$ given initial state $x_1$. Lets define inverse mixing time $\rho_{mix}(t)$ as follows
\begin{align*}
\rho_{mix}(t)=\sup_{x_1}\left\|f_{1\rightarrow t}(\cdot|x_1)-\omega\right\|_{TV}
\end{align*}
~\citet{kontorovich2014uniform} show  that this measure can be bounded by%\aacomment{replace with Roi's paper as ref}
\begin{align*}
\rho_{mix}(t)\leq G\theta^{t-1},
\end{align*}
where $1\leq G<\infty$ is $geometric$ $ergodicity$ and $0\leq \theta<1$ is contraction coefficient of Markov chain.

As before, let  $\bfy^n:=[\vec{y}_1, \ldots, \vec{y}_n]\in \calY^n$ denote the sequence of observations from HMM and let $x^n :=[x_1, \ldots, x_n]\in \calX^n$ denote the sequence of hidden states. We now consider matrix valued function $\Phi:\calY^{n}\rightarrow \mathbb{R}^{d_1\times d_2}$. It is said to be   $c$-Lipschitz  with respect to  spectral norm when
\begin{align*}
\sup_{\bfy^n, {\bfy'}^n \in \calY^n}\frac{\left\|\Phi(\bfy^n)-\Phi({\bfy'}^n)\right\|_2}{\|\bfy^n-{\bfy'}^n\|_H}\leq c
\end{align*}
where $\|\cdot\|_H$ is norm with respect to Hamming metric,   and $\bfy^n, {\bfy'}^n$ are any two sequences of  sample observations.
\begin{theorem}[HMM Concentration Bound]
Consider Hidden Markov Model with finite sequence of $n$ samples $\vec{y}_i$ as observations from finite observation set $\bfy^n$ and arbitrary initial state distribution. For any $c$-Lipschitz matrix valued function $\Phi(\cdot)$, we have
\begin{align*}
\left\|\Phi(\bfy^n)-\mathbb{E}[\Phi(\bfy^n)]\right\|_2\leq G\frac{1+\frac{1}{\sqrt{2}cn^{\frac{3}{2}}}}{1-\theta}\sqrt{8c^2n\log(\frac{(d_1+d_2)}{\delta})}
\end{align*}
with probability at least $1-\delta$, where G is  $geometric$ $ergodicity$ constant of corresponding Markov chain, and the $\mathbb{E}[\Phi(\bfy^n)]$ is expectation over samples of HMM when the initial distribution corresponds to the stationary distribution.
\label{thm:GeneralConcentrationBound}
\end{theorem}
\begin{proof}
In the Appendix. \ref{app:Proof_main}
\end{proof}
%Later, we extend this result  to the more complicated model of memory-less policy POMDP in order to derive our regret bounds. 
\begin{theorem}[POMDP Concentration Bound]
Consider Partially Observable Markov Decision Process with finite sequence of $n(l)$ samples $\vec{y}_i^{(l)}$ for all $i\in\lbrace1,2,\ldots,n(l)\rbrace~\forall l\in[A]$ as observations from finite observation sets $\bfy^{n(l)}$ and arbitrary initial state distribution. For any $c$-Lipschitz matrix valued function $\Phi^l(\cdot)$ function, we have
\begin{align*}
\left\|\Phi^l(\bfy^{n(l)})-\mathbb{E}[\Phi^l(\bfy^{n(l)})]\right\|_2\leq G\frac{1+\frac{1}{\sqrt{2}cn^{\frac{3}{2}}}}{1-\theta}\sqrt{8c^2n\log(\frac{(d_1+d_2)}{\delta})}
\end{align*}
with probability at least $1-\delta$, where G is  $geometric$ $ergodicity$ constant of corresponding Markov chain, and the $\mathbb{E}[\Phi(\bfy^{n(l)})]$ is expectation over samples of POMDP with middle action $l$ when the initial distribution corresponds to the stationary distribution.
\label{thm:GeneralConcentrationBoundPOMDP}
\end{theorem}
\begin{proof}
In the Appendix. \ref{app:Proof_main}
\end{proof}
%general bound for c-lipzchit HMM function

\section{Proof of Thms.~ \ref{thm:GeneralConcentrationBound} and \ref{thm:GeneralConcentrationBoundPOMDP}}\label{app:Proof_main}

The proof is based on the results in \cite{tropp2012user}, \cite{kontorovich2008concentration}, and \cite{kontorovich2014uniform} with minor modifications and applying the following inequality
\begin{align*}
\frac{G}{1-\theta}\sqrt{8c^2\frac{\log(\frac{(d_1+d_2)}{\delta})}{n}}+\frac{2G}{n(1-\theta)}\leq G\frac{1+\frac{1}{\sqrt{2}cn^{\frac{3}{2}}}}{1-\theta}\sqrt{8c^2n\log(\frac{(d_1+d_2)}{\delta})}
\end{align*}
here we just bring the sketch of the proof. In Thm.~\ref{thm:MatrixAzuma}, we give the upper confidence bound over $\left\|\Phi-\mathbb{E}[\Phi]\right\|_2$ where the expectation is with same initial distribution as it used for $\Phi$. The next step is finding upper bound for difference between $\mathbb{E}[\Phi]$ with arbitrary initial distribution and $\mathbb{E}_{stat}[\Phi]$ with initial distribution equal to stationary distribution. It is clear through \cite{kontorovich2014uniform} that this quantity is upper bounded by $\sum_iG\theta^{-(i-1)}$ which is upper bounded by $\frac{G}{(1-\theta)}$. % proof of general bound for c-lipzchit HMM function
% !TEX root = master.tex

\section{Concentration Bound}\label{app:ConBound}

\begin{theorem}[Matrix Azuma]
Consider Hidden Markov Model with finite sequence of $n$ samples $S_i$ as observations given arbitrary initial states distribution and $c-Lipschitz$ matrix valued function $\Phi:S_{1}^{n}\rightarrow \mathbb{R}^{d_1\times d_2}$ in dimension $d_1$ by $d_2$, then
\begin{align*}
\left\|\Phi-\mathbb{E}[\Phi]\right\|_2\leq \frac{1}{1-\theta}\sqrt{8c^2n\log(\frac{(d_1+d_2)}{\delta})}
\end{align*}
with probability at least $1-\delta$. The $\mathbb{E}[\Phi]$ is given same initial distribution of samples.
\label{thm:MatrixAzuma}
\end{theorem}
\begin{proof}
The Thm.~ 7.1 \cite{tropp2012user} presents the upper confidence bound over the summation of matrix random variables. Consider a finite sequence of $d$ by $d'$ matrix $\Psi_i$, then for variance parameter $\sigma^2$ which is upper bound for $\sum_i\left[\Psi_i-\mathbb{E}_{i-1}[\Psi_i]\right],~~\forall i$
\begin{align*}
\left\|\sum_i\left[\Psi_i-\mathbb{E}_{i-1}[\Psi_i]\right]\right\|_2\leq \sqrt{8\sigma^2 \log(\frac{d+d'}{\delta})}
\end{align*}
with probability at least $1-\delta$.\\
For the function $\Phi$, lets define the martingale difference of function $\Phi$ as the input random variable with arbitrary initial distribution over states.
\begin{align*}
MD_i(\Phi;S_1^i)=\mathbb{E}[\Phi|S_1^i]-\mathbb{E}[\Phi|S_1^{i-1}]
\end{align*}
where $S_i^j$ is sub set of samples from $i'th$ position in sequence to $j'th$ one.
then the summation over these set of random variable gives $\mathbb{E}[\Phi|S_1^n]-\mathbb{E}[\Phi]$ which is  $\Phi(S_1^n)-\mathbb{E}[\Phi]$ and $\mathbb{E}[\Phi]$ is expectation with same initial state distribution . The remaining part is finding $\sigma$ which is upper bound for $\left\|\sum_iMD_i(\Phi;S_1^i)\right\|_2$ for all possible sequence. Lets define  $MD_i(\Phi)=\max_{S_1^i}MD_i(\Phi;S_1^i)$ and through \cite{kontorovich2008concentration} it is easy to show that $\left\|MD_i(\Phi)\right\|_2$ is $c-Lipchitz$ function and it is upper bounded by $cH_{i,n}$. In \cite{kontorovich2014uniform} it is shown that $H_{i,n}$ is upper bounded by $G\theta(n-i)$.
\end{proof}
For the case when $\Phi$ is symmetric matrix, $d_1+d_2$ can be reduced to just $d$ and constant $8$ can be reduced to $2$. \\
The result in Thm.~ \ref{thm:MatrixAzuma} can be extended to the situation when distribution of next state depends on current state and current observation and even more complicated models like memory-less policy POMDP.
\begin{theorem}[Concentration Bound]
Consider finite sequence of multiple views are drawn from memory less policy POMDP with common middle action and their corresponding covariance matrix $\vec{v}_{\nu,t}^{(l)}\otimes \vec{v}_{\nu',t}^{(l)}$ for $\nu,\nu'\in\{1,2,3\}$ and $\nu\neq \nu'$. For simplicity, lets just consider one set of $\nu,\nu'$, one specific middle action, and $n$ samples are drawn. Define random variable $\Phi_i:=\frac{1}{N(l)}\left[\mathbb{E}\left[\sum_t\vec{v}_{\nu,t}^{(l)}\otimes \vec{v}_{\nu',t}^{(l)}\Big|S_1^{i}\right]-\mathbb{E}\left[\sum_t\vec{v}_{\nu,t}^{(l)}\otimes \vec{v}_{\nu',t}^{(l)}\Big|S_1^{i-1}\right]\right]$ with dimensions $d_\nu\times d_{\nu'}$ where $d_\nu$ and $d_{\nu'}$ for $\nu,\nu'\in \{1,2,3\}$ are the dimension along the $\nu$ and ${\nu}'$ views.
\begin{align*}
\left\|\sum_i\Phi_i\right\|_2=\left\|\frac{1}{N(l)}\sum_t\Big[\vec{v}_{\nu,t}^{(l)}\otimes \vec{v}_{\nu',t}^{(l)}\Big]-\frac{1}{N(l)}\mathbb{E}[\sum_t\vec{v}_{\nu,t}^{(l)}\otimes \vec{v}_{\nu',t}^{(l)}]\right\|_2\leq \frac{G(\pi)}{1-\theta(\pi)}\sqrt{8\frac{\log\frac{(d_{\nu}+d_{\nu'})}{\delta}}{N(l)}}
\end{align*}
with probability at least $1-\delta$.\\
For tensor case, $\frac{1}{N(l)}\left[\mathbb{E}\left[\sum_t\vec{v}_{\nu,t}^{(l)}\otimes \vec{v}_{\nu',t}^{(l)}\otimes \vec{v}_{\nu'',t}^{(l)}\Big|S_1^{i}\right]-\mathbb{E}\left[\sum_t\vec{v}_{\nu,t}^{(l)}\otimes \vec{v}_{\nu',t}^{(l)}\otimes \vec{v}_{\nu'',t}^{(l)}\Big|S_1^{i-1}\right]\right]$ where $[\nu,\nu',\nu'']$ can be any permutation of set $\{1,2,3\}$.
\begin{align*}
\left\|\frac{1}{N(l)}\sum_t\Big[\vec{v}_{\nu,t}^{(l)}\otimes \vec{v}_{\nu',t}^{(l)}\otimes \vec{v}_{\nu'',t}^{(l)}\Big]-\frac{1}{N(l)}\mathbb{E}[\sum_t\vec{v}_{\nu,t}^{(l)}\otimes \vec{v}_{\nu',t}^{(l)}\otimes \vec{v}_{\nu'',t}^{(l)}]\right\|_2\leq\frac{G(\pi)}{1-\theta(\pi)} \sqrt{8\frac{\log\frac{(d_\nu d_{\nu'}+d_{\nu''})}{\delta}}{N(l)}}
\end{align*}
with probability at least $1-\delta$.
\label{thm:ConcentrationBound}
\end{theorem}
\begin{proof}
%Define random variable $S_t=\vec{v}_{\nu,t}^{(l)}\otimes \vec{v}_{\nu',t}^{(l)}-\mathbb{E}_{t-1}[\vec{v}_{\nu,t}^{(l)}\otimes \vec{v}_{\nu',t}^{(l)}]$

For simplicity lets just proof the first claim in Thm.~ \ref{thm:ConcentrationBound} and the proof for the second claim would be followed by same procedure. To proof the Thm.~\ref{thm:ConcentrationBound} it is needed to bring together the results from \cite{tropp2012user}, \cite{kontorovich2008concentration}, Thms.~ \ref{thm:GeneralConcentrationBound} , and \ref{thm:MatrixAzuma} and then modify them. The Thm.~ 7.1 in \cite{tropp2012user} presents following upper confidence bounds
\begin{align*}
\left\|\frac{1}{N(l)}\sum_t\Big[\vec{v}_{\nu,t}^{(l)}\otimes \vec{v}_{\nu',t}^{(l)}\Big]-\frac{1}{N(l)}\mathbb{E}[\sum_t\vec{v}_{\nu,t}^{(l)}\otimes \vec{v}_{\nu',t}^{(l)}]\right\|_2\leq \sqrt{8(\widetilde{\sigma}_{Pairs}^{\nu,\nu'})^2\log\frac{(d_{\nu}+d_{\nu'})}{\delta}}
\end{align*}
with probability at least $1-\delta$. And
\begin{align*}
\left\|\frac{1}{N(l)}\sum_t\Big[\vec{v}_{\nu,t}^{(l)}\otimes \vec{v}_{\nu',t}^{(l)}\otimes \vec{v}_{\nu'',t}^{(l)}\Big]-\frac{1}{N(l)}\mathbb{E}[\sum_t\vec{v}_{\nu,t}^{(l)}\otimes \vec{v}_{\nu',t}^{(l)}\otimes \vec{v}_{\nu'',t}^{(l)}]\right\|_2\leq \sqrt{8(\widetilde{\sigma}_{Triples}^{\nu,\nu',\nu''})^2\log\frac{(d_\nu d_{\nu'}+d_{\nu''})}{\delta}}
\end{align*}
with probability at least $1-\delta$. It is needed to show that $(\widetilde{\sigma}_{Pairs}^{i,i'})^2\leq \frac{G(\pi)^2}{n(1-\theta(\pi))^2}$ and $(\widetilde{\sigma}_{Triples}^{i,i',i''})^2\leq \frac{G(\pi)^2}{n(1-\theta(\pi))^2}$.\\
\end{proof}

For the function $\Phi:S_1^n\rightarrow R^{d_1\times d_2}$, where $S_1^n$ is a collection of all possible $\{S_1,S_2, \ldots, S_n\}$ with length $n$. Its martingale difference is defined as follows
\begin{align*}
MD_i(\Phi;S_1^i)=\mathbb{E}[\Phi|S_1^i]-\mathbb{E}[\Phi|S_1^{i-1}]
\end{align*}
and then $MD_i(\Phi)=\max_{S_1^i}MD_i(\Phi;S_1^i)$.\\
The upper bound over $\widetilde{\sigma}_{Pairs}^{\nu,\nu'}$ is as follows
\begin{align*}
(\widetilde{\sigma}_{Pairs}^{\nu,\nu'})^2\leq \left\|\sum_{t=1}^nU_t^2\right\|_2
\end{align*}
Where $U_t$ is a fixed sequence of matrices which follows $MD_t\preceq U_t$ for all possible $MD_t$ and $\forall t$. This bound over Triple tensor can be derived after matricizing the martingale difference.
Next step is to upper bound the $\left\|\sum_tU_t^2\right\|_2$. \\
 Lets define new set of variables; given each action (middle action) $a_i=l$ there are the following set of variables; $\VV^{i|l}$ is collection of \\
$\vec{y}_{p(i,l)-1}, a_{p(i,l)-1}, r_{p(i,l)-1},\vec{y}_{p(i,l)}, r_{p(i,l)}, \vec{y}_{p(i,l)+1}$, \\
where $\VV^{i|l}$ is $i'th$ triple with middle action equal to $l$ and $p(i,l)$ is its corresponding position in the original sequence. Lets define variable $\HH^{i|l}$, which is consequence of four hidden states $x_{p(i,l)-1},x_{p(i,l)},x_{p(i,l)+1}, x_{p(i,l)+2}$. The variables $\VV^{j|l}_i$ and $\HH^{j|l}_i$, for $i\leq j$, are corresponding to set of consecutive $i\rightarrow j$ triple views and quadruple hidden states. Note that this is the time to define mixing coefficients.
\begin{align*}
\eta^{(l)}_{i,j}(\vv_1^{i-1|l},\varrho,\varrho'):=\left\|\mathbb{P}(\VV^{N(l)|l}_j|\VV^{i|l}_1=\vv_1^{i-1|l},\varrho,l)-\mathbb{P}(\VV^{N(l)|l}_j|\VV^{i-1|l}_1=\vv_1^{i-1|l},\varrho',l)\right\|_{TV}
\end{align*}
where $TV$ is total variation distance between distributions and
\begin{align*}
\bar{\eta}^{(l)}_{i,j}:=\sup_{\vv_1^{i-1|l},\varrho,\varrho'}\eta^{(l)}_{ij}(\vv_1^{i-1|l},\varrho,\varrho')
\end{align*}
where $\mathbb{P}(\VV^{i|l}_1=\vv_1^{i-1|l},\varrho,l)$ and $\mathbb{P}(\VV^{i|l}_1=\vv_1^{i-1|l},\varrho'.l)$ are nonzero for all possible input variables. Then for $\Delta_{N(l)}$ \\ \\
$(\Delta_{N(l)})_{i,j} =
\left\{
	\begin{array}{ll}
		1  & \mbox{if } i=j \\
		 \bar{\eta}^{(l)}_{i,j} & \mbox{if } i < j\\
		 0  & \mbox{ } otherwise.
	\end{array}
\right.$\\ \\
and $H_{n(l),i}=1+\bar{\eta}^{(l)}_{i,i+1}+\ldots+\bar{\eta}^{(l)}_{i,n}$

%which is upper bounded by $cH_{n,i}$. It means that $\sigma_{Triples}^2\leq \left\|\sum_i^nMD_i(\Phi))\right\|_2\leq nc^2\frac{1}{(1-\theta)^2}\leq \frac{1}{n(1-\theta)^2} $

\paragraph{Martingale Difference}.
To upper bound for $\widetilde{\sigma}_{Pairs}^{\nu,\nu'}$, it is enough to upper bound $\left\|\sum_{t=1}^nU_t^2\right\|_2$ or directly upper bound $\left\|\sum_{t=1}^nMD_t^2\right\|_2$ for all possible sequence of samples. The result in \cite{kontorovich2008concentration} shows that this is upper bounded by $\sum_{i=1}^{n(l)}\left\|MD_i(\Phi))\right\|^2_2$ and each $\left\|MD_i(\Phi))\right\|_2\leq cH_{n,i}$ when the $\left\|\Phi\right\|_2$ is $c-Lipschitz$. \\

In addition, it is obvious that for the class of moment functions with elements in $[0,1]$ the $c$ is upper bounded by $\frac{1}{N(l)}$ for the purpose of this paper. The remaining shows the upper bound over $H_{n,i}$ and then $\sum_{i=1}^{n(l)}H_{n,i}^2$.

\begin{lemma}
The function $H_{n,i}$ is upper bounded by $\frac{G(\pi)}{1-\theta(\pi)}$ and then $\sum_{i=1}^{n}(cH_{n,i})^2\leq nc^2\frac{G^2(\pi)}{(1-\theta(\pi))^2}\leq \frac{G^2(\pi)}{n(1-\theta(\pi))^2} $
\end{lemma}
\begin{proof}
As it mentioned, $H_{n(l),i}=1+\bar{\eta}^{(l)}_{i,i+1}+\ldots+\bar{\eta}^{(l)}_{i,n}$, and it is needed to find the upper bound over $1+\bar{\eta}^{(l)}_{i,i+1}+\ldots+\bar{\eta}^{(l)}_{i,n}$
\begin{align*}
&\eta^{(l)}_{ij}(\vv_1^{i-1|l},\varrho,\varrho')\\
&=\frac{1}{2}\sum_{\vv^{N(l)|l}_j}|\mathbb{P}(\VV^{N(l)|l}_j=\vv^{N(l)|l}_j|\VV^{i|l}_1=\vv_1^{i-1|l},\varrho,l)-\mathbb{P}(\VV^{N(l)|l}_j=\vv^{N(l)|l}_j|\VV^{i-1|l}_1=\vv_1^{i-1|l},\varrho',l)|
\end{align*}
For the first part
\begin{align*}
&\mathbb{P}(\VV^{N(l)|l}_j=\vv^{N(l)|l}_j|\VV^{i|l}_1=\vv_1^{i-1|l},\varrho,l)\\
&=\sum_{\hhh_1^{i|l},\hhh_j^{N(l)|l}}\mathbb{P}(\VV^{N(l)|l}_j=\vv^{N(l)|l}_j,\HH_1^{i|l}=\hhh_1^{i|l},\HH_j^{N(l)|l}=\hhh_j^{N(l)|l}|\VV^{i|l}_1=\vv_1^{i-1|l},\varrho,l)
\end{align*}
Lets assume, for simplicity, that the hidden states on $\hhh^{i|l}$ do not have overlap with states on $\hhh^{i-1|l}$ and $\hhh^{i+1|l}$.
\begin{align*}
\mathbb{P}(\VV^{N(l)|l}_j&=\vv^{N(l)|l}_j|\VV^{i|l}_1=\vv_1^{i-1|l},\varrho,l)=\\
&\sum_{\hhh_1^{i|l},\hhh_j^{N(l)|l}}\mathbb{P}(\VV^{N(l)|l}_j=\vv^{N(l)|l}_j,\VV^{i-1|l}_1=\vv_1^{i-1|l},\varrho,l|\HH_1^{i|l}=\hhh_1^{i|l},\HH_j^{N(l)|l}=\hhh_j^{N(l)|l})\\
&\mathbb{P}(\HH_1^{i|l}=\hhh_1^{i|l},\HH_j^{N(l)|l}=\hhh_j^{N(l)|l})\frac{1}{\mathbb{P}(\VV^{i|l}_1=\vv_1^{i-1|l},\varrho,l)}\\
&\hspace{-0cm}=\sum_{\hhh_1^{i|l},\hhh_j^{N(l)|l}}\mathbb{P}(\VV^{N(l)|l}_j=\vv^{N(l)|l}_j,l|\HH_j^{N(l)|l}=\hhh_j^{N(l)|l})\mathbb{P}(\VV^{i|l}_1=\vv_1^{i-1|l},\varrho,l|\HH_1^{i|l}=\hhh_1^{i|l},\HH)\\
&\mathbb{P}(\HH_1^{i|l}=\hhh_1^{i|l},\HH_j^{N(l)|l}=\hhh_j^{N(l)|l})\frac{1}{\mathbb{P}(\VV^{i|l}_1=\vv_1^{i-1|l},\varrho,l)}
\end{align*}
with this representation
\begin{align*}
\eta^{(l)}_{ij}(\vv_1^{i-1|l},\varrho,\varrho')&=\frac{1}{2}\sum_{\vv_j^{N(l)|l}}|\sum_{\hhh_1^{i|l},\hhh_j^{N(l)|l}}\mathbb{P}(\VV^{N(l)|l}_j=\vv^{N(l)|l}_j|\HH_j^{N(l)|l}=\hhh_j^{N(l),l|l})\mathbb{P}(\HH_1^{i|l}=\hhh_1^{i|l},\HH_j^{N(l)|l}=\hhh_j^{N(l)|l})\\
&\mathbb{P}(\VV^{i|l}_1=\vv_1^{i-1|l},\varrho',l|\HH_1^{i|l}=\hhh_1^{i|l})\big(\frac{\mathbb{P}(\varrho,l|\HH^{i|l}=\hhh^{i|l})}{\mathbb{P}(\VV^{i|l}_1=\vv_1^{i-1|l},\varrho,l)}-\frac{\mathbb{P}(\varrho',l|\HH^{i|l}=\hhh^{i|l})}{\mathbb{P}(\VV^{i|l}_1=\vv_1^{i-1|l},\varrho',l)}\big)|
\end{align*}
\begin{align*}
\eta^{(l)}_{ij}(\vv_1^{i-1|l},\varrho,\varrho')&\leq \frac{1}{2}\sum_{\hhh^{j|l}}|\sum_{\hhh_1^{i|l}}\mathbb{P}(\VV^{N(l)|l}_j=\vv^{N(l)|l}_j,l|\HH_j^{N(l)|l}=\hhh_j^{N(l)|l})\mathbb{P}(\HH_1^{i|l}=\hhh_1^{i|l},\HH_j^{N(l)|l}=\hhh_j^{N(l)|l})\\
&\mathbb{P}(\VV^{i|l}_1=\vv_1^{i-1|l},\varrho',l|\HH_1^{i|l}=\hhh_1^{i|l})\big(\frac{\mathbb{P}(\varrho,l|\HH^{i|l}=\hhh^{i|l})}{\mathbb{P}(\VV^{i|l}_1=\vv_1^{i-1|l},\varrho,l)}-\frac{\mathbb{P}(\varrho',l|\HH^{i|l}=\hhh^{i|l})}{\mathbb{P}(\VV^{i|l}_1=\vv_1^{i-1|l},\varrho',l)}\big)|
\end{align*}
\begin{align*}
\eta^{(l)}_{ij}(\vv_1^{i-1|l},\varrho,\varrho')&\leq \frac{1}{2}\sum_{x^{p(j)-1|l}}|\sum_{\hhh_1^{i|l}}\mathbb{P}(\HH_1^{i|l}=\hhh_1^{i|l})\mathbb{P}(x^{p(j)-1|l}|x^{p(i,l)+2|l})\mathbb{P}(\VV^{i-1|l}_1=\vv_1^{i-1|l},l|\HH_1^{i|l}=\hhh_1^{i|l})q(\hhh^{i|l})|
\end{align*}
where
\begin{align*}
q(v,l):= \frac{\mathbb{P}(\varrho,l|\HH^{i|l}=v)}{\mathbb{P}(\VV^{i|l}_1=\vv_1^{i-1|l},\varrho,l)}-\frac{\mathbb{P}(\varrho',l|\HH^{i|l}=v)}{\mathbb{P}(\VV^{i|l}_1=\vv_1^{i-1|l},\varrho',l)}
\end{align*}
then
\begin{align*}
\eta^{(l)}_{ij}(\vv_1^{i-1|l},\varrho,\varrho')&\leq \frac{1}{2}\sum_{x^{p(j)-1|l}}|\sum_{\hhh^{i|l}}\mathbb{P}(x^{p(j)-1|l}|x^{p(i,l)+2|l})h(\hhh^{i|l},l)|\\
&\leq \frac{1}{2}\sum_{x^{p(j)-1|l}}|\sum_{x^{p(i,l)+1|l}}\mathbb{P}(x^{p(j)-1|l}|x^{p(i,l)+2|l})\sum_{x^{p(i,l)+1|l},x^{p(i,l)|l},x^{p(i,l)-1|l}}h(\hhh^{i|l},l)|\\
&\leq \frac{1}{2}\sum_{x^{p(j)-1|l}}|\sum_{x^{p(i,l)+2|l}}\mathbb{P}(x^{p(j)-1|l}|x^{p(i,l)+2|l})\bar{h}(x^{p(i,l)+2|l},l)|
\end{align*}
where
\begin{align*}
h(v,l):=\sum_{\hhh_1^{i-1|l}}\mathbb{P}(\HH_1^{i|l}=\hhh_1^{i|l})\mathbb{P}(\VV^{i-1|l}_1=\vv_1^{i-1|l},l|\HH_1^{i|l}=\hhh_1^{i|l})q({v},l)
\end{align*}
\begin{align*}
\bar{h}(x^{p(i,l)+2|l},l):=\sum_{x^{p(i,l)+1|l},x^{p(i,l)|l},x^{p(i,l)-1|l}}h(\hhh^{i|l},l)
\end{align*}
as a consequence
\begin{align*}
\eta^{(l)}_{ij}(\vv_1^{i-1|l},\varrho,\varrho')\leq \left\|\frac{1}{2}\bar{h}^\top P^{i,j}\right\|_1
\end{align*}
where $P^{i,j}=\mathbb{P}(x^{p(j)-1|l}|x^{p(i,l)+2|l})$. Through Lemma A.2 in \cite{kontorovich2008concentration} and \cite{kontorovich2014uniform}, when $\sum_x{\bar{h}(x,l)}=0$ and $\frac{1}{2}\left\|\bar{h}\right\|_1\leq 1$, it is clear that $\eta^{(l)}_{ij}(\vv_1^{i-1|l},\varrho,\varrho')$, and also $\bar{\eta}^{(l)}_{ij}$ can be bounded by  $\frac{1}{2}\left\|\bar{h}^\top\right\|_1 G(\pi)\theta(\pi)^{p(j,l)-p(i,l)-4}$. To verify $\sum_x{\bar{h}}=0$ and $\frac{1}{2}\left\|\bar{h}\right\|_1\leq 1$
\begin{align*}
\sum_{x^{p(i,l)+2|l}}&\bar{h}(x^{p(i,l)+2|l},l)=\sum_{x^{p(i,l)+2|l}}\sum_{x^{p(i,l)+1|l},x^{p(i,l)|l},x^{p(i,l)-1|l}}h(\hhh^{i|l},l)=\sum_{\hhh^{i|l}}h(\hhh^{i|l},l)\\
&=\sum_{\hhh_1^{i|l}}\mathbb{P}(\HH_1^{i|l}=\hhh_1^{i|l})\mathbb{P}(\VV^{i-1|l}_1=\vv_1^{i-1|l},l|\HH_1^{i|l}=\hhh_1^{i|l})q({\hhh^{i|l}},l)\\
&=\sum_{\hhh_1^{i|l}}\mathbb{P}(\HH_1^{i|l}=\hhh_1^{i|l})\mathbb{P}(\VV^{i|l}_1=\vv_1^{i-1|l},l|\HH_1^{i|l}=\hhh_1^{i|l}) \Bigg(\frac{\mathbb{P}(\varrho,l|\HH^{i|l}={\hhh^{i|l}})}{\mathbb{P}(\VV^{i|l}_1=\vv_1^{i-1|l},\varrho,l)}-\frac{\mathbb{P}(\varrho',l|\HH^{i|l}={\hhh^{i|l}})}{\mathbb{P}(\VV^{i-1|l}_1=\vv_1^{i-1|l},\varrho',l)}\Bigg)
\end{align*}
For the first part of parenthesis
\begin{align*}
\sum_{\hhh_1^{i|l}}\mathbb{P}(\HH_1^{i|l}=\hhh_1^{i|l})\mathbb{P}(\VV^{i-1|l}_1=\vv_1^{i-1|l},l|\HH_1^{i|l}=\hhh_1^{i|l}) \Bigg(\frac{\mathbb{P}(\varrho,l|\HH^{i|l}={\hhh^{i|l}})}{\mathbb{P}(\VV^{i|l}_1=\vv_1^{i-1|l},\varrho,l)}\Bigg)=1
\end{align*}
and same for the second one. This shows the conditions for the Lemma A.2 in \cite{kontorovich2008concentration}, $\sum_x{\bar{h}(x,l)}=0$ and $\frac{1}{2}\left\|\bar{h}\right\|_1\leq 1$, are met. \\
The presented proof is for the case of non-overlapped with states of $\HH^{i|l}$, for the other cases, the overlapped situation, the proof is pretty much similar to the non-overlapped case.
%The remaining part is bounding the parameter $\theta$ based on model parameters.
%
%\begin{align*}
%\theta=\sup_{x,x'}\left\|\ f_{T,\pi}(\cdot|x) -f_{T,\pi}(\cdot|x'&) \right\|_{TV}=\sup_{x,x'}\left\|T(x,I,\Pi Ox) -T(x',I,\Pi Ox')\right\|_{TV}\\
%&\leq  \sup_{i,j,i',j'\in [X]}\left\|T(\vec{e}_i,I,\vec{e_j}) -T(\vec{e}_{i'},I,\vec{e_{j'}})\right\|_{TV}\leq 1
%\end{align*}
%where $\Pi$ is $A\times Y$ policy matrix and $x$ and $x'$ are basis vectors.
 Now, it is the time to upper bound $\left\|\Delta_{N(l)}\right\|_\infty$.
\begin{align*}
H_{n,i}=1+\sum_{j=i}^{N(l)}\bar{\eta}^{(l)}_{ij}\leq 1+\max_i\sum_{j>i}\bar{\eta}^{(l)}_{ij}&\leq G(\pi)\sum_{i=0}^{n(l)}\theta(\pi)^{p(i,l)} \leq G(\pi)\sum_{i=0}^{n(l)}\theta(\pi)^{i}\\
&= G(\pi)\frac{1-\theta(\pi)^{n(l)}}{1-\theta(\pi)}\leq\frac{G(\pi)}{1-\theta(\pi)}
\end{align*}
\end{proof}
Define $\mathbb{E}_{stat}$ as the expectation with initial distribution equals to stationary distribution. Generally, in tensor decomposition, we are interested in
\begin{align*}
&\left\|\frac{1}{N(l)}\sum_t\Big[\vec{v}_{\nu,i}^{(l)}\otimes \vec{v}_{\nu',i}^{(l)}\Big]-\frac{1}{N(l)}\mathbb{E}_{stat}[\sum_i\vec{v}_{\nu,t}^{(l)}\otimes \vec{v}_{\nu',t}^{(l)}]\right\|_2\\
&\left\|\frac{1}{N(l)}\sum_t\Big[\vec{v}_{\nu,i}^{(l)}\otimes \vec{v}_{\nu',i}^{(l)}\otimes \vec{v}_{\nu'',i}^{(l)}\Big]-\frac{1}{N(l)}\mathbb{E}_{stat}[\sum_i\vec{v}_{\nu,i}^{(l)}\otimes \vec{v}_{\nu',i}^{(l)}\otimes \vec{v}_{\nu'',i}^{(l)}]\right\|_2
\end{align*}
instead of
\begin{align*}
&\left\|\frac{1}{N(l)}\sum_t\Big[\vec{v}_{\nu,i}^{(l)}\otimes \vec{v}_{\nu',i}^{(l)}\Big]-\frac{1}{N(l)}\mathbb{E}[\sum_i\vec{v}_{\nu,i}^{(l)}\otimes \vec{v}_{\nu',i}^{(l)}]\right\|_2\\
&\left\|\frac{1}{N(l)}\sum_t\Big[\vec{v}_{\nu,i}^{(l)}\otimes \vec{v}_{\nu',i}^{(l)}\otimes \vec{v}_{\nu'',i}^{(l)}\Big]-\frac{1}{N(l)}\mathbb{E}[\sum_i\vec{v}_{\nu,i}^{(l)}\otimes \vec{v}_{\nu',t}^{(l)}\otimes \vec{v}_{\nu'',i}^{(l)}]\right\|_2
\end{align*}
which are derived through Thm.~ \ref{thm:ConcentrationBound}. To come up with the upper confidence bound over the above mentioned interesting deviation, it is needed to derive the upper bound for deviation over expectation with arbitrary initial state distribution and expectation with stationary distribution over initial states, for simplicity, lets just derive the bound for second order moment.
\begin{align*}
\left\|\frac{1}{N(l)}\mathbb{E}_{n}[\sum_i\vec{v}_{\nu,i}^{(l)}\otimes \vec{v}_{\nu',i}^{(l)}]-\frac{1}{N(l)}\mathbb{E}_{stat}[\sum_i\vec{v}_{\nu,i}^{(l)}\otimes \vec{v}_{\nu',i}^{(l)}]\right\|_2
\end{align*}
As the bound $\epsilon_i$ over deviation from stationary distribution of Markov chain follows $\epsilon=G(\pi)\theta(\pi)^{-i}$. It results in
\begin{align*}
\left\|\frac{1}{N(l)}\mathbb{E}[\sum_i\vec{v}_{\nu,i}^{(l)}\otimes \vec{v}_{\nu',i}^{(l)}]-\frac{1}{N(l)}\mathbb{E}_{stat}[\sum_i\vec{v}_{\nu,i}^{(l)}\otimes \vec{v}_{\nu',i}^{(l)}]\right\|_2\leq 2\frac{G(\pi)}{N(l)(1-\theta(\pi))}
\end{align*}
which is negligible compared to $\widetilde{\mathcal{O}}(\frac{1}{\sqrt{n}})$ 
\begin{corollary}
These result hold for pure HMM model. For tensor case, where $[\nu,\nu',\nu'']$ is any permutation of set $\{1,2,3\}$.
\begin{align*}
\left\|\frac{1}{N(l)}\sum_t\Big[\vec{v}_{\nu,t}\otimes \vec{v}_{\nu',t}\Big]-\frac{1}{N}\mathbb{E}[\sum_t\vec{v}_{\nu,t}\otimes \vec{v}_{\nu',t}]\right\|_2\leq \frac{G}{1-\theta}\sqrt{8\frac{\log\frac{(d_{\nu}+d_{\nu'})}{\delta}}{N(l)}}
\end{align*}
with probability at least $1-\delta$ and\\

\begin{align*}
\left\|\frac{1}{N}\sum_t\Big[\vec{v}_{\nu,t}\otimes \vec{v}_{\nu',t}\otimes \vec{v}_{\nu'',t}\Big]-\frac{1}{N}\mathbb{E}[\sum_t\vec{v}_{\nu,t}\otimes \vec{v}_{\nu',t}\otimes \vec{v}_{\nu'',t}]\right\|_2\leq\frac{G}{1-\theta} \sqrt{8\frac{\log\frac{(d_\nu d_{\nu'}+d_{\nu''})}{\delta}}{N}}
\end{align*}
with probability at least $1-\delta$. The deviation bound is as follows
\begin{align*}
\left\|\frac{1}{N}\mathbb{E}[\sum_i\vec{v}_{\nu,i}\otimes \vec{v}_{\nu',i}]-\frac{1}{N}\mathbb{E}_{stat}[\sum_i\vec{v}_{\nu,i}\otimes \vec{v}_{\nu',i}]\right\|_2\leq 2\frac{G}{N(1-\theta)}
\end{align*}
\end{corollary}
\begin{proof}
Through \cite{kontorovich2008concentration} and \cite{kontorovich2014uniform} it is shown that for the HMM models, the value of $H_{n,i}$ is bounded by $\frac{G}{1-\theta}$ and then it means that the corresponding martingale difference is bounded by $\frac{cG}{1-\theta}$. In the consequence, the $\sigma^2_{HMM,\Phi}$ is bounded by $\frac{G^2}{n(1-\theta)^2}$.
\end{proof} %Azuma bound and part of the proof for general bound for c-lipzchit HMM function

\section{Whitening and Symmetrization Bound}\label{s:whitening}
\begin{theorem}[Whitening, Symmetrization and De-Whitening Bound]
Pick any $\delta$. Then for HMM model with $k$ hidden state and its multi-view representation with factor matrices $A_1,A_2,A_3$, and finite observation set with dimension $d_1,d_2,d_3$ corresponds to multi-view representation, when the number of samples with arbitrary initial state distribution satisfies
\begin{align*}
n\geq \left(\frac{G\frac{2\sqrt{2}+1}{1-\theta}}{{\omega_{\min}\min_i\lbrace\sigma^2_k(A_i)\rbrace}}\right)^2\log(2\frac{(d_1d_2+d_3)}{\delta})\max\left\{\frac{16k^{\frac{1}{3}}}{C^{\frac{2}{3}}\omega_{\min}^{\frac{1}{3}}} ,4,\frac{2\sqrt{2}k}{C^2\omega_{\min}\min_i\lbrace\sigma^2_k(A_i)\rbrace}\right\} 
\end{align*}
for some constant $C$ . After tensor symmetrizing and whitening, with low order polynomial computation complexity, the robust power method in \cite{anandkumar2012method} yield to whitened component of the views $\mu_1,\ldots,\mu_k$ , such that with probability at least $1-\delta$, we have
\begin{align*}
\|\mu_j-(\widehat{\mu}_j))\|_2\leq 18\epsilon_M
\end{align*}
for $j\in\lbrace 1,\ldots,k\rbrace$ up to permutation and 
\begin{align*}
\epsilon_M\leq \frac{2\sqrt{2} G\frac{2\sqrt{2}+1}{1-\theta}\sqrt{\frac{\log(\frac{2(d_1d_2+d_3)}{\delta})}{n}}}{(\omega^{\frac{1}{2}}_{\min}\min_i\lbrace\sigma_k(A_i))\rbrace^3}+\frac{\left(4 G\frac{2\sqrt{2}+1}{1-\theta}\sqrt{\frac{\log(2\frac{(d_1+d_2)}{\delta})}{n}}\right)^3}{{(\min_i\lbrace\sigma_k(A_i)\rbrace)^6}\omega^{3.5}_{\min}}
\end{align*}
Therefore
\begin{align*}
\left\|(A_i){(:,j)}-(\wh{A}_i)_{:,j}\right\|_2\leq \epsilon_3
\end{align*}
for $i\in\lbrace 1,2,3\rbrace$, $j\in\lbrace 1,\ldots,k\rbrace$ up to permutation and 
\begin{align*}
\epsilon_3: = G\frac{4\sqrt{2}+4}{(\omega_{\min})^{\frac{1}{2}}(1-\theta)}\sqrt{\frac{\log(2\frac{(d_1+d_2)}{\delta})}{n}}+\frac{8\epsilon_M}{\omega_{\min}}
\end{align*}

\label{thm:Whitening}
\end{theorem}
\begin{proof}
Appendix \ref{app:whitening-proof}.%
\end{proof}

\section{Whitening and  Symmetrization Bound Proof}\label{app:whitening-proof}

\textbf{Proof of Thm.~\ref{thm:Whitening}}\\
In Appendix \ref{app:ConBound}, the upper confidence bounds for deviation between empirical pairs matrices and tensor from their original ones are derived. As it is shown in \cite{song2013nonparametric} and \cite{anandkumar2014tensor} for multi-view models with factors $A_1\in\mathbb{R}^{d_1\times k}$, $A_2\in\mathbb{R}^{d_2\times k}$, $A_3\in\mathbb{R}^{d_3\times k}$ (three view model with $k$ hidden states), to derive the factor matrices, applying tensor decomposition method is one of the most efficient way. They show that for tensor decomposition method, it is needed to first; symmetrize the initial raw empirical tensor and then whiten it to get orthogonal symmetric tensor. It is well known that orthogonal symmetric tensors have unique eigenvalues and eigenvectors and can be obtained thorough power method \cite{anandkumar2014tensor}.\\
Without loss of generality, lets assume we are interested in $A_3$, the derivation can be done for other view by just permuting them. Assume tensor $M_3=\mathbb{E}[\vec{v}_{1}\otimes \vec{v}_{2}\otimes \vec{v}_{3}]$ is triple raw cross correlation between views, and matrix $R_2$ and $R_3$ are rotation matrices for rotating second and third view to first view. It means that it results in symmetric tensor $M_3(R_1,R_2,I)$. Through \cite{anandkumar2014tensor} these rotation matrices are as follow
\begin{align*}
R_1=\mathbb{E}[\vec{v}_{3}\otimes \vec{v}_{2}]\mathbb{E}[\vec{v}_{1}\otimes \vec{v}_{2}]^{-1}\\
R_2=\mathbb{E}[\vec{v}_{3}\otimes \vec{v}_{1}]\mathbb{E}[\vec{v}_{2}\otimes \vec{v}_{1}]^{-1}
\end{align*}
Define second order moment as $M_2=\mathbb{E}[\vec{v}_{1}\otimes\vec{v}_{2}]$ and its symmetrized version as $M_2(R_1,R_2)$. Lets $W\in\mathbb{R}^{d_1\times k}$ be a linear transformation such that 
\begin{align*}
M_2(R_1W,R_2W)=W^{\top}M_2(R_1,R_2)W=I
\end{align*}
where $I$ is $k\times k$ identity matrix. Then the matrix $W=U\Lambda^{-\frac{1}{2}}$ where $M_2(R_1,R_2)=U\Lambda V^{\top}$ is singular value decomposition of $M_2(R_1,R_2)$. It is well known result that tensor $M_3(W_1,W_2,W_3)=M_3(R_1W,R_2W,W)$ is symmetric orthogonal tensor and ready for power iteration to compute the unique $({A}_{3})_{i}~\forall i\in[1\ldots k]$ . 

To come up with upper confidence bound over $\left\|(\widehat{A}_3)_{i}-({A}_{3})_{i}\right\|_2$ (columns of factor matrices), it is needed to aggregate the different source of error. This deviation is due to empirical average error which derived in \ref{app:ConBound}, symmetrizing error, and whitening error. 

To obtain the upper bound over the aggregated error, lets apply the following proof technique. It is clear that for matrix $\widehat{M}_2$, we have $\widehat{W}^{\top}\widehat{R}_1^{\top}\widehat{M}_2\widehat{R}_2\widehat{W}=I$. lets assume, matrices $B,D,B$ as a singular value decomposition of $\widehat{W}^{\top}\widehat{R}_1^{\top}{M}_2\widehat{R}_2\widehat{W}=BDB^{\top}$. Then it is easy to show that for $\widetilde{W}_1=\widehat{W}_1BD^{-\frac{1}{2}}B^{\top}$, $\widetilde{W}_2=\widehat{W}_2BD^{-\frac{1}{2}}B^{\top}$, and $\widetilde{W}_3=\widehat{W}BD^{-\frac{1}{2}}B^{\top}$ then 
\begin{align*}
\widetilde{W}_2^{\top}M_2\widetilde{W}_1=I
\end{align*}
and then the $\epsilon_M$ 
\begin{align*}
&\epsilon_M=\left\|M_3(\widetilde{W}_1,\widetilde{W}_2,\widetilde{W}_3)-\widehat{M}_3(\widehat{W}_1,\widehat{W}_2,\widehat{W}_3)\right\|_2\\
&\leq \left\|(\widehat{M}_3-M_3)(\widehat{W}_1,\widehat{W}_2,\widehat{W}_3)\right\|_2+\left\|{M}_3(\widehat{W}_1-\widetilde{W}_1,\widehat{W}_2-\widetilde{W}_2,\widehat{W}_3-\widetilde{W}_3)\right\|_2
\end{align*}
It means 
\begin{align*}
\epsilon_M\leq \left\|M_3-\widehat{M}_3\right\|_2\left\|\widehat{W}_1\right\|_2\left\|\widehat{W}_2\right\|_2\left\|\widehat{W}_3\right\|_2+\left\|{M}_3(\widehat{W}_1-\widetilde{W}_1,\widehat{W}_2-\widetilde{W}_2,\widehat{W}_3-\widetilde{W}_3)\right\|_2
\end{align*}
Lets assume $U_{1,2}\Lambda_{1,2}V_{1,2}^{\top}=M_2$ is singular value decomposition of matrix $M_2$. From ${W}^{\top}{R}_1^{\top}{M}_2{R}_2{W}=I$ and the fact that $W_1U_{1,2}\Lambda^{\frac{1}{2}}=W_2V_{1,2}\Lambda^{\frac{1}{2}}$ which are the square root of matrix $M_2$ and to be able to learn all factor matrices we can show that $\left\|{W}_i\right\|_2\leq\frac{1}{\min_i\sigma_k(A_iDiag(\omega)^{\frac{1}{2}})}\leq\frac{1}{\omega^{\frac{1}{2}}_{\min}\min_i\sigma_k(A_i)}$ for $i\in\{1,2,3\}$. Now, it is clear to say, when $\left\|\widehat{M}_2-M_2\right\|_2\leq0.5\sigma_k(M_2)$ then $\left\|\widehat{W}_i\right\|_2\leq\frac{\sqrt{2}}{\omega^{\frac{1}{2}}_{\min}\min_i\sigma_k(A_i)}$ for $i\in\{1,2,3\}$ and 
\begin{align*}
 \left\|M_3-\widehat{M}_3\right\|_2\left\|\widehat{W}_1\right\|_2\left\|\widehat{W}_2\right\|_2\left\|\widehat{W}_3\right\|_2\leq \frac{2\sqrt{2}\left\|\widehat{M}_3-M_3\right\|_2}{(\omega^{\frac{1}{2}}_{\min}\min_i\sigma_k(A_i))^3}
\end{align*}
To bound the second term in $\epsilon_M$
\begin{align*}
\left\|{M}_3(\widehat{W}_1-\widetilde{W}_1,\widehat{W}_2-\widetilde{W}_2,\widehat{W}_3-\widetilde{W}_3)\right\|_2\leq\frac{1}{\sqrt{\omega_{\min}}}\prod_{i=1}^{3}\left\|Diag(\omega)^{\frac{1}{2}}A_i^{\top}(\widehat{W}_i-\widetilde{W}_i)\right\|_2
\end{align*}
then 
\begin{align*}
\left\|Diag(\omega)^{\frac{1}{2}}A_i(\widehat{W}_i-\widetilde{W}_i)\right\|_2=\left\|Diag(\omega)^{\frac{1}{2}}A_i^{\top}\widetilde{W}_i(BD^{\frac{1}{2}}B^{\top}-I)\right\|_2\leq\left\|Diag(\omega)^{\frac{1}{2}}A_i^{\top}\widetilde{W}_i\right\|_2\left\|(D^{\frac{1}{2}}-I)\right\|_2
\end{align*}
We have that $\left\|Diag(\omega)^{\frac{1}{2}}A_i^{\top}\widetilde{W}_i\right\|_2=1$. Now we control $\left\|(D^{\frac{1}{2}}-I)\right\|_2$. Let $\widetilde{E}:=M_2-F_k$ where $F=\widehat{M}_2$, and $F_k$ is its restriction to top-$k$ singular values. Then, we have $\left\|\widetilde{E}\right\|_2\leq\left\|\widehat{M}_2-M_2\right\|_2+\sigma_{k+1}(F)\leq 2\left\|\widehat{M}_2-M_2\right\|_2$. We now have 

\begin{align}\label{eq:WDbound}
&\left\|(D^{\frac{1}{2}}-I)\right\|_2\leq \left\|(D^{\frac{1}{2}}-I)(D^{\frac{1}{2}}+I)\right\|_2\leq\left\|(D-I)\right\|_2=\left\|(BDB^{\top}-I)\right\|_2=\left\|(\widehat{W}_1^{\top}M_2\widehat{W}_2-I)\right\|_2\\
&=\left\|(\widehat{W}_1^{\top}\widetilde{E}\widehat{W}_2)\right\|_2\leq \left\|\widehat{W}_1\right\|_2\left\|\widehat{W}_2\right\|_22\left\|\widehat{M}_2-M_2\right\|_2\leq \frac{4\left\|\widehat{M}_2-M_2\right\|_2}{(\omega^{\frac{1}{2}}_{\min}\min_i\sigma_k(A_i))^2}
\end{align}
As a conclusion it is shown that     
\begin{align}
\epsilon_M\leq \frac{2\sqrt{2}\left\|\widehat{M}_3-M_3\right\|_2}{(\omega^{\frac{1}{2}}_{\min}\min_i\sigma_k(A_i))^3}+\frac{\left(\frac{4\left\|\widehat{M}_2-M_2\right\|_2}{(\omega^{\frac{1}{2}}_{\min}\min_i\sigma_k(A_i))^2}\right)^3}{\sqrt{\omega_{\min}}}
\label{eq:finalbound}
\end{align}
when $\left\|\widehat{M}_2-M_2\right\|_2\leq0.5\sigma_k(M_2)$.\\
Through  Appendix \ref{app:ConBound}, the followings hold
\begin{align*}
\left\|M_2-\widehat{M}_2\right\|_2\leq G\frac{2\sqrt{2}+1}{1-\theta}\sqrt{\frac{\log(2\frac{(d_1+d_2)}{\delta})}{n}}
\end{align*}
\begin{align*}
\left\|M_3-\widehat{M}_3\right\|_2\leq G\frac{1+\frac{1}{\sqrt{8}n^{\frac{1}{2}}}}{1-\theta}\sqrt{8\frac{\log(\frac{2(d_1d_2+d_3)}{\delta})}{n}}
\end{align*}
with probability at least $1-\delta$. It is followed by 
\begin{align*}
\epsilon_M\leq \frac{2\sqrt{2} G\frac{2\sqrt{2}+1}{1-\theta}\sqrt{\frac{\log(\frac{2(d_1d_2+d_3)}{\delta})}{n}}}{(\omega^{\frac{1}{2}}_{\min}\min_i\sigma_k(A_i))^3}+\frac{\left(4 G\frac{2\sqrt{2}+1}{1-\theta}\sqrt{\frac{\log(2\frac{(d_1+d_2)}{\delta})}{n}}\right)^3}{{(\min_i\sigma_k(A_i))^6}\omega^{3.5}_{\min}}
\end{align*}
with probability at least $1-\delta$. To this result holds, it is required $\left\|\widehat{M}_2-M_2\right\|_2\leq0.5\sigma_k(M_2)$ and from \cite{anandkumar2012method} that $\epsilon_M\leq \frac{C_1}{\sqrt{k}}$. Then for the first requirement
\begin{align*}
n\geq\left(\frac{G\frac{2\sqrt{2}+1}{1-\theta}}{0.5(\omega^{\frac{1}{2}}_{\min}\min_i\sigma_k(A_i))^2}\right)^2\log(2\frac{(d_1+d_2)}{\delta})
\end{align*}
 and for the second requirement $\epsilon_M\leq \frac{C_1}{\sqrt{k}}$ to be hold it is enough that each term in Eq~\ref{eq:finalbound} is upper bounded by $\frac{C}{\sqrt{k}}$ for some constant $C$.
\begin{align*}
\frac{C}{\sqrt{k}}\geq \frac{2\sqrt{2} G\frac{2\sqrt{2}+1}{1-\theta}\sqrt{\frac{\log(\frac{2(d_1d_2+d_3)}{\delta})}{n}}}{(\omega^{\frac{1}{2}}_{\min}\min_i\sigma_k(A_i))^3}
\end{align*}
then
\begin{align*}
n\geq \left(\frac{2\sqrt{2} G\frac{2\sqrt{2}+1}{1-\theta}}{C(\omega^{\frac{1}{2}}_{\min}\min_i\sigma_k(A_i))^3}\right)^2k\log(\frac{2(d_1d_2+d_3)}{\delta})
\end{align*}
and for the second part
\begin{align*}
\frac{C}{\sqrt{k}}\geq\frac{\left(4 G\frac{2\sqrt{2}+1}{1-\theta}\sqrt{\frac{\log(2\frac{(d_1+d_2)}{\delta})}{n}}\right)^3}{{(\min_i\sigma_k(A_i))^6}\omega^{3.5}_{\min}}
\end{align*}
\begin{align*}
n\geq\left(\frac{4k^{\frac{1}{6}} G\frac{2\sqrt{2}+1}{1-\theta}}{C^{\frac{1}{3}}{(\min_i\sigma_k(A_i))^2}\omega^{\frac{3.5}{3}}_{\min}}\right)^2\log(2\frac{(d_1+d_2)}{\delta})
\end{align*}
It means it is enough that 
\begin{align*}
&n\geq \left(\frac{G\frac{2\sqrt{2}+1}{1-\theta}}{{\omega_{\min}\min_i\sigma^2_k(A_i)}}\right)^2\\
&~~~~~~~~~~~~~~~\max\left\{\log(2\frac{(d_1+d_2)}{\delta})\max\left\{\frac{16k^{\frac{1}{3}}}{C^{\frac{2}{3}}\omega_{\min}^{\frac{1}{3}}} ,4\right\},\log(2\frac{(d_1d_2+d_3)}{\delta})\frac{2\sqrt{2}k}{C^2\omega_{\min}\min_i\sigma^2_k(A_i)}\right\} 
\end{align*}
which can be reduced to 
\begin{align*}
n\geq \left(\frac{G\frac{2\sqrt{2}+1}{1-\theta}}{{\omega_{\min}\min_i\sigma^2_k(A_i)}}\right)^2\log(2\frac{(d_1d_2+d_3)}{\delta})\max\left\{\frac{16k^{\frac{1}{3}}}{C^{\frac{2}{3}}\omega_{\min}^{\frac{1}{3}}} ,4,\frac{2\sqrt{2}k}{C^2\omega_{\min}\min_i\sigma^2_k(A_i)}\right\} 
\end{align*}
In \cite{anandkumar2014tensor} it is shown that when $\epsilon_M=\left\| M_3(W_1,W_2,W_3)-\widehat{M}_3(\widehat{W}_1,\widehat{W}_2, \widehat{W}_3)\right\|$ then the robust power method in \cite{anandkumar2012method} decomposes the tensor and comes up with set $\widehat{\lambda}_i$ and orthogonal $\widehat{\mu}_i$ where 
\begin{align*}
\left\|M_3(W_1,W_2,W_3)-\sum_i^k\widehat{\lambda}_i\widehat{\mu}_i^{\otimes^3}\right\|_2\leq 55\epsilon_M
\end{align*}
\begin{align*}
\left\|\omega_i^{-1/2}\mu_i-\widehat{\lambda}_i\widehat{\mu}_i\right\|_2\leq 8\epsilon_M\omega_i^{-1/2} 
\end{align*}
and 
\begin{align}\label{eq:bound.omega}
\left|\omega_i^{-1/2}-\widehat{\lambda}_i\right|\leq 8\epsilon_M\omega_i^{-1/2} 
\end{align}
It can be verified that 
\begin{align*}
\|\mu_i-\wh{\mu}_i\|_2 \leq 18\epsilon_M.
\end{align*}

\begin{proof}
In order to simplify the notation, in the following we use $\mu = \mu_{i}$, $\omega = \omega_i$, and $\zeta = \omega_i^{-1/2}$, similar terms for the estimated quantities. From above mentioned bound, we have
\begin{align*}
\big\|\zeta \mu - \wh{\zeta} \wh{\mu}\big\|_2 = \big\|\zeta (\mu - \wh{\mu}) - (\wh{\zeta}-\zeta) \wh{\mu}\big\|_2 \leq 8\zeta\epsilon_3(l).
\end{align*}
We take the square of the left hand side and we obtain
\begin{align*}
\big\|\zeta (\mu - \wh{\mu}) - (\wh{\zeta}-\zeta) \wh{\mu}\big\|_2^2 &= \zeta^2\|\mu-\wh{\mu}\|_2^2 + (\wh{\zeta}-\zeta)^2 \|\wh{\mu}\|_2^2 - 2 \zeta(\zeta-\wh{\zeta})\sum_{s=1}^d [\wh{\mu}]_s ([\mu]_s-[\wh{\mu}]_s)\\
& \geq \zeta^2\|\mu-\wh{\mu}\|_2^2 - 2 \zeta|\zeta-\wh{\zeta}|  \Big|\sum_{s=1}^d [\wh{\mu}]_s ([\mu]_s-[\wh{\mu}]_s)\Big|\\
& \geq \zeta^2\|\mu-\wh{\mu}\|_2^2 - 2 \zeta|\zeta-\wh{\zeta}\|\wh{\mu}\|_2 \|\mu-\wh{\mu}\|_2,
\end{align*}
where in the last step we used the Cauchy-Schwarz inequality. Thus we obtain the second-order equation
\begin{align*}
\zeta\|\mu-\wh{\mu}\|_2^2 - 2 (\zeta-\wh{\zeta})\|\wh{\mu}\|_2 \|\mu-\wh{\mu}\|_2 \leq 64\zeta\epsilon_3(l)^2.
\end{align*}
Solving for $\|\mu-\wh{\mu}\|_2$ we obtain
\begin{align*}
\|\mu-\wh{\mu}\|_2 \leq \frac{|\zeta-\wh{\zeta}|  \|\wh{\mu}\|_2 + \sqrt{(\zeta-\wh{\zeta})^2\|\wh{\mu}\|_2^2 + 64\zeta^2\epsilon_M^2}}{\zeta}.
\end{align*}
Now we can use the bound in Eq.~\ref{eq:bound.omega} and the fact that $\|\wh{\mu}\|_2 \leq \|\wh{\mu}\|_1 \leq 1$ since $\wh{\mu}$ is a probability distribution and obtain
\begin{align*}
\|\mu-\wh{\mu}\|_2 \leq \frac{5\epsilon_M + \sqrt{25\epsilon_M^2 + 64\zeta^2\epsilon_M^2}}{\zeta} = \frac{\epsilon_M}{\zeta}\big(5 + \sqrt{25 + 64\zeta^2}\big) \leq \frac{\epsilon_M}{\zeta}\big(10 + 8\zeta\big).
\end{align*}
Plugging the original notation into the previous expression, we obtain the final statement. Finally, since $\zeta = \omega_\pi^{(l)}(i)^{-1/2}$ and $\omega_\pi^{(l)}$ is a probability, we have $1/\zeta \leq 1$ and thus
\begin{align*}
\|\mu-\wh{\mu}\|_2 \leq 18\epsilon_M.
\end{align*}
which all results are up to permutation.
\end{proof}
\begin{lemma}[De-Whitening]
The upper bound over the de-whitened $\mu_i$ is as follow
\begin{align}\label{eq:final-dewhitening}
\epsilon_3:=\left\|(A_3)_i-(\wh{A_3})_i\right\|_2\leq G\frac{4\sqrt{2}+2}{(\omega_{\min})^{\frac{1}{2}}(1-\theta)}\sqrt{\frac{\log(2\frac{(d_1+d_2)}{\delta})}{n}}+\frac{8\epsilon_M}{\omega_{\min}}
\end{align}

\end{lemma}

\begin{proof}
As it is shown in \cite{anandkumar2012method}, to reconstruct the columns of views $A_1,A_2,A_3$, de-whitening process is needed. It is shown that the columns can be recovered by $(A_1)_i=\widetilde{W}_1^\dagger \lambda_i\mu_i$, $(A_2)_i=\widetilde{W}_2^\dagger \lambda_i\mu_i$, and $(A_3)_i=\widetilde{W}_3^\dagger \lambda_i\mu_i$. For simplicity, let just investigate the third view, the process for other two views is same as third view. 
\begin{align*}
\left\|(A_3)_i-(\wh{A_3})_i\right\|_2\leq\left\|\widetilde{W}_3^\dagger-\wh{W}_3^{\dagger}\right\|_2\left\|\lambda_i\mu_i\right\|_2+\left\|\wh{W}_3^{\dagger}\right\|_2 \left\|\lambda_i\mu_i-\wh{\lambda}_i\wh{\mu}_i\right\|_2
\end{align*}
it is clear that $\left\|\lambda_i\mu_i-\wh{\lambda}_i\wh{\mu}_i\right\|_2\leq \frac{8\epsilon_M}{(\omega_{\min})^{\frac{1}{2}}}$, $\left\|\wh{W}_3^{\dagger}\right\|\leq 1$, and $\left\|\lambda_i\mu_i\right\|_2\leq \frac{1}{\omega_{\min}}$. 
\begin{align*}
\left\|\widetilde{W}_3^\dagger-\wh{W}_3^{\dagger}\right\|_2=\left\|(BD^{\frac{1}{2}}B^{\top}-I)\wh{W}_3^{\dagger}\right\|_2\leq 2\left\|M_2-\wh{M}_2\right\|_2
\end{align*}
where the last inequality is inspired by Eq~ \ref{eq:WDbound}. Then
\begin{align*}
\left\|(A_3)_i-(\wh{A_3})_i\right\|_2\leq \frac{2}{\omega_{\min}}\left\|M_2-\wh{M}_2\right\|_2+\frac{8\epsilon_M}{(\omega_{\min})^{\frac{1}{2}}}
\end{align*}
Therefore
\begin{align}
\epsilon_3:=\left\|(A_3)_i-(\wh{A_3})_i\right\|_2\leq G\frac{4\sqrt{2}+2}{(\omega_{\min})^{\frac{1}{2}}(1-\theta)}\sqrt{\frac{\log(2\frac{(d_1+d_2)}{\delta})}{n}}+\frac{8\epsilon_M}{\omega_{\min}}
\end{align}

%  $\leq\frac{2}{(\omega^{\frac{1}{2}}_{\min}\min_i\sigma_k(A_i))^2}$

\end{proof}

%\newpage
%\input{Table}

}

\end{document}